\documentclass[letterpaper]{article} 
\usepackage[preprint]{aaai2027}  
\usepackage{times}  
\usepackage{helvet} 
\usepackage{courier} 
\usepackage[hyphens]{url}  
\usepackage{graphicx} 
\usepackage{natbib}  
\usepackage{caption} 
\usepackage{float}  

\usepackage{array}
\usepackage{amsmath, amssymb, amsthm, mathtools}
\usepackage{algorithm}
\usepackage{algorithmic}
\usepackage{booktabs}
\usepackage{multicol}

\theoremstyle{plain}
\newtheorem{theorem}{Theorem}
\newtheorem{lemma}{Lemma}

\newtheorem{proposition}{Proposition}

\theoremstyle{definition}

\theoremstyle{remark}

\newcommand{\R}{\mathbb{R}}

\title{Convergence Theory of Knowledge Distillation in Asynchronous P2P Gossip Learning Network}

\author{
\mdseries
Lucas Qingyang Fang$^{1}$,
Tiyao Liu$^{2}$,
Jinhao Jing$^{3}$,
Zeji Li$^{4}$,
Kaijie Chen$^{5}$,\\
Harikrishna Kuttivelil$^{1}$\corresponding,
Katia Obraczka$^{1}$\corresponding
}
\affiliations{
$^{1}$University of California, Santa Cruz, USA\\
$^{2}$China University of Petroleum, China\\
$^{3}$The Chinese University of Hong Kong, Shenzhen, China\\
$^{4}$City University of Hong Kong, Hong Kong SAR, China\\
$^{5}$University of Virginia, USA
}

\begin{document}

\maketitle

\begin{abstract}
Decentralized, serverless learning increasingly connects devices running different architectures, where the standard tool, decentralized SGD, is undefined as models with different parameter counts cannot be averaged. Knowledge distillation (KD) exchanges soft predictions rather than weights and sidesteps this obstacle, yet convergence theory for fully decentralized, asynchronous peer-to-peer (P2P) KD is lacking. We provide one, relocating consensus from parameter space to function (output) space: a KD event is a geometric contraction operator in logit space on the peers' predictive distributions, which we analyse in the Hilbert space of predictions on a reference measure. Under standard smoothness/variance assumptions and two realizability assumptions, one bridging parameter SGD to the functional step and one controlling restricted task/KD alignment, the time-averaged functional stationarity and function-space disagreement converge at rate $\mathcal{O}(1/(\eta T))$ to an $\mathcal{O}(\eta)+\mathcal{O}(B_f^2)+\mathcal{O}(\zeta_f^2)$ neighbourhood. Here $B_f$ is the distance from the task optimum to the peers' reachable classes and $\zeta_f$ measures persistent local-task heterogeneity. Across homogeneous, width-heterogeneous, and mixed-family networks of the experiments, KD contracts function disagreement by $40$--$61\times$, while isolated training does not. The sampled stationarity diagnostic has late-transient exponents $0.99$--$1.90$ on the shared-skeleton main runs, and the four-point step-size sweep exhibits the predicted transient--neighbourhood tradeoff.
\end{abstract}
\section{Introduction}
\label{sec:intro}

Gossip learning removes the central server from distributed training. Each peer keeps
private data, performs local updates, and communicates with neighbours over a peer-to-peer
(P2P) network without a global clock. In decentralized SGD, the communication step is a
weighted average of parameter vectors, $x_i\leftarrow\sum_j W_{ij}x_j$. This operation
requires all peers to use the same parameter space $\R^d$ \citep{lian2017dpsgd,boyd2006gossip},
so it is undefined when the
network contains different architectures.

Knowledge distillation (KD) provides a natural alternative because peers exchange soft
predictions rather than parameters. Predictions lie in the class simplex regardless of the
network that produced them. Federated distillation methods use this common output space,
but the methods considered in this line of work rely on a coordinator or synchronized
rounds \citep{hinton2015distill,li2019fedmd,lin2020feddf,chang2019cronus} and do not analyze the fully decentralized, asynchronous setting. We explore whether an asynchronize P2P KD process admits a convergence theory
when the peers have different reachable function classes.

We measure consensus on a reference measure $\mu$. For peer
$i$, let $p_i=\mathrm{softmax}(z_i/\tau)$ denote the temperature-scaled prediction. The
predictions belong to the common Hilbert space $L^2(\mu;\R^C)$, where we measure disagreement
by $\Psi^f_t=\tfrac1N\sum_i\mathbb E\|p_i-\bar p\|_\mu^2$. A KD exchange has logit
gradient $\tau(p_i-p_j)$, so it acts as a function-space mixing step. Its contraction rate
is proportional to $\rho_f=\Theta(\eta\alpha p_{\min}\lambda_2(\mathcal L)/|E|)$.
We couple this contraction with descent of the task risk in a Lyapunov argument. The proof
uses A4 (detailed and justified in Appendix~\ref{app:assumption-validity}, along with A10) to relate a parameter update to its functional effect and A10 to control directions that are outside an individual peer's reachable class.

\paragraph{Contributions.}
We establish a convergence bound for asynchronous random-edge P2P KD that controls both
time-averaged functional stationarity and function-space disagreement within an
$\mathcal{O}(\eta)+\mathcal{O}(B_f^2)+\mathcal{O}(\zeta_f^2)$ neighbourhood.
The three terms separate stochastic discretization, representational error, and
persistent local-task heterogeneity. We characterize the function-space
communication operator that remains defined for heterogeneous peers and keep parameter
averaging as a homogeneous-model control. The appendix gives a self-contained proof and
states the two modelling bridges explicitly, including sufficient verification models
for A10. In the experiments, we record the theorem's
observables and use isolated training and homogeneous D-SGD only as mechanism controls.
The experiments do not make a task-accuracy ranking claim or determine $B_f$ from
architecture.

\subsection{Related Work}
\label{sec:related}
\paragraph{Decentralized optimization and gossip SGD.}
Decentralized-SGD theory covers non-convex optimization with parameter averaging
and has been extended to changing topology, directed graphs, data heterogeneity,
and asynchronous updates \citep{lian2017dpsgd,koloskova2020unified,assran2019pushsum}. These analyses use local stochastic gradients in a
shared parameter space~\cite{dandi2022mixing, sun2019sonata}. Thus, this mechanism is naturally undefined in a system with models that have different parameter size and architecture, while Knowledge Distillation offers a solution because it operates at output space.

\paragraph{Knowledge distillation (KD) and federated KD.}
KD transfers a teacher's soft predictions to a student~\citep{hinton2015distill}.
Because the exchanged object is a prediction, KD can connect models with different
architectures. FedMD distils on a shared public set~\citep{li2019fedmd}, FedDF uses
server-side ensemble distillation~\citep{lin2020feddf}, and Cronus studies black-box
soft-label exchange under malicious clients~\citep{chang2019cronus}. These methods
use a coordinator or a synchronized protocol. Recent decentralized federated KD
methods retain central coordination~\cite{taya2022decentralized}, including IMFL's gradient aggregator
\citep{ying2026_decen_fed_kd_incentive_mechanism} and TopoMoDistill's synchronization
server~\citep{wu2026topology_fed_kd}. Serverless gossip distillation~\citep{khowaja2026splitgossip_gossip_distillation} has so far been
evaluated, but its general convergence behaviour is not established. Our setting
is the asynchronize serverless case, for which we provide the first convergence analysis in function space. The more comprehensive related work is in Appendix~\ref{sec:expanded-related-work}.

\section{Preliminaries}
\label{sec:preliminaries}

\subsection{Important Notations}
The full table can be found in Appendix~\ref{app: whole notation}. 
\begin{table}[H]
\centering
\label{tab:notation}
\begin{tabular}{@{}lp{0.72\columnwidth}@{}}
\toprule
\textbf{Symbol} & \textbf{Description} \\
\midrule
$N$, $\lambda_2(\mathcal{L})$ & $N$ peers on a graph with algebraic connectivity $\lambda_2(\mathcal{L})$. \\
$z_i^{(t)}$, $p_i^{(t)}$ & Logits $z_i^{(t)} \in \mathbb{R}^C$; distribution $p_i^{(t)} := \mathrm{softmax}(z_i^{(t)}/\tau)$ at temperature $\tau>0$. \\
$\mathcal{H}$ & Hilbert space $L^2(\mu; \mathbb{R}^C)$ with $\|p_i - p_j\|_\mu^2 = \int_{\mathcal{X}} \|p_i(x) - p_j(x)\|^2 d\mu(x)$. \\
$F(p)$, $F^\star$ & Global objective $F(p) := \frac{1}{N}\sum_i F_i(p)$; task minimiser $F^\star := \inf_{p} F(p)$. \\
$\Psi^f_t$ & Consensus error $\frac{1}{N}\sum_i \mathbb{E}\|p_i^{(t)} - \bar{p}^{(t)}\|_\mu^2$, with mean predictor $\bar{p}^{(t)}$. \\
$\rho_f$ & Geometric contraction rate at which KD gossip shrinks $\Psi^f_t$. \\
$A_t^f$ & Teacher lag (drift from distilling a lagging EMA teacher $p_j^{\mathrm{EMA}}$). \\
$B_f^2$ & Capacity floor (A3); oracle distance measuring how far the required optimum lies outside the weakest architecture's reach. \\
\bottomrule
\end{tabular}
\end{table}

\subsection{Core Assumptions}
\label{asp:realizability}
The standard assumptions listed in Appendix~\ref{app: detailed assumptions} follow
those used in D-SGD theory~\citep{zeng2025_dsgd_bounds,dandi2022mixing}. For the assumptions specific to P2P KD, they are listed below. For A4 and A10, though remaining hypotheses, they are well-grounded assumptions that are strongly defended both mathematically and from related work that established on that in Appendix~\ref{app:assumption-validity}. 

\label{asp:a3-capacity-floor}
\textbf{(A3) Representational capacity floor:}
A3 names the irreducible approximation error that appears when heterogeneous
peers are asked to approach the same task-optimal predictor. Let
$\mathcal P_i:=\{p_\theta:\theta\in\mathbb R^{d_i}\}\subseteq\mathcal H$ be
the reachable prediction class of peer $i$ on the reference measure, and let
$q^\star:=\arg\min_{p\in\mathcal S}F(p)$ be the constrained task minimizer
over $\mathcal S:=\{p=\operatorname{softmax}(z/\tau):\|z\|_\infty\le G_z\}$.
The oracle capacity floor is finite:
\[
\begin{aligned}
B_f^2&:=\frac1N\sum_{i=1}^N
\operatorname{dist}_\mu(q^\star,\mathcal P_i)^2,\\
\operatorname{dist}_\mu(q,\mathcal P_i)
&:=\inf_{p\in\mathcal P_i}\|q-p\|_\mu .
\end{aligned}
\]
This floor bridges the gap between a common output-space optimum and
peer-specific reachable function classes. The proof uses $B_f^2$ only as an
approximation-error parameter in the terminal neighbourhood. Appendix
\ref{app: detailed assumptions} gives the constrained-minimizer convention and
explains why $B_f$ is an oracle quantity.\\

\label{asp:a4-kernel-bridge}
\textbf{(A4) Parameter-to-logit realizability (kernel geometry):}
This states when a small parameter-space SGD step can be interpreted as a
controlled finite-support step in logit and prediction space. Let
$\mu_M=M^{-1}\sum_{r=1}^M\delta_{x_r}$  be the empirical reference measure and
let
\[
\mathcal H_z:=\{Z\in\mathbb R^{M\times C}:Z_{r,:}\mathbf 1=0\}
\]
with the $M^{-1}$ Frobenius inner product. For peer $i$, define centered
logits $Z_i(\theta_i)$, the logit Jacobian
$J_i^z:=D_{\theta_i}Z_i(\theta_i)$, the restricted logit kernel
$\Theta_i:=J_i^z(J_i^z)^*$, the tangent range
$T_i^z:=\operatorname{range}(J_i^z)$, and the orthogonal projector
$\Pi_i^z:=\Pi_{T_i^z}$. In the frozen finite-support kernel region used by
Theorem~\ref{thm:main_convergence},
\[
\kappa_-\Pi_i^z\preceq\Theta_i\preceq\kappa_+\Pi_i^z .
\]
For the prediction Jacobian $J_i^p:=D_{\theta_i}P_i$ and prediction kernel
$K_i:=J_i^p(J_i^p)^*$, with tangent range $T_i^p$ and projector $\Pi_i^p$,
there are constants $0<\underline\kappa_p\le\overline\kappa_p<\infty$ such
that
\[
\underline\kappa_p\Pi_i^p\preceq K_i\preceq
\overline\kappa_p\Pi_i^p .
\]
The parameter-to-function Taylor remainders and stochastic pushforward noise
satisfy the conditional moment bounds used in Appendix
\ref{app: detailed assumptions}. For the private-batch to reference-support
task pushforward, write
\[
\begin{aligned}
G_i^{\rm task}&=K_i g_i^t+\delta_i^t,\\
\frac1N\sum_i\|\delta_i^t\|_\mu^2
&\le
\chi_G\frac1N\sum_i\|g_i^t\|_\mu^2+\chi_BB_f^2+\chi_\zeta\zeta_f^2 .
\end{aligned}
\]
Here $g_i^t=\nabla_fF_i(p_i^t)$. This cross-kernel fidelity clause is part of
A4 and is not asserted by A6. A4 bridges actual parameter SGD and the
finite-support functional recursion used in the proof. It does not assert
full-space surjectivity and does not imply A10. The chain-rule identity,
Taylor remainder, and kernel-persistence details are deferred to Appendix
\ref{app: detailed assumptions} (the assumption itself is further justified in Appendix~\ref{app:assumption-validity}) and Lemma~\ref{lem:f-kd-gradient}.\\
\label{asp:a10-alignment}
\textbf{(A10) Restricted task/KD alignment:}
This limits how much of the task gradient and directed KD pull lies outside the
tangent directions that an individual peer can realize. For every event time
$t$, conditional on $\mathcal F_t$, let
$g_i^t:=\nabla_fF_i(p_i^t)$,
$\psi_t^f:=N^{-1}\sum_i\|p_i^t-\bar p^t\|_\mu^2$, and
$\psi_t^z:=N^{-1}\sum_i\|Z_i^t-\bar Z^t\|_{\mu_M}^2$. The task-gradient normal
component obeys
\[
\begin{aligned}
&\frac1N\sum_i\|(I-\Pi_i^{p,t})g_i^t\|_\mu^2\\
&\le \gamma_{\mathrm{task}}\frac1N\sum_i\|g_i^t\|_\mu^2
+\nu_{\mathrm{task},\Psi}\psi_t^f\\
&\quad+\nu_{\mathrm{task},B}B_f^2,\qquad
0\le\gamma_{\mathrm{task}}<1 .
\end{aligned}
\]
For the KD clause, define the centered finite-support pull
\[
\begin{aligned}
d_{ij}^z&:=A_{ij}(Z_i-Z_j)=P_i-P_j,\\
A_{ij}&:=\int_0^1D_Z\operatorname{softmax}
\left(\frac{Z_j+s(Z_i-Z_j)}{\tau}\right)\,ds .
\end{aligned}
\]
Both $Z_i-Z_j$ and $P_i-P_j$ are represented in the class-sum-zero
finite-support ambient space, so $\Pi_i^z$ acts on $d_{ij}^z$ in that common
coordinate system. The directed KD-pull normal components obey
\[
\begin{split}
\frac1{|\vec E|}\sum_{(i,j)\in\vec E}
\|(I-\Pi_i^{z})d_{ij}^z\|_{\mu_M}^2
\le{}&\gamma_{\mathrm{KD}}\psi_t^z+\nu_{\mathrm{KD},B}B_f^2,
\\
\vec E:=\{(i,j),(j,i):(i,j)\in E\}.
\end{split}
\]
A10 bridges common output-space KD pulls and peer-specific tangent spaces. The
small-gain conditions use only the transient coefficients
$\gamma_{\mathrm{task}}$ and $\gamma_{\mathrm{KD}}$; the coefficients multiplying
$B_f^2$ need only be finite. Appendix Propositions~\ref{prop:a10-graph}--%
\ref{prop:a10-counterexample} give sufficient cases and explain why the
relative task-gradient term is necessary (the assumption itself is further justified in Appendix~\ref{app:assumption-validity}).

\section{System Architecture}
\label{sec:architecture}
We use the gossip-distillation architecture shown in Figure~\ref{fig:Asynchronize P2P KD Architecture}.
\begin{figure}
    \centering
        \caption{Asynchronous P2P knowledge-distillation architecture. Each peer alternates
    between the student and teacher roles as gossip pairs form stochastically. No
    parameters are exchanged.}
    \includegraphics[width=1.03\linewidth]{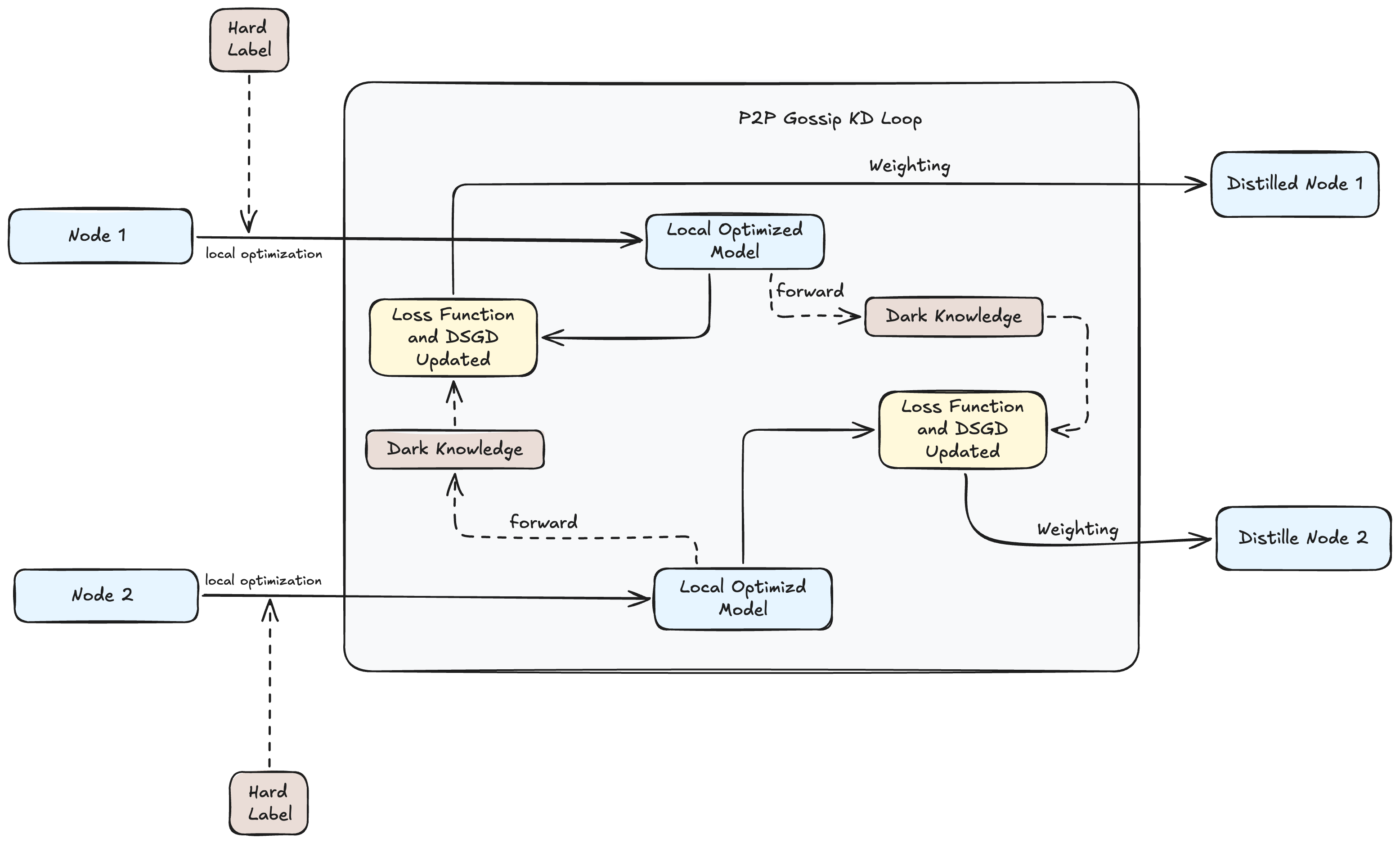}
    \label{fig:Asynchronize P2P KD Architecture}
\end{figure}

\subsection{Network Topology and Heterogeneous Peers}
The communication graph is an undirected connected graph $\mathcal{G}=(V,E)$ with
peer set $V$ and edge set $E$. There are $N=|V|$ peers. Peer $i$ has parameters
$\theta_i\in\mathbb{R}^{d_i}$, and the dimensions $d_i$ may differ. Parameter averaging
is then undefined, so peers exchange output distributions. Given an input $x$,
peer $i$ produces logits $z_i(x;\theta_i)\in\mathbb{R}^C$ and the temperature-scaled
distribution $p_i(x;\theta_i)\in\Delta^{C-1}$:
\begin{equation}
p_{i}(x)_c = \frac{\exp(z_{i}(x)_c / \tau)}{\sum_{k=1}^C \exp(z_{i}(x)_k / \tau)}
\end{equation}
where $\tau > 0$ is the distillation temperature governing the softness of the distribution.

\subsection{Random-Edge Asynchrony and EMA Teachers}
We use a random-edge asynchrony model with no global clock. At event $t$, edge
$e=(i,j)\in E$ is activated and both endpoints act once as student and teacher. The
model captures uncoordinated activation but does not model concurrent writes. A bounded
number of such races can be represented by increasing the lag bound $A$ below. Peer $j$
transmits soft probabilities on a shared reference batch. In the EMA variant, those
probabilities come from an exponential moving average of its parameters, denoted by
$p_j^{\mathrm{EMA}}$. The live-network variant sets the EMA parameters equal to the
current parameters.

\subsection{The P2P-KD Loss and Student Update}
When edge $(i,j)$ is selected, peer $i$ takes an SGD step on a loss that combines the
supervised task with a distillation pull toward peer $j$:
\begin{equation}
\begin{split}
\mathcal{L}_i(\theta_i) = (1-\alpha)\, \ell_{\mathrm{CE}}\big(p_i(x;\theta_i), y\big) +
\\
\alpha\, \tau^2\, D_{\mathrm{KL}}\big(p_j^{\mathrm{EMA}}(x) \parallel p_i(x;\theta_i)\big),
\end{split}
\end{equation}
Here $\ell_{\mathrm{CE}}$ is the cross-entropy with label $y$, $D_{\mathrm{KL}}$ is the
teacher-to-student divergence, and $\alpha\in(0,1)$ controls their relative weight.

\begin{algorithm}
\caption{Asynchronous random-edge P2P knowledge distillation}
\label{alg:p2pkd}
\begin{algorithmic}[1]
\REQUIRE peers $\{\theta_i\}_{i=1}^N$ (heterogeneous), graph $\mathcal G=(V,E)$, reference measure $\mu$, weight $\alpha$, temperature $\tau$, step $\eta$, EMA decay $\beta$
\STATE initialise each $\theta_i$ (local CE-only pretraining sets $\Psi^f_0$), then set $\theta_i^{\mathrm{EMA}}\!\leftarrow\!\theta_i$
\FOR{event $t=0,1,2,\dots$}
  \STATE sample an active edge $(i,j)\sim\mathrm{Unif}(E)$ (clock)
  \STATE draw a reference/training batch $x$ and form teacher labels $p_j^{T}\!\leftarrow\!\mathrm{softmax}(z_j^{\mathrm{EMA}}(x)/\tau)$ (live: $\theta_j^{\mathrm{EMA}}\!=\!\theta_j$)
  \STATE $\theta_i \leftarrow \theta_i-\eta\nabla_{\theta_i}\!\big[(1-\alpha)\ell_{\mathrm{CE}}(p_i(x),y)+\alpha\tau^2 D_{\mathrm{KL}}(p_j^{T}\|p_i(x))\big]$
  \STATE symmetrically update $\theta_j$ with teacher $p_i^{T}$ \COMMENT{both endpoints are student and teacher}
  \STATE $\theta_i^{\mathrm{EMA}}\!\leftarrow\!\beta\theta_i^{\mathrm{EMA}}+(1-\beta)\theta_i$ and update $j$ likewise
\ENDFOR
\end{algorithmic}
\end{algorithm}

\subsection{Network Objective}
Let $\mathcal D_i$ denote the private data distribution of peer $i$. Its local risk is
$$F_i(p):=\mathbb E_{(x,y)\sim\mathcal D_i}\,\ell_{\mathrm{CE}}(p(x),y),$$
and the network objective is the average local risk:
\begin{equation}
     \begin{split}
         F(p):=\frac1N &\sum_{i=1}^N F_i(p)\ ,\quad F^\star:=\inf_{p\in\mathcal H}F(p)>-\infty,\\
         &\Delta_F:=F(\bar p^{(0)})-F^\star
     \end{split}
 \end{equation}
The identity $\nabla_{\!f}F=\tfrac1N\sum_i\nabla_{\!f}F_i$ follows from this definition.

\section{Theoretical Convergence Analysis}
\label{sec:convergence}
We analyse the event process on the finite reference measure $\mu_M$ and use
$\mu:=\mu_M$ in all theorem norms below. This is
the space on which heterogeneous peers have a common state. A4 separates the
exact parameter-to-logit chain rule from the local frozen-kernel condition, and
A10 controls only the components that the peer-specific tangent spaces cannot
realize. The proof uses the actual preconditioned KD update.

\subsection{KD Geometry}
\begin{lemma}[KD logit gradient]
\label{lem:main-kd-gradient}
For $P_i=\operatorname{softmax}(Z_i/\tau)$ and a detached teacher $P_j^T$,
\begin{equation}
\nabla_{Z_i}\!\left[\tau^2D_{\rm KL}(P_j^T\|P_i)\right]
=\tau(P_i-P_j^T).
\label{eq:main-kd-gradient}
\end{equation}
Consequently, the first-order logit update is
$-\eta\alpha\tau\Theta_i(P_i-P_j^T)$.
\end{lemma}
See Lemma~\ref{lem:f-kd-gradient} for the chain-rule and Taylor-remainder
proof. The temperature factor in Equation~\ref{eq:main-kd-gradient} is exact.

\begin{lemma}[Softmax geometry]
\label{lem:main-kl-geometry}
Under A1, centered logits and probabilities satisfy
\begin{equation}
\frac{p_{\min}}{\tau}\|Z-Z'\|_{\mu_M}
\le \|P-P'\|_{\mu_M}
\le\frac1{2\tau}\|Z-Z'\|_{\mu_M},
\label{eq:main-softmax-equivalence}
\end{equation}
and $\tau^2D_{\rm KL}(P'\|P)$ is bounded above and below by fixed multiples
of $\|Z-Z'\|_{\mu_M}^2$.
\end{lemma}
The constants and the induced equivalence $\Psi_t^f\asymp\Psi_t^z$ are proved
in Lemma~\ref{lem:f-consensus-equivalence}.

\subsection{Global KD Contraction}
For each frozen range-restricted logit kernel define
\begin{equation}
\begin{split}
W_i&:=\Theta_i^\dagger+\gamma_0(I-\Pi_i^z),\\
\mathcal V_W(Z)&:=\min_{c\in\mathcal H_z}
\frac1N\sum_i\|Z_i-c\|_{W_i}^2,
\end{split}
\label{eq:main-quotient}
\end{equation}
where $\gamma_0>0$. The metric is positive definite and uniformly equivalent
to arithmetic logit disagreement. It also satisfies $W_i\Theta_i=\Pi_i^z$.
We write
$\Psi_t^z:=N^{-1}\sum_i\|Z_i^t-\bar Z^t\|_{\mu_M}^2$ for the unweighted
logit disagreement.

\begin{theorem}[Random-edge KD contraction with defect]
\label{thm:main-contraction}
Let $G_t=N^{-1}\sum_i\|g_i^t\|_\mu^2$. For a connected graph with no
isolated peer, under A1, A4, and the KD
clause of A10, there are constants $r_0,C_G,C_B,C_\zeta,C_\sigma,C_A>0$,
independent of $\eta$ and $T$, such that
\begin{equation}
\begin{split}
\mathbb E_t\mathcal V_W(Z^{t+1})
\le{}&(1-r_0\eta)\mathcal V_W(Z^t)+C_G\eta G_t
+C_B\eta B_f^2\\+&C_\zeta\eta\zeta_f^2
+C_\sigma\eta^2+C_A\eta A_t^f,
\end{split}
\label{eq:main-consensus-recursion}
\end{equation}
with coefficients:
\begin{align*}
C_G &= R_G K_{\rm pre} ;\quad C_B = \frac{4 M_W d_{\max} \nu_{\rm KD, B} \alpha \tau}{N(p_{\min}/\tau)\lambda_2 m_W} + R_B K_{\rm pre} \\
C_\zeta &= R_\zeta K_{\rm pre}; \quad C_\eta = R_0 K_{\rm pre} + \Sigma_0;\quad C_A = R_A \alpha \tau K_{\rm pre}
\end{align*}
provided $\eta\le\eta_{\rm KD}$ and $\gamma_{\rm KD}$ is below the explicit
absorption threshold in Lemma~\ref{lem:f-global-contraction}. Moreover,
\begin{equation}
r_0=\frac{\alpha p_{\min}\lambda_2(\mathcal L)}{2M_W|E|},
\qquad
m_WI\preceq W_i\preceq M_WI.
\label{eq:main-rate}
\end{equation}
\end{theorem}
The term $C_G\eta G_t$ is critical. A local CE step is a deterministic
peer-specific force and enters consensus at first order. Under non-IID data it
cannot be relabelled as $O(\eta^2)$. The term $C_\zeta\zeta_f^2$ in the final
neighbourhood records this persistent task heterogeneity.

\subsection{Task Descent and Staleness}
Uniform edge sampling activates peer $i$ with probability
$q_i=d_i/|E|$. Define the activation-corrected task potential
\begin{equation}
\mathcal R_t:=\frac1N\sum_iq_i^{-1}
\big(F_i(p_i^t)-F_i^\star\big).
\label{eq:main-task-potential}
\end{equation}
The inverse activation weights make the conditional edge expectation a
uniform peer average even when the graph is not regular.

\begin{lemma}[Functional task descent]
\label{lem:main-descent}
Under A2, A4, A6, and the task clause of A10, there are constants
$a_F,b_F,D_B,D_\zeta,D_\sigma,D_A>0$ such that
\begin{equation}
\begin{split}
\mathbb E_t\mathcal R_{t+1}
\le{}&\mathcal R_t-a_F\eta G_t+b_F\eta\mathcal V_W(Z^t)
+D_B\eta B_f^2\\&+D_\zeta\eta\zeta_f^2+D_\sigma\eta^2+D_A\eta A_t^f,
\end{split}
\label{eq:main-task-recursion}
\end{equation}

where $a_F>0$ when the task margin, including the factor $(1-\alpha)$ and
the cross-kernel residual constants from A4, is positive; the theorem assumes
$0\le\alpha<1$.
Furthermore,
\begin{equation}
\|\nabla_fF(\bar p^t)\|_\mu^2
\le 2G_t+2L_f^2\Psi_t^f.
\label{eq:main-global-gradient}
\end{equation}
\end{lemma}
Lemma~\ref{lem:f-descent} proves these statements directly through
$\langle g_i,K_i g_i\rangle$; no $\eta$-independent conditioning error is
absorbed into an $O(\eta)$ remainder.

\begin{lemma}[Simplex-bounded staleness]
\label{lem:main-staleness}
If the EMA teacher is at most $A$ versions behind, then
\begin{equation}
A_t^f\le \bar A^f
=c_{\rm step}\left(\min\left\{A,\frac1{1-\beta}\right\}\right)^2\eta^2.
\label{eq:main-staleness}
\end{equation}
\end{lemma}
The proof is Lemma~\ref{lem:f-staleness}. It uses the simplex diameter and does
not feed consensus back into the lag bound.

\subsection{Coupled Bound}
\begin{lemma}[Strict coupled descent]
\label{lem:main-lyapunov}
Suppose there exists $\lambda>0$ such that
\begin{equation}
a_F-\lambda C_G>0,\qquad \lambda r_0-b_F>0.
\label{eq:main-small-gain}
\end{equation}
For $\mathcal L_t:=\mathcal R_t+\lambda\mathcal V_W(Z^t)$, Equations
\ref{eq:main-consensus-recursion} and \ref{eq:main-task-recursion} imply
\begin{equation}
\mathbb E_t\mathcal L_{t+1}
\le \mathcal L_t-c_G\eta G_t-c_V\eta\mathcal V_W(Z^t)
+C\eta(B_f^2+\zeta_f^2)+C'\eta^2.
\label{eq:main-lyapunov}
\end{equation}
for constants $c_G,c_V>0$. The staleness contribution is absorbed into
$C'\eta^2$ by Equation~\ref{eq:main-staleness}.
\end{lemma}
See Lemma~\ref{lem:f-lyapunov}. Condition~\ref{eq:main-small-gain} states the
required balance between task drift and KD mixing; it is not implied by a small
step size.

\begin{theorem}[Main convergence result]
\label{thm:main_convergence}
Under A1--A10, $0<\alpha<1$, a connected graph with no isolated peers, the
frozen-kernel condition in A4, the alignment margins in A10,
Condition~\ref{eq:main-small-gain}, and $0<\eta\le\eta_{\max}$, there are
constants $c,C_0,C_1,C_2,C_3>0$ such that
\begin{equation}
\begin{split}
\frac1T\sum_{t=0}^{T-1}
&\left[\mathbb E\|\nabla_fF(\bar p^t)\|_\mu^2+c\Psi_t^f\right]
\le \frac{C_0}{\eta T}+C_1\eta+C_2B_f^2\\&\qquad \qquad \qquad \qquad \qquad +C_3\zeta_f^2\\
\text{with}&\quad K_* = \max\left\{ \frac{2}{c_G}, \frac{2L_f^2 C_{fz}/m_W + c C_{fz}/m_W}{c_V} \right\}\\
&C_0 = K_* \mathcal{L}_0;\quad 
C_1 = K_* K_\eta;\\
&C_2 = K_* K_B;\quad 
C_3 = K_* K_\zeta.
\end{split}
\label{eq:main-theorem}
\end{equation}
\end{theorem}
The transient constant contains the initial task potential and disagreement as
$C_1$ contains stochastic-gradient, curvature, and bounded-lag constants. A
stable-kernel extension holds only when its metric/projector drift coefficient
is strictly smaller than the contraction margin.\\
The complete proof is Theorem~\ref{thm:fa}. The additional
$O(\zeta_f^2)$ term is unavoidable for the implemented non-IID local CE
updates: even with $B_f=0$ and zero stochastic noise, two peers with different
quadratic objectives can have a nonzero consensus fixed point independent of
$\eta$. The theorem therefore separates stochastic discretization, capacity,
and deterministic task heterogeneity rather than assigning all three effects
to the step size.

\section{Experiments}
\label{sec:experiments}
The experiments test the observable implications of
Theorem~\ref{thm:main_convergence}. They are not designed as a benchmark against
specialized decentralized optimizers, so we are not trying to beat the benchmark or other models. The completed campaign contains 25
trajectories and records function-space consensus, a finite-reference
stationarity proxy, rate analysis, and task accuracy every 1,000 edge activations.\\

We perform the experiments to provide insights into the consensus error convergence and error neighbourhood of gossip P2P KD under the real world setting. This way, we can tell if the proposed theoretical convergence rate of $\mathcal{O}(1/\eta T)$ is justified, and further, if the bound for convergence rate is tight. In addition, does the stationary proxy $\frac{1}{T}||\nabla F||^2$ decreases during the event-indexed transient. Finally, will heterogeneous settings converge in the same
qualitative sense even though a parameter-average D-SGD control is undefined
when peer dimensions and architecture differs. 

\begin{figure}[H]
    \centering
     \caption{Task accuracy for the
width-heterogeneous Setting B with $\eta=0.05$, seed $0$.}
    \includegraphics[width=0.98\linewidth]{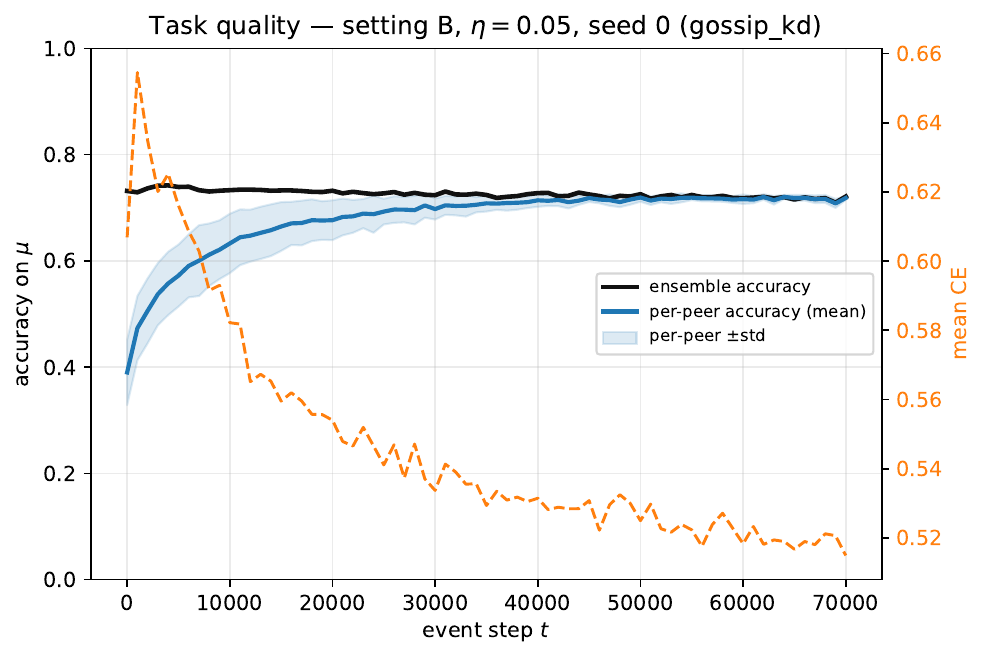}
    \label{fig: Task Acc for setting B}
\end{figure}
\begin{table*}[t]
\centering
\caption{Completed main KD runs. $T$ is the actual stored horizon,
$\widehat b_\Psi$ and $\widehat b_G$ are final-60\% $1/T$-regression
intercepts. $\widehat G_T$ is Equation~\ref{eq:exp-rate-model-explained}, and
$\widetilde p_{\rm dyn}$ is the late-transient median.}
\label{tab:exp-main}
\setlength{\tabcolsep}{6pt}
\begin{tabular}{@{}lrrrrrrrr@{}}
\toprule
Set & Seed & $T$ & $\widehat\Psi_T^f$ & $\widehat b_\Psi$ &
$\widehat G_T$ & $\widehat b_G$ & $\widetilde p_{\rm dyn}$ & Ens. Accuracy\\
\midrule
A & 0 & 70,000& $3.739\!\times\!10^{-4}$ & $3.868\!\times\!10^{-4}$ & 0.330& 0.316 & 1.03 & 0.729\\
A & 1 & 70,000& $3.522\!\times\!10^{-4}$ & $2.793\!\times\!10^{-4}$ & 0.397& 0.389 & 1.02 & 0.643\\
B & 0 & 70,000& $3.326\!\times\!10^{-4}$ & $3.983\!\times\!10^{-4}$ & 0.331& 0.316 & 1.03 & 0.721\\
B & 1 & 70,000& $2.809\!\times\!10^{-4}$ & $2.714\!\times\!10^{-4}$ & 0.399& 0.391 & 0.99 & 0.622\\
C & 0 & 70,000& $2.882\!\times\!10^{-4}$ & $4.081\!\times\!10^{-4}$ & 0.337& 0.323 & 1.21 & 0.726\\
C & 1 & 70,000& $4.101\!\times\!10^{-4}$ & $2.782\!\times\!10^{-4}$ & 0.399& 0.398 & 1.90 & 0.603\\
D & 0 & 70,000& $3.688\!\times\!10^{-4}$ & $8.264\!\times\!10^{-4}$ & 0.472& 0.526 & 0.88& 0.535\\
D & 1 & 70,000& $3.100\!\times\!10^{-4}$ & $3.155\!\times\!10^{-4}$ & 0.642& 0.729 & 0.91& 0.329\\
\bottomrule
\end{tabular}
\end{table*}
\subsection{Protocol and Observables}
\label{sec:exp-setup}

We train $N=20$ CIFAR-10~\cite{krizhevsky2009learning_cifar_10} peers whose private shards follow a Dirichlet label
partition with concentration $\gamma_{\rm part}=0.3$. Each seed fixes a connected
$\operatorname{ER}(0.3)$ graph. Seeds 0 and 1 have 51 and 58 edges and spectral
gaps 1.232 and 1.822, respectively. Setting A contains 20 medium-width ResNets.
Settings B and C use width rosters $7$s/$7$m/$6$l and
$7$xs/$7$m/$6$xl for ResNets. Setting D mixes ResNet-18, MobileNetV2~\cite{dong2020mobilenetv2, sandler2018mobilenetv2}, and ShuffleNetV2~\cite{ma2018shufflenet}.
Every model receives 1,500 local CE pretraining steps. Gossip then runs for
70,000 events with a live teacher, KD weight $\alpha=0.5$, temperature $\tau=4$,
batch size 64, and no gossip momentum. \\

Main runs use $\eta=0.05$ and two seeds;
the seed-0 sweep uses $\eta\in\{0.0125,0.025,0.05,0.1\}$. Appendix~\ref{app:experiments} gives the complete configuration and run
inventory, including all of the results of the experiments alongside visualization of those experiments. The complete code system, experiment log and README are in code and data supplement for reproduction.

The reference measure $\mu$ is the empirical measure on a fixed 2,000-image
CIFAR-10 test subset. Logits are clipped at 10 and parameter
gradients at norm 5. These controls instantiate the bounded finite-support
setting used by A1 and the moment bounds in A4, while the experiments measure
the consequences of A4 and A10 rather than treating either assumption as an
empirical conclusion.

At event $t$, the consensus statistic is
\begin{equation}
\widehat\Psi_t^f=\frac1N\sum_{i=1}^N
\left\|p_i^t-\bar p^t\right\|_\mu^2,
\qquad \bar p^t=\frac1N\sum_{i=1}^Np_i^t.
\label{eq:exp-consensus}
\end{equation}

\subsection{Rate Diagnostics}
\label{sec:exp-rate-explanation}
\paragraph{Decay (Consensus) Speed}
Figure~\ref{fig:main-rate} tests the transient term in
Theorem~\ref{thm:main_convergence}. For a fixed step size (learning rate), the theorem predicts that the time-averaged stationarity term approaches a nonzero neighbourhood at
order $\mathcal O(1/T)$. We represent this behaviour by
\begin{equation}
\widehat G_T \approx \widehat b_G + aT^{-p},
\label{eq:exp-rate-model-explained}
\end{equation}
where $\widehat G_T$ is the sampled prefix average of the stationarity proxy,
$\widehat b_G$ is its fitted terminal neighbourhood, and
$\widehat G_T-\widehat b_G$ is the part that has not yet decayed. The reference
order is $p=1$. A value $p>1$ describes faster decay and is also compatible
with the upper bound. Equation~\ref{eq:exp-rate-model-explained} is a diagnostic model for the finite trajectory, not an equality asserted by the theorem. As long as $p>1$, the theoretical convergence bound is correct and the closer $p$ is to $1$, the tighter the convergence rate bound is.\\

$E_T^+$ measures how far the running stationarity statistic remains above its
terminal neighbourhood.
\begin{equation}
E_T^+=\max\{\widehat G_T-\widehat b_G,\varepsilon_b\},
\label{eq:exp-regularized-excess-explained}
\end{equation}
where $\varepsilon_b$ is the estimated tail-noise level. The maximum prevents random floor crossings from
becoming undefined on a logarithmic axis. If the excess follows $a/T$, its
log--log curve is approximately a straight line with slope $-1$, parallel to
the dashed $1/T$ reference. A shallower curve corresponds to $p<1$, while a
steeper curve corresponds to $p>1$. In Figure~\ref{fig:main-rate}, the curve is
shallow during the early trajectory and steepens near the end. The global fit
$\widehat p=1.26$ consequently describes the whole trajectory perfectly as a
$1/T$ regime.

\paragraph{A direct check of inverse-time scaling.}
The second panel of Figure~\ref{fig:main-rate} plots the signed scaled excess
\begin{equation}
Q_T=T(\widehat G_T-\widehat b_G).
\label{eq:exp-scaled-excess-explained}
\end{equation}
If $\widehat G_T-\widehat b_G\approx a/T$, then $Q_T\approx a$ and the curve forms
a plateau. A rising curve indicates decay slower than $1/T$, and a falling
curve indicates decay faster than $1/T$. Keeping the statistic signed also
shows when stochastic observations cross the fitted neighbourhood. In
Figure~\ref{fig:main-rate}, $Q_T$ rises during the early slow-decay phase,
reaches a maximum, and then turns downward. The turnover marks the transition from $p<1$ sharply toward $p\geq1$.

\paragraph{The decay rate at each stage of training.}
The dynamic exponent $\widehat p_{\rm dyn}(T)$ is the negative slope of
$\log E_T^+$ against $\log T$ within a moving window of seven recorded
snapshots. For the displayed Setting-C
run, the late median is $1.26$ and the interquartile range is
$[0.74,1.77]$. The local exponent approaches and then exceeds one in the same
region where Panel (b) stops rising and turns downward. The rate is slow at the early stage, but it recovers to a rate faster than $1/T$ in later rounds. The more detailed experiment results and analysis for this part can be seen in Appendix~\ref{app:experiments}.

\begin{figure*}[t]
\centering
\caption{Full-horizon stationarity-rate diagnostic for Setting B, seed 1.
The late resolvable transient has median
$\widetilde p_{\rm dyn}=1.26$ with interquartile range $[0.74,1.77]$. The three
panels show log--log excess decay, signed $T$-scaled excess, and the dynamic
exponent defined in Equation~\ref{eq:exp-rate-model-explained}.}
\includegraphics[width=0.98\textwidth]{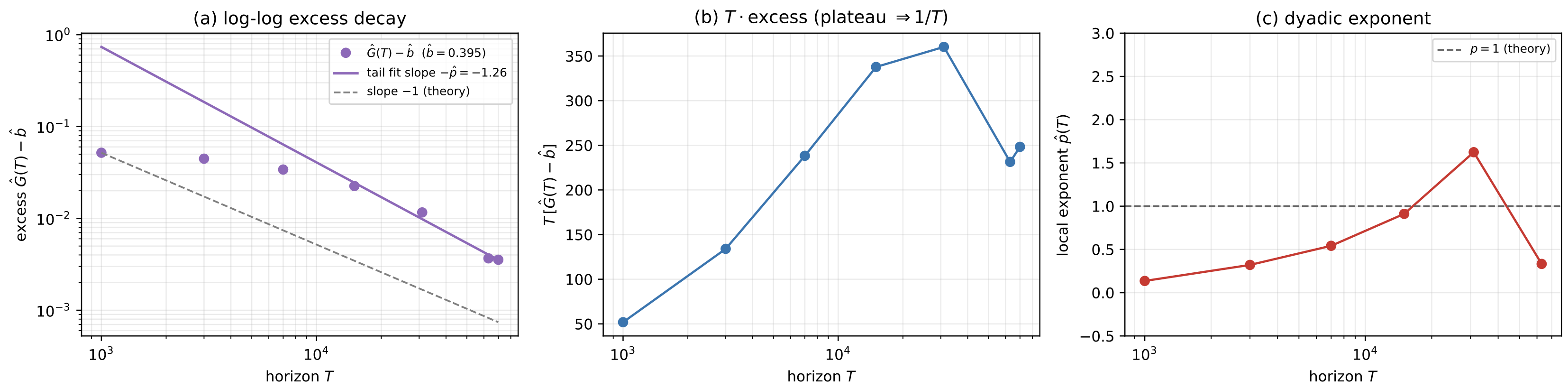}
\label{fig:main-rate}
\end{figure*}

\subsection{Empirical Validation of the Bound}
\label{sec:exp-validation}
\paragraph{KD-induced contraction.}
Across the eight main KD runs, $\widehat\Psi_t^f$ contracts by
$40.2$--$61.0\times$ from the post-pretraining state. The terminal values are
$2.81\times10^{-4}$ to $4.10\times10^{-4}$. This contraction holds for the
homogeneous, width-heterogeneous, and mixed-family rosters, matching the common
output-space mechanism in Theorem~\ref{thm:main-contraction}. In the four
isolated controls, disagreement remains at
$1.36\times10^{-2}$ to $1.51\times10^{-2}$. At 70,000 events the seed-0 KD
values are therefore $36$--$48\times$ smaller. This control isolates peer
exchange as the source of consensus contraction; it is not an optimizer
benchmark.
\begin{figure}[H]
    \centering
      \caption{Showcase for full event function-space
consensus trace for Setting C. This corresponds to the two observables in
Theorem~\ref{thm:main_convergence}.}
    \includegraphics[width=0.98\linewidth]{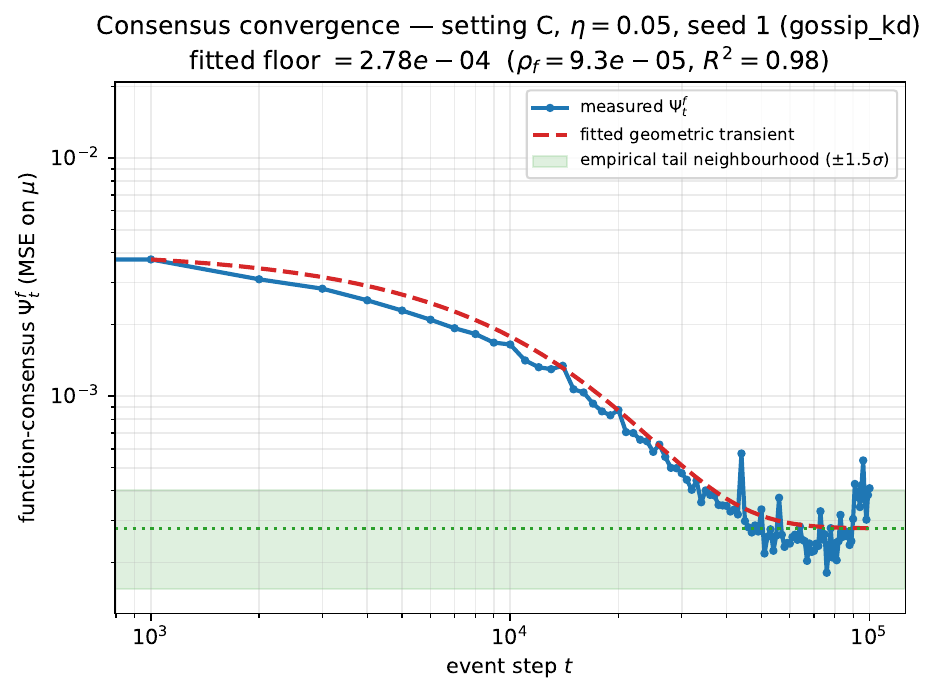}
    \label{fig:Set C convergence showcase}
\end{figure}

\paragraph{Stationarity and rate.}
For the eight Setting A--D main runs, $\widetilde p_{\rm dyn}$ ranges from 0.88 to 1.90, with
four values in $[0.88,1.03]$, the two Set-C values at 1.21 and 1.90, and two Set-D values at 0.88 and 0.91. These
measurements support decay at the $1/T$ order or faster over the resolvable
transient predicted by Theorem~\ref{thm:main_convergence}. 
\begin{figure}
    \centering
     \caption{Showcase for instantaneous and sampled-prefix
stationarity statistics for Setting B. correspond to the two observables in
Theorem~\ref{thm:main_convergence}.}
    \includegraphics[width=0.98\linewidth]{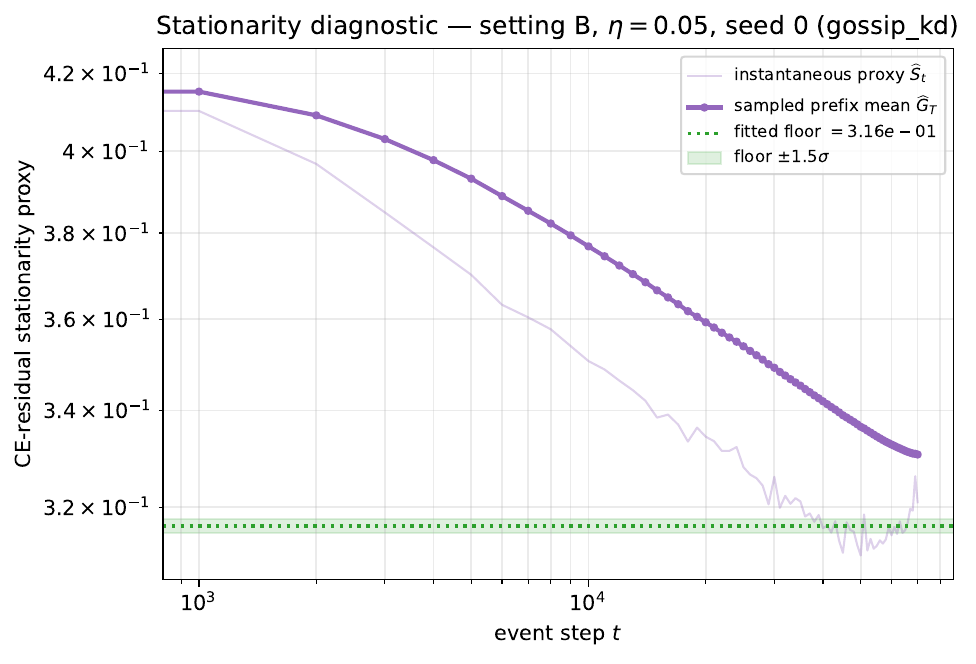}
    \label{fig:stationary showcase under setting B}
\end{figure}
\paragraph{Neighbourhood signatures.}
The four-point seed-0 sweep exhibits the finite-horizon competition in the
bound. For all of the settings from A--D, the fitted consensus floor decreases from
$(1.10,1.17,1.25)\times10^{-3}$ at $\eta=0.0125$ to
$(3.87,3.98,4.08)\times10^{-4}$ at $\eta=0.05$, then rises slightly to
$(4.04,4.20,4.23)\times10^{-4}$ at $\eta=0.1$. Small $\eta$ retains the
$1/(\eta T)$ transient at the fixed horizon, while large $\eta$ exposes the
discretization neighbourhood. At every swept step size the shared-skeleton
floor ordering is A$<$B$<$C, consistent with an architecture-dependent
additive term. The sweep does not separately identify the oracle quantities
$B_f^2$ and $\zeta_f^2$; doing so would require direct reachable-class probes
rather than a regression on training trajectories.

\paragraph{Task quality and controls.}
The accuracy panel shows that consensus is not obtained by collapsing to a
uniform predictor. The main KD ensemble accuracies span 0.535--0.729 for seed 0 and remain well above chance in every roster. We report D-SGD only for Setting A, where parameter averaging is defined. Its use of momentum 0.9 differs from the analysed gossip phase, so it remains an appendix mechanism reference and is not used for a performance claim. Lastly, the more detailed experiment results and analysis for this part can be seen in Appendix~\ref{app:experiments}.
\section{Conclusion}
\label{sec:conclusion}

We studied asynchronous random-edge peer-to-peer knowledge distillation for
serverless networks whose peers need not share a parameter space. The analysis
places all peers in the common prediction space $L^2(\mu,\mathbb R^C)$, treats
each KD exchange as a kernel-weighted mixing step, and couples quotient-space
contraction with descent of the task objective. Under the stated regularity and
realizability conditions,
Theorem~\ref{thm:main_convergence} gives an
$\mathcal O(1/(\eta T))$ transient plus an
$\mathcal O(\eta)+\mathcal O(B_f^2)+\mathcal O(\zeta_f^2)$ neighbourhood,
separating discretization, capacity, and persistent task heterogeneity. The
CIFAR-10 diagnostics are consistent with this prediction as KD reduces
function-space disagreement by $40.2$--$61.0\times$ across all experiments,
shared-skeleton runs show late-transient exponents near or above one, and the
step-size sweep exhibits the predicted transient--neighbourhood tradeoff.

\subsubsection{Limitations and Future Work}
\label{sec:limitations-future-work}

The theorem is conditional on stable kernels, bounded staleness, A4 for the
parameter-to-logit bridge, A10 for restricted alignment. These assumptions identify where the proof uses network geometry,
but they are not directly automatic for arbitrary feature learning or aggressive step
sizes.
The experiments also estimate a combined terminal neighbourhood, not the
individual $\mathcal O(\eta)$, $B_f^2$, and $\zeta_f^2$ contributions, and they
use a finite reference support for both KD and evaluation. Future work should
measure A4/A10 residuals during training, separate the KD and held-out
evaluation supports, as well as estimate $B_f$ through reachable-class probes,

\newpage
\bibliography{references}
\clearpage

\appendix
\setcounter{section}{0}
\setcounter{table}{0}
\renewcommand{\thesection}{\Alph{section}}
\setcounter{equation}{0}
\renewcommand{\theequation}{\Alph{section}\arabic{equation}}
\setcounter{theorem}{0}\setcounter{lemma}{0}
\renewcommand{\thetheorem}{\Alph{section}\arabic{theorem}}
\renewcommand{\thelemma}{\Alph{section}\arabic{lemma}}
\renewcommand{\theproposition}{\Alph{section}\arabic{proposition}}

\begin{center}\Large\bfseries Appendix\end{center}
\noindent This appendix is self-contained. Appendix~\ref{app: detailed assumptions}
states Assumptions A1--A10 in full. The proof then proceeds as a sequence of
named lemmas and propositions culminating in Theorem~\ref{thm:fa}. In the end, it details about all of the experiments setups and results with comprehensive statistic analysis.

\section{Notations Clarification}
\label{app: whole notation}

This section details every mathematical notation, variable, operator, and constant used throughout the proof of convergence in the federated knowledge distillation appendix.

\subsection{Network and Graph Topology}
\begin{itemize}
    \item $N$: Total number of peers (nodes) in the network.
    \item $i, j, k$: Indices representing specific peers.
    \item $E$: The set of undirected edges representing communication links in the graph.
    \item $\vec{E}$: The set of directed edges (both orientations for every undirected edge).
    \item $e, e_t$: A specific edge that is active at time step $t$.
    \item $d_i$: The degree (number of neighbors) of peer $i$.
    \item $d_{\max}$: The maximum degree across all peers in the graph.
    \item $\mathcal{L}$: The unweighted graph Laplacian matrix, defined as $\mathcal{L} = D - A$ (Degree matrix minus Adjacency matrix).
    \item $\lambda_1, \lambda_2, \dots, \lambda_{\max}$: The eigenvalues of the Laplacian $\mathcal{L}$. The algebraic connectivity ($\lambda_2$) dictates the speed of information mixing across the graph.
    \item $q_i$: The probability that peer $i$ is activated in a given event, defined under uniform edge sampling as $d_i / |E|$.
\end{itemize}

\subsection{Time and Optimization Parameters}
\begin{itemize}
    \item $t$: Discrete event time step (counts individual edge activations).
    \item $T$: Total number of events/steps in the horizon.
    \item $\eta$: The base parameter-space step size (learning rate) for SGD.
    \item $\alpha$: The knowledge distillation (KD) balance parameter ($0 < \alpha < 1$). Regulates the trade-off between the local task loss and the KD loss.
    \item $\tau$: The softmax temperature scalar used to scale logits before applying the softmax function.
    \item $w$: The effective step size for the KD logit contraction, defined exactly as $w = \eta \alpha \tau$.
\end{itemize}

\subsection{Data, Functions, and Prediction Spaces}
\begin{itemize}
    \item $\mu$: The reference measure. A finite evaluation dataset over which all functional metrics (logits, predictions) are measured.
    \item $C$: The number of classes in the classification task.
    \item $\theta_i$: The neural network parameter vector (weights and biases) for peer $i$.
    \item $\theta_i^0$: The parameters of peer $i$ at initialization.
    \item $\Delta \theta_i$: The parameter update applied to peer $i$ during a training step.
    \item $Z_i, Z$: The logits (pre-softmax outputs) evaluated over the reference measure $\mu$.
    \item $Z_i^t$: Logits of peer $i$ specifically at time $t$.
    \item $\bar{Z}$: The mean logits across the entire network, $N^{-1}\sum_i Z_i$.
    \item $P_i, p_i$: The prediction probabilities, computed as $P_i = \operatorname{softmax}(Z_i/\tau)$.
    \item $P_j^T, p_j^{\mathrm{EMA}}$$: The teacher target predictions from peer $j . Can be the live prediction or an Exponential Moving Average (EMA).
    \item $\bar{p}^t$: The mean prediction probability across all peers at time $t$.
    \item $p_{\min}$: A strictly positive lower bound on all class probabilities (enforced by Assumption A1).
    \item $q^\star$: The hypothetical optimal target prediction distribution for the global task.
\end{itemize}

\subsection{Jacobians and Neural Tangent Kernels (NTK)}
\begin{itemize}
    \item $J_i^z$: The Jacobian matrix of the logits with respect to the parameters, $D_{\theta_i} Z_i(\theta_i)$.
    \item $J_i^p$: The Jacobian of the prediction probabilities with respect to the parameters.
    \item $D_Z P$: The Jacobian of the softmax mapping from logits to probabilities.
    \item $\Theta_i$: The Neural Tangent Kernel (NTK) in \emph{logit space}, defined as $\Theta_i = J_i^z (J_i^z)^*$.
    \item $\Theta_i^0$: The initial logit-space NTK at $\theta_i^0$.
    \item $K_i$: The Neural Tangent Kernel in \emph{prediction space}, defined as $K_i = J_i^p (J_i^p)^*$.
    \item $H_i$: A cross-Jacobian operator mapping the gradient from the prediction-space KD loss into a parameter update, $H_i = \tau J_i^p (J_i^z)^*$.
\end{itemize}

\subsection{Tangent Spaces and Projections}
\begin{itemize}
    \item $\mathcal{H}_z$: The full ambient vector space of evaluated logits.
    \item $T_i^z$: The \emph{logit tangent space} (the column space/range of $J_i^z$ or $\Theta_i$). Represents logit directions the model is physically capable of moving in.
    \item $\Pi_i^z$: The orthogonal projection matrix onto the logit tangent space $T_i^z$.
    \item $I - \Pi_i^z$: The orthogonal projection onto the \emph{normal space} (directions the model cannot move).
    \item $T_i^p, \Pi_i^p$: The \emph{prediction tangent space} and its orthogonal projection matrix.
    \item $M(P)$: The softmax metric tensor defined as $M(P) = \operatorname{diag}(P) - P P^\top$.
\end{itemize}

\subsection{Losses, Gradients, and Residuals}
\begin{itemize}
    \item $F_i(p)$: The local task risk (e.g., Cross-Entropy loss) for peer $i$.
    \item $F(p)$: The global task risk, averaged across all peers $N^{-1}\sum F_i(p)$.
    \item $F_i^\star, F^\star$: The infimum (minimum achievable value) of the local and global task risks.
    \item $g_i^t$: The functional gradient of the local task risk with respect to the predictions, $g_i^t = \nabla_f F_i(p_i^t)$.
    \item $D_{\mathrm{KL}}(P'\|P)$: The Kullback-Leibler divergence between two probability distributions.
    \item $d_{ij}$: The logit difference (directed edge pull) between peer $i$ and peer $j$, $d_{ij} = Z_i - Z_j$.
    \item $\delta_i^t$: The task residual defect. The difference between the actual pushforward of the network and the ideal functional descent direction.
    \item $r_i^z, r_i^p$: The Taylor series remainder (curvature) vectors for logits and predictions.
    \item $\xi_{i, e}, \xi_i^z, \xi_i^p$: Martingale difference noise vectors representing the stochasticity introduced by random edge sampling.
\end{itemize}

\subsection{Consensus and Error Metrics (The Observables)}
\begin{itemize}
    \item $\Psi_t^f$: Function-space consensus error. The variance of predictions across peers: $\frac{1}{N}\sum_i \mathbb{E}\|p_i^t - \bar{p}^t\|_\mu^2$.
    \item $\Psi_t^z$: Logit-space consensus error. The variance of logits across peers: $\frac{1}{N}\sum_i \mathbb{E}\|Z_i^t - \bar{Z}^t\|_\mu^2$.
    \item $\mathcal{V}_W(Z)$: The \emph{quotient consensus energy}, the primary metric for analyzing contraction: $\min_c \frac{1}{N}\sum_i \|Z_i - c\|_{W_i}^2$.
    \item $c_W$: The weighted centroid of logits that minimizes the quotient energy $\mathcal{V}_W(Z)$.
    \item $Q_i$: The logit deviation from the weighted centroid, $Q_i = Z_i - c_W$.
    \item $G_t$: The average squared norm of the functional task gradients, $\frac{1}{N}\sum_i \|g_i^t\|^2$.
    \item $B_f^2$: The \emph{representational capacity floor}. The irreducible minimum distance between the optimal function $q^\star$ and the functions the network architectures can represent.
    \item $\zeta_f^2$: The data heterogeneity bound. Measures how much the local dataset gradients vary from the global average.
    \item $A_t^f$: The teacher lag error. Evaluates how ``stale'' the EMA teacher predictions are compared to live predictions.
    \item $\mathcal{R}_t$: The activation-corrected task potential (a weighted sum of local task sub-optimalities).
    \item $\mathcal{L}_t$: The strict Lyapunov function combining task risk and consensus error, $\mathcal{L}_t = \mathcal{R}_t + \lambda \mathcal{V}_W(Z^t)$.
\end{itemize}

\subsection{Bounding Constants and Multipliers}
\begin{itemize}
    \item $\kappa_-, \kappa_+$: The strict lower and upper bounds on the non-zero eigenvalues of the restricted logit NTK $\Theta_i$.
    \item $\underline{\kappa}_p, \overline{\kappa}_p$: The lower and upper spectral bounds for the restricted prediction kernel $K_i$.
    \item $L_J, L_K$: Lipschitz continuity constants for the Jacobian matrices and the NTK.
    \item $L_f$: Smoothness constant (Lipschitz gradient) for the functional task risk $F_i$.
    \item $H_{\max}$: The absolute upper bound on the operator norm of the cross-Jacobian $H_i$.
    \item $\gamma_{\mathrm{KD}}, \nu_{\mathrm{KD, B}}$: Constants bounding the non-realizable (normal) component of the KD pull (from Assumption A10).
    \item $\gamma_{\mathrm{task}}, \nu_{\mathrm{task, \Psi}}, \nu_{\mathrm{task, B}}$: Constants bounding the non-realizable (normal) component of the task gradient.
    \item $W_i$: The pseudo-inverse metric matrix, used to precondition the quotient energy, $W_i = \Theta_i^\dagger + \gamma_0(I - \Pi_i^z)$.
    \item $m_W, M_W$: The absolute minimum and maximum eigenvalues of $W_i$.
    \item $a_W, b_W, \varepsilon_W$: Geometric contraction constants governing the speed of consensus.
    \item $D_{b,t}, D_{\xi,t}$: Aggregate variance bounds for deterministic perturbations (curvature, task, lag) and stochastic noise.
    \item $C_G, C_B, C_\zeta, C_\eta, C_A, c_0$: Aggregate algebraic coefficients governing the recursive consensus drift.
    \item $a_F, b_F, D_B, D_\zeta, D_\eta, D_A$: Aggregate algebraic coefficients governing the task descent recursion.
    \item $c_G, c_V, K_B, K_\zeta, K_\eta$: The ultimate algebraic coefficients in the strict Lyapunov recursion representing the combined system drift.
    \item $C_0, C_1, C_2, C_3$: The final macroscopic constants in Theorem 1's $\mathcal{O}(1/T)$ convergence bound. $C_0$ scales the transient time rate, $C_1$ scales the step-size error, and $C_2, C_3$ scale the persistent capacity and data heterogeneity floors.
\end{itemize}

\section{Explicit Analytical Coefficients of the Convergence Proof}
\label{app: whole coefficients}

This section explicitly defines every composite coefficient used in the convergence proof (Lemmas 7, 8, 9, and Theorem 1), tracing their exact dependence on the structural network and optimization parameters.

\subsection{Consensus Recursion Coefficients (Theorem 1 in Main Paper, Lemma 7 in Appendix)}
The global contraction of the quotient energy $\mathcal{V}_W(Z)$ follows the discrete-time Lyapunov recursion:
 \[
 \begin{split}
\mathbb{E}_t \mathcal{V}_W(Z^{t+1}) \le (1 - c_0 \eta) \mathcal{V}_W(Z^t) + C_G \eta G_t + C_B \eta B_f^2 + 
\\C_\zeta \eta \zeta_f^2 + C_\eta \eta^2 + C_A \eta A_t^f.
\end{split}
 \]
The transient contraction rate $c_0$ is defined as:
$$ c_0 = \frac{\alpha p_{\min}\lambda_2}{2 M_W |E|} $$
where $M_W = \max\{\kappa_-^{-1}, \gamma_0\}$. 
The perturbation prefactors $R_G, R_B, R_\zeta, R_0, R_A$ are defined by bounding the task, curvature, and lag vectors using Young's inequality:
\begin{align*}
R_G &= 9 M_W (1 - \alpha)^2 (\overline{\kappa}_p^2 + \chi_G) \\
R_B &= 9 M_W (1 - \alpha)^2 \chi_B \\
R_\zeta &= 9 M_W (1 - \alpha)^2 \chi_\zeta \\
R_0 &= \frac{3}{4} M_W L_\Theta^2 G_{\theta,4}^4 \\
R_A &= \frac{3 M_W \kappa_+ d_{\max}}{|E|} 
\end{align*}
Let $K_{\rm pre} = \left(\frac{8 d_{\max}}{a_W \alpha \tau |E|} + 2 \right)$. By absorbing the $w = \eta \alpha \tau$ factors into the stochastic edge distribution, the composite tracking coefficients are:
\begin{align*}
C_G &= R_G K_{\rm pre} \\
C_B &= \frac{4 M_W d_{\max} \nu_{\rm KD, B} \alpha \tau}{N(p_{\min}/\tau)\lambda_2 m_W} + R_B K_{\rm pre} \\
C_\zeta &= R_\zeta K_{\rm pre} \\
C_\eta &= R_0 K_{\rm pre} + \Sigma_0 \\
C_A &= R_A \alpha \tau K_{\rm pre}
\end{align*}

\subsection{Task Descent Coefficients (Lemma 8 in Appendix)}
The activation-corrected task potential $\mathcal{R}_t$ satisfies the descent inequality:
\[
\mathbb{E}_t \mathcal{R}_{t+1} \le \mathcal{R}_t - a_F \eta G_t + b_F \eta \mathcal{V}_W(Z^t) + D_B \eta B_f^2 + 
D_\zeta \eta \zeta_f^2 + D_\eta \eta^2 + D_A \eta A_t^f.
\]
The task drift explicitly counters the consensus drift. The descent coefficient $a_F$ and the corresponding error coefficients are:
\[
\begin{split}
a_F &= (1-\alpha)\left[ \underline{\kappa}_p(1-\gamma_{\rm task}) - \frac{\varepsilon_\delta}{2} \right] - \varepsilon_{\rm KD} - \eta_{\max} \frac{L_f}{2} C_{\Delta, G} \\
b_F &= (1-\alpha)\left[ \underline{\kappa}_p \nu_{\rm task, \Psi} + \frac{\chi_G}{2 \varepsilon_\delta} \right] \frac{C_{fz}}{m_W} + \alpha C_{\rm KD, V} +\\ 
&\qquad \qquad \qquad \qquad \qquad \quad \qquad \qquad \quad \eta_{\max} \frac{L_f}{2} C_{\Delta, V} \\
D_B &= (1-\alpha)\left[ \underline{\kappa}_p \nu_{\rm task, B} + \frac{\chi_B}{2 \varepsilon_\delta} \right] + \eta_{\max} \frac{L_f}{2} C_{\Delta, B} \\
D_\zeta &= (1-\alpha)\frac{\chi_\zeta}{2 \varepsilon_\delta} \\
D_A &= \alpha C_{\rm KD, A} + \eta_{\max} \frac{L_f}{2} C_{\Delta, A} \\
D_\eta &= \frac{L_f}{2} (4 R_p + 4 \Sigma_p)
\end{split}
\]
\subsection{Strict Lyapunov Coefficients (Lemma 5 in Main Paper, Lemma 9 in Appendix)}
The strict Lyapunov function is defined as $\mathcal{L}_t = \mathcal{R}_t + \lambda \mathcal{V}_W(Z^t)$. By taking a linear combination of the inequalities from Lemma 7 and Lemma 8, we obtain:
$$ \mathbb{E}_t \mathcal{L}_{t+1} \le \mathcal{L}_t - c_G \eta G_t - c_V \eta \mathcal{V}_W(Z^t) + K_B \eta B_f^2 + K_\zeta \eta \zeta_f^2 + K_\eta \eta^2. $$
The necessary small-gain condition for convergence is $b_F C_G < a_F c_0$. Assuming this holds, there exists $\lambda > 0$ yielding strictly positive descent factors:
\begin{align*}
c_G &= a_F - \lambda C_G \quad > 0 \\
c_V &= \lambda c_0 - b_F \quad > 0
\end{align*}
The remaining static error floors are linearly composed as follows:
\begin{align*}
K_B &= D_B + \lambda C_B \\
K_\zeta &= D_\zeta + \lambda C_\zeta \\
K_\eta &= D_\eta + \lambda C_\eta + (D_A + \lambda C_A)\bar{A}^f
\end{align*}
where $\bar{A}^f$ bounds the teacher staleness variance via $A_t^f \le \bar{A}^f \eta^2$.

\subsection{Final Asymptotic Rates (Theorem 1)}
The overall $\mathcal{O}(1/T)$ rate is derived by telescoping the strict Lyapunov recursion. By isolating the time-average of the functional gradient and consensus error, we bound:
$$ \frac{1}{T}\sum_{t=0}^{T-1} \left[ \mathbb{E}\|\nabla_f F(\bar{p}^t)\|_\mu^2 + c \Psi_t^f \right] \le \frac{C_0}{\eta T} + C_1 \eta + C_2 B_f^2 + C_3 \zeta_f^2. $$
The universal scaling factor $K_*$ balances the dual contraction criteria:
$$ K_* = \max\left\{ \frac{2}{c_G}, \frac{2L_f^2 C_{fz}/m_W + c C_{fz}/m_W}{c_V} \right\}. $$
Multiplying the Lyapunov errors by this global scalar yields the final macroscopic coefficients:
\begin{align*}
C_0 &= K_* \mathcal{L}_0 \\
C_1 &= K_* K_\eta \\
C_2 &= K_* K_B \\
C_3 &= K_* K_\zeta
\end{align*}

\section{Expanded Related Work and Closest Comparisons}
\label{sec:expanded-related-work}

\subsection{Federated Distillation with a Coordinator}

Knowledge distillation (KD) transfers information through soft predictions
rather than through weights~\citep{hinton2015distill}. This property makes KD a
natural communication primitive for federated learning with heterogeneous
architectures, since predictors share the class simplex even when their
parameter vectors have different dimensions. FedMD distils client models on a
public dataset and alternates local private training with public-set
distillation~\citep{li2019fedmd}. FedDF fuses client models by server-side
ensemble distillation~\citep{lin2020feddf}. Cronus studies heterogeneous
collaborative learning through black-box soft-label transfer in a setting where
robustness to malicious clients is central~\citep{chang2019cronus}. Federated
Mutual Learning, data-free heterogeneous FL distillation, and DS-FL further
show that soft-label or generated-data transfer can reduce the need for direct
parameter aggregation under non-IID data and heterogeneous clients
\citep{shen2020fml,zhu2021datafree,zhang2022dsfl}. These methods establish the
practical relevance of prediction-level communication. Their protocols still
retain a coordinator, a global aggregation step, a synchronized round
structure, or a server-side distillation stage.

FedKD is a particularly relevant communication-efficient representative of
this line~\citep{wu2022fedkd}. Its mechanism keeps a local mentor model and a
shared mentee model at each client, uses adaptive mutual distillation, and
combines this with dynamic gradient compression. The reported empirical result
is a large reduction in communication cost, up to $94.89\%$ on the tasks
studied by Wu et al., without using auxiliary public data. This differs from
the present paper in the core protocol. FedKD remains a server-based FL method:
the shared mentee participates in server aggregation, while our process samples
one graph edge at a time and exchanges only peer predictions.

\subsection{Decentralized and Serverless KD-FL Systems}

Recent work has moved KD closer to serverless or peer-level learning. FedDKD,
which we use as the verified representative of the decentralized-KD line,
decentralizes knowledge transfer among clients inside a federated procedure
\citep{gong2022feddkd}. The method is close in communication object because
clients exchange predictive information, but the analyzed protocol is not the
same asynchronous random-edge process considered here.

DeSA addresses decentralized FL with both data and model heterogeneity by
introducing synthetic anchors~\citep{huang2024desa}. Each client trains with a
regularization term and a KD term induced by the anchors. The ICML 2024 paper
uses domain-adaptation theory to motivate the anchor mechanism and reports
improved intra-domain and inter-domain accuracy over decentralized baselines.
DeSA is directly relevant for a federated-learning reviewer because it combines
serverless decentralization, heterogeneity, and KD. Its guarantee and
algorithmic object are still different from ours. DeSA uses synthetic anchors
to create a shared alignment interface, while our theorem studies live
random-edge distillation on a fixed graph and proves contraction in a common
prediction space.

FedEnD is another close system because it is serverless and peer-to-peer
\citep{martinezbeltran2026fedend}. Its mechanism is a two-stage
specialist-student protocol. Clients first train specialists on private data,
then broadcast model outputs once so that each student can distil from a
class-distribution-weighted peer ensemble. The reported results emphasize
accuracy and communication: FedEnD improves macro-F1 by up to $5.7\%$ in
pathologically skewed non-IID settings, reduces communication by $68.6\%$
relative to standard iterative averaging, and reports up to $84\%$ lower
communication than SCAFFOLD. FedEnD and this paper are complementary. FedEnD
targets one-shot decentralized ensemble distillation without auxiliary public
data. Our paper studies an asynchronous stream of random-edge KD events and
derives a conditional convergence theorem for the resulting prediction-space
process.

Other recent peer-oriented KD-FL systems reinforce the same motivation.
IMFLKD studies an incentive mechanism for decentralized FL based on KD and
uses blockchain, smart contracts, label aggregation, contribution evaluation,
and reputation incentives~\citep{ying2026_decen_fed_kd_incentive_mechanism}.
TopoMoDistill uses topology-aware neighbour aggregation and compressed
class-wise latent prototypes for decentralized multimodal FL
\citep{wu2026topology_fed_kd}. SplitGossip studies edge-based fine-tuning of
distilled language models with split gossip learning and gradient compression
\citep{khowaja2026splitgossip_gossip_distillation}. These works make peer-level
KD an active systems direction. The gap addressed here is narrower: a
serverless random-edge P2P KD process with heterogeneous reachable prediction
classes and a function-space convergence bound.

\subsection{Decentralized Optimization and the Type Mismatch}

Classical decentralized optimization analyzes consensus and descent when every
peer stores a vector in a shared parameter space. Randomized gossip averaging
and distributed subgradient methods characterize graph mixing in this common
space~\citep{boyd2006gossip}. Decentralized SGD and push-sum variants extend
the analysis to nonconvex stochastic optimization, directed communication,
changing topology, local updates, compression, and data heterogeneity
\citep{lian2017dpsgd,assran2019pushsum,koloskova2020unified,dandi2022mixing}.
These results do not apply directly when peers use different architectures,
because the parameter average $x_i\leftarrow\sum_jW_{ij}x_j$ is not defined
across different dimensions. Our analysis preserves the graph viewpoint but
changes the state variable from parameters to evaluated predictions. A KD event
then becomes a contraction step in a shared finite-support output space.

\paragraph{Implications for baselines.}
For a federated or distributed learning audience, isolated training and
homogeneous D-SGD are mechanism controls rather than a complete empirical
positioning. They test whether KD itself contracts prediction disagreement and
whether ordinary parameter gossip is available only in a homogeneous roster.
They do not compare against the closest heterogeneous FL distillation systems.
If the paper makes a broad empirical competitiveness claim, the relevant
baseline set should include DeSA, FedEnD, a FedDKD or decentralized-KD
representative, FedKD, and at least one recent heterogeneous FL distillation
method. If the paper keeps its current theory-first claim, these methods should
still appear in the closest-work table, while the experiments should be
described as diagnostics of the theorem's mechanism rather than as a SOTA
benchmark.

\section{Detailed Assumptions.}
\label{app: detailed assumptions}
\paragraph{A1 (Bounded logits).}
$\|z_i^{(t)}(x)\|_\infty\le G_z$ for all $i,t,x$,
implemented by weight decay + logit/gradient clipping. Lemma
\ref{lem:f-softmax-spectrum} proves that every class probability is bounded below,
$$p_i(x)_c\ \ge\ p_{\min}:=\frac{e^{-2G_z/\tau}}{C}\ >\ 0 .$$

\paragraph{A2 (Bounded stochastic variance).}
. The stochastic CE functional gradient has
variance $\le\sigma_f^2$ per peer (in the $\mathcal H$-metric; used in
Lemmas~\ref{lem:main-descent} and \ref{lem:f-consensus-recursion}).

\paragraph{A3 (Representational capacity floor).}
 Each peer's
reachable prediction set $\mathcal P_i:=\{\,p_\theta:\theta\in\mathbb R^{d_i}\,\}\subseteq\mathcal H$
is assumed closed, although it may be non-convex. Define the representational floor as the
explicit $L^2(\mu)$ quantity
\[
\boxed{B_f^2:=\ \frac1N\sum_{i=1}^N \mathrm{dist}_\mu\!\big(q^\star,\mathcal P_i\big)^2, \mathrm{dist}_\mu(q,\mathcal P_i)=\inf_{p\in\mathcal P_i}\|q-p\|_\mu}
\]
where $q^\star:=\arg\min_{p\in\mathcal S}F(p)$ is the \emph{constrained} task minimiser over the bounded-logit set $\mathcal S:=\{p=\mathrm{softmax}(z/\tau):\|z\|_\infty\le G_z\}$. The set $\mathcal S$ is compact and $F$ is assumed continuous on it, so $q^\star$ exists. Constraining to $\mathcal S$ is essential for classification: the unconstrained CE minimiser drives logits to infinity (one-hot), whereas on $\mathcal S$ the minimiser is finite and interior, and under label smoothing (as used in our experiments) its target is non-one-hot, making $q^\star$ a genuine interior point consistent with A1. Let $\pi_i:=\Pi_{\mathcal P_i}(q^\star)$ be the closest reachable point of peer $i$, so $\tfrac1N\sum_i\|\pi_i-q^\star\|_\mu^2 = B_f^2$. We anchor to the fixed $q^\star$ rather than a free $\inf_q$ (the latter is attained by the trivial uniform predictor every peer can represent, giving the useless $B_f=0$). $B_f=0$ iff every peer can represent the required optimum. This is a \emph{capacity} floor (approximation error to $q^\star$), not a pairwise heterogeneity measure. $B_f$ is a theoretical \emph{oracle} quantity depending on the unknown $q^\star$ and the non-convex $\mathcal P_i$. Theorem~\ref{thm:fa} uses it only as an upper-bound parameter.

\paragraph{A4 (Parameter-to-logit realizability, kernel geometry).}
Let $\mu_M=M^{-1}\sum_{r=1}^M\delta_{x_r}$ be the empirical reference measure and
$\mathcal H_z=\{Z\in\mathbb R^{M\times C}:Z_{r,:}\mathbf 1=0\}$ with the
$M^{-1}$ Frobenius inner product. Define
$Z_i(\theta_i)=\mathsf C[z_i(x_r;\theta_i)]_{r=1}^M$,
$J_i^z=D_{\theta_i}Z_i(\theta_i)$, and
$\Theta_i=J_i^z(J_i^z)^*$. Let
$T_i^z=\operatorname{range}(J_i^z)$ and $\Pi_i^z=\Pi_{T_i^z}$.
The frozen restricted NTK satisfies
$$
\kappa_-\Pi_i^z\preceq\Theta_i\preceq\kappa_+\Pi_i^z,\qquad
\|D Z_i(\theta')-D Z_i(\theta)\|\le L_\Theta\|\theta'-\theta\|,
$$
so $\ker(\Theta_i)=(T_i^z)^\perp$. This range-restricted statement is
compatible with capacity-deficient peers and asserts no infinite-dimensional
coercivity. Let $J_i^p=D_{\theta_i}p_i$ on the reference support,
$T_i^p=\operatorname{range}(J_i^p)$, $\Pi_i^p=\Pi_{T_i^p}$, and
$K_i=J_i^p(J_i^p)^*$. We require explicit prediction-space constants
$0<\underline\kappa_p\le\overline\kappa_p<\infty$ satisfying
$$
\underline\kappa_p\Pi_i^p\preceq K_i\preceq
\overline\kappa_p\Pi_i^p.
$$
For $Z_i,Z_j$, define the averaged softmax Jacobian

\begin{align}
A_{ij}&:=\int_0^1D_Z\operatorname{softmax}
\left(\frac{Z_j+s(Z_i-Z_j)}{\tau}\right)ds,
\\
P_i-P_j&=A_{ij}(Z_i-Z_j).
\end{align}

For a detached KD
teacher,
$$
\nabla_{\theta_i}\!\left[\tau^2D_{\mathrm{KL}}(P_j\|P_i)\right]
=(J_i^z)^*\!\left[\tau(P_i-P_j)\right],
$$
and therefore the random evaluated-logit recursion is

\begin{align}
\Delta Z_i
=-\eta\Big[(1-\alpha)G_i^{\mathrm{task}}
+\alpha\tau\Theta_i(P_i-P_j)\Big]+r_i^z+\xi_i^z,
\\ \mathbb E_t\xi_i^z=0.
\end{align}

Here $G_i^{\mathrm{task}}$ is the cross-kernel pushforward from the private CE
support to the reference support. We require
$\mathbb E_t\|r_i^z\|^2\le R_z\eta^4$ and
$\mathbb E_t\|\xi_i^z\|^2\le\sigma_z^2\eta^2$; these are the explicit
second-moment forms used by the proof. For ReLU networks the local remainder condition requires a fixed
activation pattern or a separately bounded activation-crossing remainder.
Kernel linearization and empirical conditioning in the relevant local regimes
are established, under their respective hypotheses, by Jacot et
al.~\cite{jacot2018ntk}, Lee et al.~\cite{lee2019wide}, Du et
al.~\cite{du2019gradient}, Allen-Zhu et al.~\cite{allenzhu2019convergence}, and
Banerjee et al.~\cite{banerjee2023ntk}. These results motivate A4 but do not
establish it for arbitrary narrow feature-learning networks.

\paragraph{A10 (Restricted task/KD alignment).}
For every $t$, almost surely conditional on $\mathcal F_t$, let
$g_i^t=\nabla_fF_i(p_i^t)$. The task residual satisfies
\begin{equation}
\begin{split}
\frac1N\sum_i\|(I-\Pi_i^{p,t})g_i^t\|_\mu^2
\le\gamma_{\rm task}\frac1N\sum_i\|g_i^t\|_\mu^2\\
+\nu_{\rm task,\Psi}\Psi_t^f+\nu_{\rm task,B}B_f^2,
\qquad 0\le\gamma_{\rm task}<1.
\end{split}
\label{eq:a10-task}
\end{equation}
For $d_{ij}^z=P_i-P_j=A_{ij}(Z_i-Z_j)$, the directed KD residual satisfies
\begin{equation}
\frac1{|\vec E|}\sum_{(i,j)\in\vec E}
\|(I-\Pi_i^z)d_{ij}^z\|_{\mu_M}^2
\le\gamma_{\rm KD}\Psi_t^z+\nu_{\rm KD,B}B_f^2.
\label{eq:a10}
\end{equation}
The transient coefficients are subject to separate absorption margins; the
capacity coefficients need only be finite. This is a restricted-angle/error
condition. It is not implied by A3 or by a PL inequality. The appendix proves
a graph-angle sufficient condition, an exact affine-class case, and the
counterexample that necessitates the relative term in
Equation~\ref{eq:a10-task}. Bounded-dissimilarity and relative-error assumptions
play analogous roles in heterogeneous and inexact optimization
\citep{li2020fedprox}, but do not prove A10 for neural prediction manifolds.
For the frozen-kernel theorem, the projectors in A10 are the fixed projectors
from A4; a time-varying feature-learning extension would require an additional
projector-drift term.

\paragraph{A5 (Softmax/KL geometry).}
. On the bounded-logit set, $\tau^2D_{\mathrm{KL}}(p'\|p)$
is two-sided comparable to the squared mean-subtracted logit gap, with explicit
constants in Lemma~\ref{lem:f-consensus-equivalence}. This lets us pass between KL-, $L^2$-probability-, and
$L^2$-logit-consensus.

\paragraph{A6 (Data heterogeneity).}
 The per-peer CE functional gradients disperse about
their mean by at most $\zeta_f^2$: $\tfrac1N\sum_i\|\nabla_{\!f}F_i(p)-\nabla_{\!f}F(p)\|_\mu^2\le\zeta_f^2$.

\paragraph{A7 (Bounded soft-target gradient).}
. The KD functional
gradient $\tau(p_i-p_j)$ and the CE functional gradient are bounded (predictions
in the simplex, probabilities $\ge p_{\min}$ by A1). No separate bounded-gradient
assumption is needed; boundedness is a consequence of A1.

\paragraph{A8 (Bounded teacher lag).}
The teacher is either live, a finite-memory EMA with support at most $A$, or an
ordinary infinite-memory EMA with decay $\beta$. Define
$A_{\rm eff}=\min\{A,(1-\beta)^{-1}\}$ for the finite-memory case and
$A_{\rm eff}=(1-\beta)^{-1}$ for the ordinary infinite-memory case. The
lag quantity is $A_t^f=N^{-1}\sum_i\mathbb E\|p_i^t-p_i^{\rm EMA,t}\|_\mu^2$.

\paragraph{A9 (Objective lower bound).}
Simply $F^\star>-\infty$ .

\section{Assumption Validity and Diagnostics}
\label{app:assumption-validity}

A4 and A10 are structural bridge assumptions. A4 connects parameter updates to
finite-support logit or prediction updates through local kernel geometry. A10
requires the task gradient and KD pull to have enough mass in the tangent
directions that each peer can realize. These assumptions are not universal
facts about finite neural networks. They are local, trajectory-dependent
conditions that make the random-edge KD recursion analyzable. The role of this
section is to state why these assumptions are mathematically standard in form,
give sufficient regimes where they hold, and define empirical diagnostics that
can audit them on stored checkpoints.

\subsection{Precedents for A4 and A10}

A4 follows the same logic as finite-support kernel analyses in the neural
tangent and lazy-training literature. The neural tangent kernel was introduced
to describe the training dynamics of wide neural networks through an evolving
kernel~\citep{jacot2018ntk}. Subsequent wide-network results show that, under
their own width, initialization, and step-size hypotheses, network outputs
evolve close to their linearization around initialization
\citep{lee2019wide,allenzhu2019convergence}. Banerjee et al. study positive
definiteness of the NTK at initialization under linear-width scaling
\citep{banerjee2023ntk}. Bai and Lee analyze quadratic and higher-order terms,
which is useful here because it marks the boundary of a purely linearized
argument~\citep{bailee2020beyond}. These results support the use of local
finite-support kernels and restricted conditioning, but they do not prove A4
for arbitrary finite-width feature-learning networks.

The range restriction in A4 also has a precedent outside neural networks. In
high-dimensional statistics, restricted eigenvalue and compatibility
conditions require curvature only on identifiable cones or structured
subspaces, rather than on the full ambient space
\citep{bickel2009simultaneous,vandegeer2009conditions}. A4 uses the same
principle in prediction space. It does not require the finite neural tangent
kernel to be coercive on all evaluated logits. It requires coercivity only on
the tangent range that the peer's parameters can move.

A10 has an analogous role to relative-error and gradient-related assumptions in
optimization. Inexact Newton methods allow an approximate linear solve when the
residual is controlled relative to the target system
\citep{dembo1982inexact,eisenstat1996choosing}. Classical descent analyses
allow directions that are not exact gradients when they maintain a nonzero
angle with a descent direction~\citep{nocedalwright2006}. In heterogeneous FL,
FedProx uses bounded dissimilarity to control the mismatch between local and
global objectives~\citep{li2020fedprox}. A10 is the function-space analogue:
it controls the part of the task or KD vector that lies outside the peer's
reachable tangent range, and it charges persistent mismatch to consensus,
capacity, and heterogeneity terms.
\begin{table*}[H]
\centering
\caption{Closest related settings. ``Reference data'' denotes public data,
synthetic anchors, generated data, proxy data, or the fixed reference support
used for distillation or diagnostics.}
\label{tab:closest-work-expanded}
\begin{tabular}{|@{}p{0.15\textwidth}|p{0.15\textwidth}|p{0.07\textwidth}|p{0.10\textwidth}|p{0.11\textwidth}|p{0.09\textwidth}|p{0.15\textwidth}|}
\hline
Work & Setting & Serverless & Model heterogeneity & Reference data & Asynchrony & Convergence theory \\
\hline
FedMD / public-set KD~\citep{li2019fedmd} &
Federated public-set distillation &
No &
Yes &
Public dataset &
Synchronous &
No random-edge P2P theorem \\\hline
\addlinespace
FedDF / data-free FL distillation~\citep{lin2020feddf,zhu2021datafree,zhang2022dsfl} &
Server-side ensemble or generated-data distillation &
No &
Yes &
Server proxy or generated data &
Synchronous &
No serverless P2P theorem \\\hline
\addlinespace
FedKD~\citep{wu2022fedkd} &
Server-based adaptive mutual KD with compression &
No &
Partly, through mentor/mentee structure &
No auxiliary public data &
Round-based &
No asynchronous random-edge theorem \\\hline
\addlinespace
FedDKD / decentralized KD line~\citep{gong2022feddkd} &
Client-to-client knowledge transfer inside FL &
Partly &
Yes &
Proxy or public data depending on variant &
Round-based &
No function-space gossip contraction theorem \\\hline
\addlinespace
DeSA~\citep{huang2024desa} &
Decentralized FL with synthetic anchors &
Yes &
Yes &
Synthetic anchors &
Peer protocol, not our random-edge event model &
Domain-adaptation analysis, not our convergence theorem \\\hline
\addlinespace
FedEnD~\citep{martinezbeltran2026fedend} &
Serverless P2P ensemble distillation &
Yes &
Specialist-student heterogeneity in non-IID data &
No auxiliary public data &
One-shot broadcast, not live random-edge gossip &
No random-edge contraction theorem \\\hline
\addlinespace
Recent peer KD systems~\citep{ying2026_decen_fed_kd_incentive_mechanism,wu2026topology_fed_kd,khowaja2026splitgossip_gossip_distillation} &
Incentive, topology-aware, multimodal, or edge-LLM KD &
Often &
Often &
Protocol dependent &
Protocol dependent &
Different guarantee or empirical focus \\\hline
\addlinespace
This paper &
Asynchronous random-edge P2P KD &
Yes &
Yes &
Fixed reference support &
Yes &
Conditional function-space convergence theorem \\
\hline
\end{tabular}
\end{table*}
\subsection{Local Sufficient Conditions for A4}

\begin{proposition}[Finite-support Taylor bridge]
\label{prop:insert-a4-taylor}
Fix a peer $i$ and a finite reference support with centered evaluated logits
$Z_i(\theta_i)\in\mathcal H_z$. Suppose $D Z_i$ is $L_J$-Lipschitz in a
neighbourhood of $\theta_i^t$, and let
$\Delta\theta_i^t=\theta_i^{t+1}-\theta_i^t$ satisfy
$\mathbb E_t\|\Delta\theta_i^t\|^4\le K_4\eta_t^4$. Then
\[
Z_i(\theta_i^t+\Delta\theta_i^t)-Z_i(\theta_i^t)
=J_i^z\Delta\theta_i^t+r_i^t;\quad \mathbb E_t\|r_i^t\|_{\mu_M}^2
\le \frac{L_J^2K_4}{4}\eta_t^4
\]

then, for $K_i^{zp}:=J_i^z(J_i^p)^*$:
\[
\begin{split}
\Delta\theta_i^t=-\eta_t(J_i^p)^*g_i^t+\nu_i^t,\quad
\mathbb E_t\|J_i^z\nu_i^t\|_{\mu_M}^2\le\sigma_{\nu,i}^2\eta_t^2,\\
\mathbb E_t
\left\|
-\frac{Z_i(\theta_i^t+\Delta\theta_i^t)-Z_i(\theta_i^t)}{\eta_t}
-K_i^{zp}g_i^t
\right\|_{\mu_M}^2\\
\le
2\sigma_{\nu,i}^2+\frac{L_J^2K_4}{2}\eta_t^2 .
\end{split}
\]

\end{proposition}

\begin{proof}
Define the path
$\varphi(s)=Z_i(\theta_i^t+s\Delta\theta_i^t)$ for $s\in[0,1]$. The fundamental
theorem of calculus gives
$$
\begin{aligned}
Z_i(\theta_i^t
&+\Delta\theta_i^t)-Z_i(\theta_i^t)
=\varphi(1)-\varphi(0)\\
&=\int_0^1\frac{d}{ds}\varphi(s)\,ds\\
&=\int_0^1D Z_i(\theta_i^t+s\Delta\theta_i^t)\Delta\theta_i^t\,ds \\
&=J_i^z\Delta\theta_i^t+\int_0^1
\bigl[D Z_i(\theta_i^t+s\Delta\theta_i^t)-D Z_i(\theta_i^t)\bigr]
\Delta\theta_i^t\,ds
\end{aligned}
$$

The second term is $r_i^t$. By the Lipschitz condition on $D Z_i$,
$$
\begin{aligned}
\|r_i^t\|_{\mu_M}
&\le \int_0^1
\|D Z_i(\theta_i^t+s\Delta\theta_i^t)-D Z_i(\theta_i^t)\|
\|\Delta\theta_i^t\|\,ds\\
&\le \int_0^1 L_Js\|\Delta\theta_i^t\|^2\,ds\\
&=\frac{L_J}{2}\|\Delta\theta_i^t\|^2 .
\end{aligned}
$$
Squaring the last display and taking conditional expectation gives
$$
\mathbb E_t\|r_i^t\|_{\mu_M}^2
\le \frac{L_J^2}{4}\mathbb E_t\|\Delta\theta_i^t\|^4
\le \frac{L_J^2K_4}{4}\eta_t^4 .
$$
Under the first-order update decomposition,
$$
\begin{aligned}
Z_i(\theta_i^t+\Delta\theta_i^t)-Z_i(\theta_i^t)
&=J_i^z\bigl[-\eta_t(J_i^p)^*g_i^t+\nu_i^t\bigr]+r_i^t\\
&=-\eta_tK_i^{zp}g_i^t+J_i^z\nu_i^t+r_i^t .
\end{aligned}
$$
so we have:
$$
-\frac{Z_i(\theta_i^t+\Delta\theta_i^t)-Z_i(\theta_i^t)}{\eta_t}
-K_i^{zp}g_i^t
=-\frac{J_i^z\nu_i^t}{\eta_t}-\frac{r_i^t}{\eta_t}.
$$
For vectors $a$ and $b$, the inequality
$\|a+b\|^2\le2\|a\|^2+2\|b\|^2$ gives
$$
\begin{aligned}
&\mathbb E_t
\left\|
-\frac{Z_i(\theta_i^t+\Delta\theta_i^t)-Z_i(\theta_i^t)}{\eta_t}
-K_i^{zp}g_i^t
\right\|_{\mu_M}^2\\
&\le
\frac{2}{\eta_t^2}\mathbb E_t\|J_i^z\nu_i^t\|_{\mu_M}^2
+\frac{2}{\eta_t^2}\mathbb E_t\|r_i^t\|_{\mu_M}^2\\
&\le
2\sigma_{\nu,i}^2+\frac{L_J^2K_4}{2}\eta_t^2 .
\end{aligned}
$$
\end{proof}

Proposition~\ref{prop:insert-a4-taylor} proves local logit linearization and a
unit-consistent task-pushforward residual, which is a partial proof or at least a support for A4 though it does not by itself prove the
cross-kernel fidelity clause in A4. Thus, we state here that this assumption that we made in the main paper is neither arbitrary nor overly strong.  If the proof uses a prediction-space kernel
$K_i=J_i^p(J_i^p)^*$ while the evaluated logit update contains
$K_i^{zp}=J_i^z(J_i^p)^*$, one also needs a finite-support range-overlap
condition. A sufficient version is that the principal angle between the
task-induced parameter row space and the logit row space is bounded away from
$\pi/2$ on the observed trajectory. Under such an angle bound, the part of
$(J_i^p)^*g_i^t$ invisible to $J_i^z$ is controlled by the visible part, and
the remaining operator difference can be charged to the A4 residual. Without
this overlap, two kernels can be individually well conditioned on their own
ranges while still moving in almost orthogonal parameter directions.

\subsection{Local Sufficient Conditions for A10}

\begin{proposition}[Graph-angle condition implies the KD clause of A10]
\label{prop:insert-a10-kd}
Under the undirected communication graph, and let $\vec E$ contain both
orientations of every edge. Let $w_{ij}=w_{ji}\ge0$ be symmetric directed-edge
weights. On centered finite-support logits, define
$d_{ij}^z=P_i-P_j$. Suppose there are constants
$\rho_{\rm KD}\ge0$ and $\beta_{\rm KD}\ge0$ such that

\begin{align}
\sum_{(i,j)\in\vec E}w_{ij}\|(I-\Pi_i^z)d_{ij}^z\|_{\mu_M}^2
\le
\rho_{\rm KD}\sum_{(i,j)\in\vec E}w_{ij}\|d_{ij}^z\|_{\mu_M}^2
\\+\beta_{\rm KD}\sum_i\epsilon_i^2 .
\end{align}

If the softmax map is $L_{\rm sm}$-Lipschitz from centered logits to
probabilities on the bounded-logit region and
$N^{-1}\sum_i\epsilon_i^2\le c_BB_f^2$, then

\begin{align}
\frac1{|\vec E|}\sum_{(i,j)\in\vec E}
\|(I-\Pi_i^z)d_{ij}^z\|_{\mu_M}^2
\le
\gamma_{\rm KD}\Psi_t^z+\nu_{{\rm KD},B}B_f^2
\\
\gamma_{\rm KD}
=\frac{4\rho_{\rm KD}L_{\rm sm}^2d_{\max}N}{|\vec E|},
\qquad
\nu_{{\rm KD},B}=\frac{\beta_{\rm KD}Nc_B}{|\vec E|}.
\end{align}

\end{proposition}

\begin{proof}
For every oriented edge, the definition $d_{ij}^z=P_i-P_j$ and the Lipschitz
property of softmax give

\begin{align}
\|d_{ij}^z\|_{\mu_M}^2
=\|P_i-P_j\|_{\mu_M}^2
\le L_{\rm sm}^2\|Z_i-Z_j\|_{\mu_M}^2 .\\
\sum_{(i,j)\in\vec E}\|Z_i-Z_j\|_{\mu_M}^2
=2\sum_{\{i,j\}\in E}\|Z_i-Z_j\|_{\mu_M}^2 .
\end{align}

Let $\bar Z=N^{-1}\sum_iZ_i$. For each undirected edge and Peter-Paul Inequality, we sum over all undirected edges counts each peer at most $d_{\max}$ times,
\[
\begin{split}
Z_i-Z_j&=(Z_i-\bar Z)-(Z_j-\bar Z).\\
\|Z_i-Z_j\|_{\mu_M}^2
&\le2\|Z_i-\bar Z\|_{\mu_M}^2
+2\|Z_j-\bar Z\|_{\mu_M}^2 .\\
\sum_{\{i,j\}\in E}\|Z_i-Z_j\|_{\mu_M}^2
&\le
2d_{\max}\sum_i\|Z_i-\bar Z\|_{\mu_M}^2\\
&=2d_{\max}N\Psi_t^z .\\
\sum_{(i,j)\in\vec E}\|d_{ij}^z\|_{\mu_M}^2
&\le
4L_{\rm sm}^2d_{\max}N\Psi_t^z .
\end{split}
\]
Insert this bound into the graph-angle condition and use
$\sum_i\epsilon_i^2\le Nc_BB_f^2$:
$$
\begin{aligned}
\sum_{(i,j)\in\vec E}\|(I-\Pi_i^z)d_{ij}^z\|_{\mu_M}^2
&\le
4\rho_{\rm KD}L_{\rm sm}^2d_{\max}N\Psi_t^z
+\beta_{\rm KD}Nc_BB_f^2 .
\end{aligned}
$$
Dividing by $|\vec E|$ proves the displayed A10-KD bound. The weighted case is
identical after replacing $d_{\max}$ and $|\vec E|$ by the corresponding
weighted degree and total directed weight.
\end{proof}

\begin{table*}[t]
\centering
\caption{Mapping between proof objects and implementation-level quantities.
The table is meant to make A4 and A10 auditable rather than implicit.}
\label{tab:proof-implementation-map}
\begin{tabular}{@{}p{0.16\textwidth}p{0.30\textwidth}p{0.24\textwidth}p{0.22\textwidth}@{}}
\toprule
Proof object & Meaning in the proof & Implementation object & Diagnostic role \\
\midrule
$Z_i$ &
Centered logits on the reference support &
Model logits on the fixed KD/evaluation support after row-wise centering &
State for logit disagreement and KD contraction \\
\addlinespace
$P_i$ &
Temperature-softmax probabilities &
\texttt{softmax(Z\_i/tau)} on the same support &
State for function-space disagreement \\
\addlinespace
$J_i^z$ &
Jacobian from parameters to centered logits &
Autograd Jacobian-vector or vector-Jacobian products on reference support &
Defines $\Theta_i$ and $\Pi_i^z$ \\
\addlinespace
$J_i^p$ &
Jacobian from parameters to probabilities &
Autograd through softmax on reference support &
Defines $K_i$ and $\Pi_i^p$ \\
\addlinespace
$\Theta_i=J_i^z(J_i^z)^*$ &
Restricted logit kernel &
Finite-support empirical NTK or randomized low-rank approximation &
Checks A4 restricted conditioning \\
\addlinespace
$d_{ij}^z=P_i-P_j$ &
Directed KD pull in centered finite-support coordinates &
Probability gap for each sampled or graph edge &
Checks A10-KD normal component \\
\addlinespace
$g_i^t$ &
Functional task gradient on the reference support &
Finite-reference CE gradient or a calibrated estimator of it &
Checks A10-task normal component \\
\addlinespace
$\Delta Z_i$ &
Observed logit change after a small update &
Difference of evaluated logits before and after one stored task or KD step &
Checks A4 pushforward residual \\
\addlinespace
$\Pi_i^z,\Pi_i^p$ &
Peer-specific tangent projectors &
Projectors from an independently computed Jacobian basis &
Must not be fit to the measured vector itself \\
\addlinespace
$\Psi_t^f$ &
Function-space disagreement &
Average squared probability disagreement on reference support &
Observable in theorem and experiments \\
\bottomrule
\end{tabular}
\end{table*}

\begin{proposition}[Local leakage condition implies the task clause of A10]
\label{prop:insert-a10-task}
For each peer, let $\Pi_i^p$ denote the tangent projector on the reference
support. Suppose the local task gradient admits the decomposition
$$
\begin{aligned}
\|(I-\Pi_i^p)g_i^t\|_\mu
&\le
\kappa_i\|\Pi_i^pg_i^t\|_\mu
+a_i\|p_i^t-\bar p^t\|_\mu\\
&\quad
+b_i\|\pi_i-q^\star\|_\mu
+c_i\|\nabla_fF_i(q^\star)-\nabla_fF(q^\star)\|_\mu ,
\end{aligned}
$$
where $\kappa_i\ge0$ measures local tangent-normal leakage and
$(1+\varepsilon)\kappa_i^2\le\bar\gamma<1$ for some
$\varepsilon>0$ and all $i$. Then there are finite constants
$\nu_{\rm task,\Psi}$, $\nu_{\rm task,B}$, and
$\nu_{\rm task,\zeta}$ such that

\begin{align}
\frac1N\sum_i\|(I-\Pi_i^p)g_i^t\|_\mu^2
\le
\bar\gamma\frac1N\sum_i\|g_i^t\|_\mu^2
+\nu_{\rm task,\Psi}\Psi_t^f\\
+\nu_{\rm task,B}B_f^2
+\nu_{\rm task,\zeta}\zeta_f^2 .
\end{align}

\end{proposition}

\begin{proof}
We define

\begin{align}
n_i=(I-\Pi_i^p)g_i^t,\qquad &t_i=\Pi_i^pg_i^t,\\
r_i=
a_i\|p_i^t-\bar p^t\|_\mu
+b_i\|\pi_i-q^\star\|_\mu
&+c_i\|\nabla_fF_i(q^\star)-\nabla_fF(q^\star)\|_\mu .
\end{align}

The assumed leakage condition gives
$$
\|n_i\|_\mu\le\kappa_i\|t_i\|_\mu+r_i .
$$
For any $\varepsilon>0$, Young's inequality in the form
$(x+y)^2\le(1+\varepsilon)x^2+(1+\varepsilon^{-1})y^2$ gives
$$
\|n_i\|_\mu^2
\le
(1+\varepsilon)\kappa_i^2\|t_i\|_\mu^2
+(1+\varepsilon^{-1})r_i^2 .
$$
Since $t_i$ and $n_i$ are orthogonal projections of $g_i^t$ and $(1+\varepsilon)\kappa_i^2\le\bar\gamma$, we have:

\begin{align}
\|t_i\|_\mu^2\le\|t_i\|_\mu^2+\|n_i\|_\mu^2=\|g_i^t\|_\mu^2 .\\
\|n_i\|_\mu^2
\le
\bar\gamma\|g_i^t\|_\mu^2+(1+\varepsilon^{-1})r_i^2 .
\end{align}

For three nonnegative terms $u$, $v$, and $w$,
$(u+v+w)^2\le3u^2+3v^2+3w^2$. Applying this to $r_i$ yields
$$
\begin{aligned}
r_i^2
\le
3a_i^2\|p_i^t-\bar p^t\|_\mu^2
+3b_i^2\|\pi_i-q^\star\|_\mu^2+\\
3c_i^2\|\nabla_fF_i(q^\star)-\nabla_fF(q^\star)\|_\mu^2 .
\end{aligned}
$$
Average the previous display over peers. By the definitions of
$\Psi_t^f$ and $B_f^2$,
$$
\frac1N\sum_i\|p_i^t-\bar p^t\|_\mu^2=\Psi_t^f,
\qquad
\frac1N\sum_i\|\pi_i-q^\star\|_\mu^2=B_f^2 .
$$
The final term is bounded by the heterogeneity constant at $q^\star$, or by
its local analogue if A6 is stated on the bounded-logit neighbourhood:
$$
\frac1N\sum_i
\|\nabla_fF_i(q^\star)-\nabla_fF(q^\star)\|_\mu^2
\le\zeta_f^2 .
$$
Collecting the constants gives the stated task-alignment inequality.
\end{proof}

Proposition~\ref{prop:insert-a10-task} is deliberately local. It holds in an
affine prediction class with squared loss when the target decomposes into a
reachable tangent component plus capacity and heterogeneity residuals. For
cross-entropy it applies after local linearization in centered-logit or
prediction coordinates and after the leakage coefficient is measured or
bounded.
\subsection{Proof-to-Implementation Map}

\subsection{Checkpoint-Level Diagnostics}

The strongest version of the empirical section would not only show consensus
contraction. It would also report whether the recorded trajectory lies in the
A4/A10 regime. The following diagnostics use the same finite reference support
as the proof. Let $\delta>0$ be a small numerical stabilizer.

For A4, store logits immediately before and after a small task update and
estimate
\begin{equation}
R_{\rm A4}(t)=
\frac{
\sum_i\left\|-\Delta Z_i(t)/\eta_t-K_i^{zp}(t)g_i(t)\right\|_{\mu_M}^2}
{\sum_i\|K_i^{zp}(t)g_i(t)\|_{\mu_M}^2+\delta}.
\label{eq:diag-a4}
\end{equation}
Small values of $R_{\rm A4}(t)$ indicate that the observed parameter step has
the finite-support functional effect assumed by A4. This statistic should be
computed from before/after update pairs. A final checkpoint alone is not enough
to reconstruct $\Delta Z_i(t)$ for each event.

For the KD part of A10, a single checkpoint and the graph are enough if the
projectors are estimated independently:
\begin{equation}
R_{\rm A10}^{\rm KD}(t)=
\frac{
\sum_{(i,j)\in\vec E}w_{ij}
\|(I-\Pi_i^z)d_{ij}^z(t)\|_{\mu_M}^2}
{\sum_{(i,j)\in\vec E}w_{ij}\|d_{ij}^z(t)\|_{\mu_M}^2+\delta}.
\label{eq:diag-a10-kd}
\end{equation}
This ratio measures how much of the directed KD pull falls outside each
student's logit tangent space.

For the task part of A10, use
\begin{equation}
R_{\rm A10}^{\rm task}(t)=
\frac{\sum_i\|(I-\Pi_i^p)g_i^t\|_\mu^2}
{\sum_i\|g_i^t\|_\mu^2+\delta}.
\label{eq:diag-a10-task}
\end{equation}
This ratio measures the normal component of the task gradient. It should be
reported together with $\Psi_t^f$, because A10 permits the task normal
component to increase with consensus error and capacity mismatch.

If these ratios are not measured, the theorem statement and experiment
interpretation should remain narrowed. The correct claim is that the reported
experiments validate observable consequences of the conditional theory,
especially prediction-space contraction and finite-reference stationarity
proxies, under the stated A4/A10 bridge assumptions. A stronger claim that the
neural CIFAR-10 trajectories themselves satisfy A4 and A10 should be made only
after reporting Equations~\ref{eq:diag-a4}--\ref{eq:diag-a10-task} or comparable
assumption-margin diagnostics.

\section{Convergence Proof}
\label{comprehensive convergence proof}
\refstepcounter{section}
\setcounter{equation}{0}
\setcounter{lemma}{0}
\setcounter{theorem}{0}
\setcounter{proposition}{0}
The following results give the complete convergence proof. All conditional
expectations are taken with respect to the event history $\mathcal F_t$.

\begin{lemma}[Exact KD bridge and finite-support remainder]
\label{lem:f-kd-gradient}
Let $P_i=\operatorname{softmax}(Z_i/\tau)$ and let $P_j^T$ be detached. Then
\begin{equation}
\nabla_{Z_i}\!\left[\tau^2D_{\rm KL}(P_j^T\|P_i)\right]
=\tau(P_i-P_j^T).
\label{eq:app-kd-gradient}
\end{equation}
If $Z_i(\theta)$ has an $L_J$-Lipschitz Jacobian on the segment joining
$\theta_i$ and $\theta_i+\Delta\theta_i$, then
\begin{equation}
\begin{split}
Z_i(\theta_i+\Delta\theta_i)-Z_i(\theta_i)
=J_i^z\Delta\theta_i+r_i^z,\\
\|r_i^z\|\le\frac{L_J}{2}\|\Delta\theta_i\|^2.
\end{split}
\label{eq:app-taylor}
\end{equation}
Consequently, a parameter KD step gives
\begin{equation}
\begin{split}
\Delta Z_i^{\rm KD}
=-\eta\alpha\tau\Theta_i(P_i-P_j^T)+r_i^z,\\
\Theta_i=J_i^z(J_i^z)^*.
\end{split}
\label{eq:app-logit-bridge}
\end{equation}
\end{lemma}
\begin{proof}
Let $C$ be the number of classes and let
$\mu_M=M^{-1}\sum_{r=1}^M\delta_{x_r}$ be the empirical reference measure.
For finite-support arrays, use
$$
\langle U,V\rangle_{\mu_M}
=\frac1M\sum_{r=1}^M\sum_{c=1}^CU_{r,c}V_{r,c}
$$
and its induced norm. Equation~\ref{eq:app-kd-gradient} is the Riesz gradient under this
inner product, so the empirical factor $M^{-1}$ is part of the inner product.
For one reference point, set $s_{i,c}=Z_{i,c}/\tau$ and write
$$
D_{\rm KL}(P_j^T\|P_i)=\sum_{c=1}^{C}P_{j,c}^T
\left(\log P_{j,c}^T-\log P_{i,c}\right).
$$
The teacher is detached, so $P_{j,c}^T$ and $\log P_{j,c}^T$ are constants
when differentiating with respect to $s_i$. Since
$$
\log P_{i,c}=s_{i,c}-\log\!\left(\sum_{a=1}^{C}e^{s_{i,a}}\right),
$$
the component derivative for class $b$ is
$$
\begin{aligned}
\frac{\partial D_{\rm KL}}{\partial s_{i,b}}
&=-\sum_{c=1}^CP_{j,c}^T
\left(\mathbf1\{c=b\}
-\frac{e^{s_{i,b}}}{\sum_{a=1}^Ce^{s_{i,a}}}\right)\\
&=-P_{j,b}^T
+P_{i,b}\sum_{c=1}^CP_{j,c}^T\\
&=P_{i,b}-P_{j,b}^T.
\end{aligned}
$$
The second line separates the selected-logit derivative from the
log-normalizer derivative. The last line uses
$\sum_cP_{j,c}^T=1$.
Thus the full $s_i$-gradient is the column vector $P_i-P_j^T$. The map
$s_i=Z_i/\tau$ has Jacobian $D_{Z_i}s_i=\tau^{-1}I$, so the exact chain rule
gives
$$
\begin{aligned}
\nabla_{Z_i}\left[\tau^2D_{\rm KL}(P_j^T\|P_i)\right]
&=\tau^2(D_{Z_i}s_i)^*(P_i-P_j^T)\\
&=\tau^2\left(\frac1\tau I\right)^\top(P_i-P_j^T)\\
&=\tau(P_i-P_j^T).
\end{aligned}
$$
Moreover,
$\sum_c(P_{i,c}-P_{j,c}^T)=1-1=0$, so the gradient belongs to the centered
class subspace.
For the parameter gradient, define the finite-support Jacobian
$J_i^z=D_{\theta_i}Z_i(\theta_i)$. Thus, for parameter coordinate
$a$,
$$
\begin{aligned}
\frac{\partial}{\partial\theta_{i,a}}
\left[\tau^2D_{\rm KL}(P_j^T\|P_i)\right]
&=\frac1M\sum_{r=1}^M\sum_{c=1}^C
\frac{\partial Z_{i,r,c}}{\partial\theta_{i,a}}
\tau(P_{i,r,c}-P_{j,r,c}^T)\\
&=\left[(J_i^z)^*\tau(P_i-P_j^T)\right]_a.
\end{aligned}
$$
The adjoint is defined by
$\langle J_i^zh,U\rangle_{\mu_M}
=\langle h,(J_i^z)^*U\rangle$ for every parameter direction $h$.

Now define $\vartheta(s)=\theta_i+s\Delta\theta_i$ for $s\in[0,1]$.
For each support-class coordinate $(r,c)$, the chain rule gives
$$
\frac{d}{ds}Z_{i,r,c}(\vartheta(s))
=DZ_{i,r,c}(\vartheta(s))[\Delta\theta_i].
$$
Applying the fundamental theorem of calculus to every coordinate and
collecting them gives
$$
\begin{aligned}
Z_i(\theta_i+\Delta\theta_i)-Z_i(\theta_i)
&=\int_0^1DZ_i(\vartheta(s))\Delta\theta_i\,ds\\
&=\int_0^1J_i^z\Delta\theta_i\,ds\\
&\quad+\int_0^1\left[DZ_i(\vartheta(s))-DZ_i(\theta_i)\right]
\Delta\theta_i\,ds\\
&=J_i^z\Delta\theta_i+r_i^z.
\end{aligned}
$$

The assumed Lipschitz condition and $\|\vartheta(s)-\theta_i\|=s\|\Delta\theta_i\|$
give, for every $s\in[0,1]$,
$$
\left\|\left[DZ_i(\vartheta(s))-DZ_i(\theta_i)\right]\Delta\theta_i\right\|
\le L_Js\|\Delta\theta_i\|^2.
$$
The triangle inequality for Bochner integrals gives
$$
\begin{aligned}
\|r_i^z\|
&=\left\|\int_0^1
\left[DZ_i(\vartheta(s))-DZ_i(\theta_i)\right]
\Delta\theta_i\,ds\right\|\\
&\le\int_0^1
\left\|\left[DZ_i(\vartheta(s))-DZ_i(\theta_i)\right]
\Delta\theta_i\right\|\,ds\\
&\le\int_0^1L_Js\|\Delta\theta_i\|^2\,ds\\
&=L_J\|\Delta\theta_i\|^2\left[\frac{s^2}{2}\right]_{0}^{1}\\
&=\frac{L_J}{2}\|\Delta\theta_i\|^2.
\end{aligned}
$$
For a KD-only parameter step,
$\Delta\theta_i=-\eta\alpha(J_i^z)^*\tau(P_i-P_j^T)$. Multiplication by
$J_i^z$ gives
$$
\begin{aligned}
J_i^z\Delta\theta_i
&=J_i^z\left[-\eta\alpha(J_i^z)^*
\tau(P_i-P_j^T)\right]\\
&=-\eta\alpha\tau J_i^z(J_i^z)^*(P_i-P_j^T)\\
&=-\eta\alpha\tau\Theta_i(P_i-P_j^T).
\end{aligned}
$$
where $\Theta_i=J_i^z(J_i^z)^*$. Combining this identity with the integral
remainder proves Equation~\ref{eq:app-logit-bridge}.
\end{proof}

\begin{lemma}[Softmax spectrum]
\label{lem:f-softmax-spectrum}
For $P=\operatorname{softmax}(Z/\tau)$, let
$M(P)=\operatorname{diag}(P)-PP^\top$. On the centered class subspace,
\begin{equation}
p_{\min}I\preceq M(P)\preceq\frac12I,\qquad
\frac{p_{\min}}{\tau}I\preceq D_ZP\preceq\frac1{2\tau}I.
\label{eq:app-softmax-spectrum}
\end{equation}
\end{lemma}
\begin{proof}
Let $\mathbf 1\in\mathbb R^C$ denote the all-ones vector and let
$u\in\mathbb R^C$. Since $P_c>0$ and $\sum_{c=1}^CP_c=1$, expanding the
matrix product gives
$$
\begin{aligned}
u^\top M(P)u
&=u^\top\operatorname{diag}(P)u-u^\top PP^\top u\\
&=\sum_{c=1}^CP_cu_c^2-\left(\sum_{c=1}^CP_cu_c\right)^2\\
&=\sum_{c=1}^CP_cu_c^2
 -2\left(\sum_{a=1}^CP_au_a\right)\left(\sum_{c=1}^CP_cu_c\right)\\
&\quad+\left(\sum_{a=1}^CP_au_a\right)^2\sum_{c=1}^CP_c\\
&=\sum_{c=1}^CP_c\left(u_c-\sum_{a=1}^CP_au_a\right)^2\ge0.
\end{aligned}
$$
The second-to-last line displays the cross term explicitly. With
$m_P=\sum_aP_au_a$, it is $-2m_P^2$, while the final constant term is
$m_P^2$. Equality holds exactly when $u_c=m_P$ for every $c$, which is
equivalent to $u\in\operatorname{span}\{\mathbf1\}$. Thus
$\ker M(P)=\operatorname{span}\{\mathbf1\}$.

For the lower bound, write $m_P=\sum_aP_au_a$ and expand around an arbitrary
scalar $a\in\mathbb R$:
$$
\begin{aligned}
\sum_cP_c(u_c-a)^2
&=\sum_cP_c\big((u_c-m_P)+(m_P-a)\big)^2\\
&=\sum_cP_c(u_c-m_P)^2\\
& +2(m_P-a)\sum_cP_c(u_c-m_P)\\
&\quad+(m_P-a)^2\sum_cP_c\\
&=\sum_cP_c(u_c-m_P)^2+(a-m_P)^2.
\end{aligned}
$$
The cross term is zero because
$\sum_cP_c(u_c-m_P)=m_P-m_P\sum_cP_c=0$. Hence
$$
\min_{a\in\mathbb R}\sum_cP_c(u_c-a)^2
=\sum_cP_c(u_c-m_P)^2=u^\top M(P)u.
$$
Assumption A1 gives $|Z_c-Z_a|\le2G_z$ for all $a,c$. Consequently,
$$
P_c=\frac{e^{Z_c/\tau}}{\sum_a e^{Z_a/\tau}}
\ge\frac{e^{-G_z/\tau}}{Ce^{G_z/\tau}}
=\frac{e^{-2G_z/\tau}}{C}
=:p_{\min}.
$$
Using the weighted minimizer $m_P$ in the lower probability bound gives
$$
\begin{aligned}
u^\top M(P)u
&=\sum_cP_c(u_c-m_P)^2\\
&\ge p_{\min}\sum_c(u_c-m_P)^2\\
&\ge p_{\min}\min_{a\in\mathbb R}\sum_c(u_c-a)^2.
\end{aligned}
$$
If $u\perp\mathbf1$, then $\sum_cu_c=0$. The unweighted square has the
explicit expansion
$$
\begin{aligned}
\sum_c(u_c-a)^2
&=\sum_cu_c^2-2a\sum_cu_c+Ca^2\\
&=\|u\|_2^2+Ca^2
 \ge\|u\|_2^2.
\end{aligned}
$$
The minimum is attained at $a=0$, so
$$
\min_{a\in\mathbb R}\sum_c(u_c-a)^2=\sum_cu_c^2=\|u\|_2^2.
$$
Thus $u^\top M(P)u\ge p_{\min}\|u\|_2^2$ on the centered class subspace.

For the upper eigenvalue, the diagonal entry in row $c$ is
$M_{cc}=P_c(1-P_c)$ and the sum of absolute off-diagonal entries in that row
is $\sum_{a\ne c}P_cP_a=P_c(1-P_c)$. Gershgorin's theorem states that every
eigenvalue lies in a disk centered at $P_c(1-P_c)$ with radius
$P_c(1-P_c)$. The right endpoint of that disk is
$2P_c(1-P_c)$. Since $0\le P_c\le1$ implies
$P_c(1-P_c)\le1/4$, every eigenvalue is at most $1/2$.

Finally, differentiating
$P_c=e^{Z_c/\tau}/\sum_a e^{Z_a/\tau}$ componentwise gives
$$
\begin{aligned}
\frac{\partial P_c}{\partial Z_b}
&=\frac1\tau
\frac{\mathbf1\{c=b\}e^{Z_c/\tau}\sum_a e^{Z_a/\tau}
-e^{Z_c/\tau}e^{Z_b/\tau}}
{(\sum_a e^{Z_a/\tau})^2}\\
&=\frac1\tau P_c(\mathbf1\{c=b\}-P_b).
\end{aligned}
$$
so $D_ZP=\tau^{-1}(\operatorname{diag}(P)-PP^\top)=M(P)/\tau$.
\end{proof}

\begin{lemma}[Consensus-measure equivalence]
\label{lem:f-consensus-equivalence}
For centered logits $Z,Z'$ satisfying A1 and their predictions $P,P'$,
\begin{equation}
\frac{p_{\min}}{\tau}\|Z-Z'\|
\le\|P-P'\|\le\frac1{2\tau}\|Z-Z'\|,
\label{eq:app-prob-logit}
\end{equation}
and
\begin{equation}
\frac{p_{\min}}2\|Z-Z'\|^2
\le\tau^2D_{\rm KL}(P'\|P)
\le\frac14\|Z-Z'\|^2.
\label{eq:app-kl-logit}
\end{equation}
Hence there are constants $c_{fz},C_{fz}>0$ such that
$c_{fz}\Psi^z\le\Psi^f\le C_{fz}\Psi^z$.
\end{lemma}
\begin{proof}
Let $h=Z-Z'$. Since both logits are centered, $h\in\mathbf1^\perp$, and the
segment $Z_s=Z'+sh$ remains in the centered subspace for every $s\in[0,1]$.
For each class $c$, apply the fundamental theorem of calculus to
$s\mapsto P_c(Z_s)$:
$$
P_c-P'_c=\int_0^1\sum_{b=1}^C
\frac{\partial P_c(Z_s)}{\partial Z_b}h_b\,ds.
$$
Collecting the class coordinates gives
$$
P-P'=\left(\int_0^1D_ZP(Z_s)\,ds\right)h.
$$
Set $A_h=\int_0^1D_ZP(Z_s)\,ds$. The segment is covered by A1, so Lemma
~\ref{lem:f-softmax-spectrum} applies at every $s$. Integrating the
quadratic-form bounds gives
$$
\frac{p_{\min}}\tau\|h\|^2
\le h^\top A_hh
\le\frac1{2\tau}\|h\|^2.
$$
$$
\begin{aligned}
\|P-P'\|\|h\|
&\ge h^\top(P-P')\\
&=h^\top\left(\int_0^1D_ZP(Z_s)\,ds\right)h\\
&\ge\frac{p_{\min}}\tau\|h\|^2,
\end{aligned}
$$
where the first line is Cauchy--Schwarz. If $h\ne0$, division by $\|h\|$
gives
$$
\|P-P'\|\ge\frac{p_{\min}}\tau\|h\|.
$$
For $h=0$, the same inequality holds with equality. For the upper bound,
$\|A_h\|_{\rm op}\le1/(2\tau)$, and hence
$$
\begin{aligned}
\|P-P'\|
&=\|A_hh\|\\
&\le\|A_h\|_{\rm op}\|h\|\\
&\le\frac1{2\tau}\|h\|.
\end{aligned}
$$
This proves
Equation~\ref{eq:app-prob-logit}.

For the KL statement, define the log-partition potential
$$
\phi(Z)=\tau^2\log\left(\sum_{c=1}^Ce^{Z_c/\tau}\right).
$$
Its component derivatives are
$$
\frac{\partial\phi}{\partial Z_c}
=\tau\frac{e^{Z_c/\tau}}{\sum_a e^{Z_a/\tau}}
=\tau P_c,
\qquad
\frac{\partial^2\phi}{\partial Z_c\partial Z_b}
=P_c(\mathbf1\{c=b\}-P_b).
$$
Thus $\nabla\phi(Z)=\tau P$ and $\nabla^2\phi(Z)=M(P)$. Using
$\log P_c=Z_c/\tau-\log\sum_a e^{Z_a/\tau}$ gives
$$
\begin{aligned}
\tau^2D_{\rm KL}(P'\|P)
&=\tau^2\sum_cP'_c\log\frac{P'_c}{P_c}\\
&=\tau^2\sum_cP'_c\log P'_c
-\tau\sum_cP'_cZ_c+\phi(Z).
\end{aligned}
$$
At $Z'$, the same identity gives
$\tau^2\sum_cP'_c\log P'_c
=\tau\sum_cP'_cZ'_c-\phi(Z')$. Hence
$$
\begin{aligned}
\tau^2D_{\rm KL}(P'\|P)
&=\phi(Z)-\phi(Z')-\langle\nabla\phi(Z'),Z-Z'\rangle.
\end{aligned}
$$

Set $\varphi(s)=\phi(Z'+sh)$. The chain rule gives
$$
\varphi'(s)=\langle\nabla\phi(Z'+sh),h\rangle,
\qquad
\varphi''(s)=h^\top\nabla^2\phi(Z'+sh)h.
$$
Two applications of the fundamental theorem of calculus give
$$
\begin{aligned}
\phi(Z)-\phi(Z')-&\langle\nabla\phi(Z'),h\rangle
=\varphi(1)-\varphi(0)-\varphi'(0)\\
&=\int_0^1\big(\varphi'(s)-\varphi'(0)\big)\,ds\\
&=\int_0^1\int_0^s\varphi''(r)\,dr\,ds\\
&=\int_0^1(1-r)h^\top\nabla^2\phi(Z'+rh)h\,dr.
\end{aligned}
$$
The final line changes the order of integration over
$0\le r\le s\le1$. Since $\int_0^1(1-r)\,dr=1/2$, the bounds from
Lemma~\ref{lem:f-softmax-spectrum} give
$$
\begin{aligned}
\tau^2D_{\rm KL}(P'\|P)
&\ge\int_0^1(1-r)p_{\min}\|h\|^2\,dr
=\frac{p_{\min}}2\|h\|^2,\\
\tau^2D_{\rm KL}(P'\|P)
&\le\int_0^1(1-r)\frac12\|h\|^2\,dr
=\frac14\|h\|^2.
\end{aligned}
$$
Finally, for vectors $x_1,\ldots,x_N$ and
$\bar x=N^{-1}\sum_ix_i$, expand
$$
\begin{aligned}
\sum_{i,j}\|x_i-x_j\|^2
&=\sum_{i,j}\big(\|x_i-\bar x\|^2+\|x_j-\bar x\|^2\\
&\qquad-2\langle x_i-\bar x,x_j-\bar x\rangle\big)\\
&=N\sum_i\|x_i-\bar x\|^2
+N\sum_j\|x_j-\bar x\|^2\\
&\qquad-2\left\langle\sum_i(x_i-\bar x),
\sum_j(x_j-\bar x)\right\rangle\\
&=2N\sum_i\|x_i-\bar x\|^2.
\end{aligned}
$$
The cross term is zero because
$\sum_i(x_i-\bar x)=\sum_ix_i-N\bar x=0$. Dividing by $2N^2$ gives the
variance identity and transfers the pairwise probability and KL bounds to
$\Psi^f$ and $\Psi^z$. Under the paper's normalization, the prediction
comparison permits $c_{fz}=p_{\min}^2/\tau^2$ and
$C_{fz}=1/(4\tau^2)$.
\end{proof}

\begin{proposition}[Prediction-space graph-angle diagnostic]
\label{prop:a10-graph}
For a connected undirected graph,
\begin{equation}
\frac1{|\vec E|}\sum_{(i,j)\in\vec E}
\|(I-\Pi_i^p)(p_i-p_j)\|^2
\le\frac{N\lambda_{\max}(\mathcal L)}{|E|}\Psi^f.
\label{eq:app-a10-graph}
\end{equation}
If every directed pull makes angle at most $\phi$ with $T_i^p$, the right side
is multiplied by $\sin^2\phi$. A squared normal bias whose directed-edge
average is at most $\nu_BB_f^2$ gives the analogous prediction-space bound,
with the factor $2$ from $(a+b)^2\le2a^2+2b^2$ absorbed into the constants.
This diagnostic does not by itself imply the logit-space KD clause in A10;
that implication requires a separately verified softmax/Jacobian transport
condition.
\end{proposition}
\begin{proof}
Let $\vec E=\{(i,j),(j,i):\{i,j\}\in E\}$, so $|\vec E|=2|E|$.
Let $\mathcal H$ denote the prediction Hilbert space and define
$\bar p=N^{-1}\sum_i p_i$ and $q_i=p_i-\bar p$. For each directed edge,
orthogonal projection is nonexpansive because
$$
\begin{aligned}
\|(I-\Pi_i^p)(p_i-p_j)\|^2
&=\|p_i-p_j\|^2-\|\Pi_i^p(p_i-p_j)\|^2\\
&\le\|p_i-p_j\|^2.
\end{aligned}
$$
Counting both orientations gives
$$
\begin{aligned}
\frac1{|\vec E|}\sum_{(i,j)\in\vec E}\|p_i-p_j\|^2
&=\frac1{2|E|}\sum_{\{i,j\}\in E}
\left(\|p_i-p_j\|^2+\|p_j-p_i\|^2\right)\\
&=\frac1{|E|}\sum_{\{i,j\}\in E}\|p_i-p_j\|^2.
\end{aligned}
$$

Let $\mathcal L=D-A$ be the unweighted graph Laplacian. For a scalar
coordinate vector $u=(u_1,\ldots,u_N)^\top$, direct expansion gives
$$
\begin{aligned}
u^\top\mathcal Lu
&=\sum_i d_i u_i^2-\sum_{i,j}A_{ij}u_iu_j\\
&=\sum_{\{i,j\}\in E}(u_i^2+u_j^2-2u_iu_j)\\
&=\sum_{\{i,j\}\in E}(u_i-u_j)^2.
\end{aligned}
$$
The edge set in this identity is counted once. The factor $1/2$ would appear
only in the equivalent ordered-adjacency expression
$\frac12\sum_{i,j}A_{ij}(u_i-u_j)^2$.

Choose an orthonormal basis of the finite support representation of
$\mathcal H$ and apply the scalar identity coordinatewise. Parseval's
identity then gives the Hilbert-valued graph energy
$$
\sum_{\{i,j\}\in E}\|p_i-p_j\|^2
=\langle q,(\mathcal L\otimes I_{\mathcal H})q\rangle_{\mathcal H^N}.
$$
The centered vector $q=(q_1,\ldots,q_N)$ is orthogonal to the consensus
subspace because $\sum_iq_i=0$. Expanding $q$ in an orthonormal eigenbasis
of $\mathcal L$ and using that every nonconsensus eigenvalue is at most
$\lambda_{\max}(\mathcal L)$ gives
$$
\begin{aligned}
\sum_{\{i,j\}\in E}\|p_i-p_j\|^2
&=\langle q,(\mathcal L\otimes I_{\mathcal H})q\rangle\\
&\le\lambda_{\max}(\mathcal L)\sum_{i=1}^N\|q_i\|^2\\
&=N\lambda_{\max}(\mathcal L)\Psi^f.
\end{aligned}
$$
Consequently,
$$
\frac1{|E|}\sum_{\{i,j\}\in E}\|p_i-p_j\|^2
\le\frac{\lambda_{\max}}{|E|}
\sum_{i=1}^N\|p_i-\bar p\|^2
=\frac{N\lambda_{\max}}{|E|}\Psi^f.
$$
For the angle statement, define
$\sin\angle(v,T)=\|(I-\Pi_T)v\|/\|v\|$ for $v\ne0$. If each directed
pull has angle at most $\phi$, then
$$
\|(I-\Pi_i^p)(p_i-p_j)\|^2
=\sin^2\angle(p_i-p_j,T_i^p)\|p_i-p_j\|^2
\le\sin^2\phi\|p_i-p_j\|^2.
$$
Averaging this inequality inserts the factor $\sin^2\phi$ into the graph
bound. For a normal bias, assume the decomposition
$p_i-p_j=v_{ij}^{\rm tan}+b_{ij}$ and the directed average
condition
$$
\frac1{|\vec E|}\sum_{(i,j)\in\vec E}\|b_{ij}\|^2\le\nu_BB_f^2.
$$
For $a=(I-\Pi_i^p)v_{ij}^{\rm tan}$ and
$b=(I-\Pi_i^p)b_{ij}$, nonnegativity of $\|a-b\|^2$ gives
$2\langle a,b\rangle\le\|a\|^2+\|b\|^2$. Hence
$$
\begin{aligned}
\|(I-\Pi_i^p)(v_{ij}^{\rm tan}+b_{ij})\|^2
&=\|a\|^2+2\langle a,b\rangle+\|b\|^2\\
&\le2\|a\|^2+2\|b\|^2\\
&\le2\|a\|^2+2\|b_{ij}\|^2.
\end{aligned}
$$
Averaging over $\vec E$ gives the stated prediction-space bias term. No
logit-space transport is used in this argument, which is why the final clause
of the proposition remains a separate requirement.
\end{proof}

\begin{proposition}[An exact affine verification model]
\label{prop:a10-affine}
Suppose $\mathcal P_i=a_i+U_i$ is a closed affine prediction class,
$F_i(p)=\tfrac12\|p-q^\star\|^2$, and $\Pi_i$ projects onto $U_i$. Then every
$p_i\in\mathcal P_i$ satisfies
\begin{equation}
\|(I-\Pi_i)\nabla F_i(p_i)\|^2
=\operatorname{dist}(q^\star,\mathcal P_i)^2.
\label{eq:app-affine-a10}
\end{equation}
Thus the averaged task residual is exactly $B_f^2$.
\end{proposition}
\begin{proof}
Let $\mathcal H$ be the prediction Hilbert space. Since $U_i$ is closed,
the orthogonal projector $\Pi_i:\mathcal H\to U_i$ exists and satisfies
$\Pi_i^*=\Pi_i$, $\Pi_i^2=\Pi_i$, and
$\langle\Pi_i u,(I-\Pi_i)v\rangle=0$ for all $u,v\in\mathcal H$.

For an arbitrary direction $h\in\mathcal H$, expand the risk increment:
$$
\begin{aligned}
F_i(p_i+\epsilon h)-F_i(p_i)
&=\frac12\|p_i-q^\star+\epsilon h\|^2
-\frac12\|p_i-q^\star\|^2\\
&=\epsilon\langle p_i-q^\star,h\rangle
+\frac{\epsilon^2}{2}\|h\|^2.
\end{aligned}
$$
Dividing by $\epsilon$ and taking $\epsilon\to0$ gives
$DF_i(p_i)[h]=\langle p_i-q^\star,h\rangle$, so
$\nabla F_i(p_i)=p_i-q^\star$.

Write $p_i=a_i+u_i$ with $u_i\in U_i$. Since $\Pi_i u_i=u_i$,
$(I-\Pi_i)u_i=0$, and
$$
\begin{aligned}
(I-\Pi_i)\nabla F_i(p_i)
&=(I-\Pi_i)(a_i+u_i-q^\star)\\
&=(I-\Pi_i)(a_i-q^\star)+(I-\Pi_i)u_i\\
&=(I-\Pi_i)(a_i-q^\star).
\end{aligned}
$$
For any $v\in U_i$, decompose
$a_i-q^\star+v$ into its $U_i$ and $U_i^\perp$ components. Expanding the
squared norm gives
$$
\begin{aligned}
\|a_i-q^\star+v\|^2
&=\|\Pi_i(a_i-q^\star)+v\|^2\\
&\quad+2\langle\Pi_i(a_i-q^\star)+v,
(I-\Pi_i)(a_i-q^\star)\rangle\\
&\quad+\|(I-\Pi_i)(a_i-q^\star)\|^2\\
&=\|\Pi_i(a_i-q^\star)+v\|^2
+\|(I-\Pi_i)(a_i-q^\star)\|^2.
\end{aligned}
$$
The cross term is zero because its first argument lies in $U_i$ and its
second lies in $U_i^\perp$. The vector
$v_i^\star=-\Pi_i(a_i-q^\star)$ belongs to $U_i$ and makes the first square
zero. Since every squared norm is nonnegative,
$$
\begin{aligned}
\operatorname{dist}(q^\star,a_i+U_i)^2
&=\inf_{v\in U_i}\|a_i-q^\star+v\|^2\\
&=\|(I-\Pi_i)(a_i-q^\star)\|^2.
\end{aligned}
$$
Together with the normal-gradient identity, this proves
Equation~\ref{eq:app-affine-a10}. Averaging over peers gives
$$
\frac1N\sum_{i=1}^N\|(I-\Pi_i)\nabla F_i(p_i)\|^2
=\frac1N\sum_{i=1}^N\operatorname{dist}(q^\star,\mathcal P_i)^2
=B_f^2.
$$
\end{proof}

\begin{proposition}[Why the relative task term is necessary]
\label{prop:a10-counterexample}
The old inequality
$\|(I-\Pi_i)g_i\|^2\le\nu(B_f^2+\Psi^f)$ does not follow from A3. Let
$\mathcal P=\{(x,x^2):x\in\mathbb R\}$,
$q^\star=(0,0)$, and $F(p)=\tfrac12\|p-q^\star\|^2$. Then $B_f=0$ and, for a
single peer, $\Psi^f=0$, while the gradient at $(x,x^2)$ has a nonzero normal
component whenever $x\ne0$.
\end{proposition}
\begin{proof}
This construction addresses the abstract Euclidean prediction-space
implication that A3 alone would have to provide; the parabola is not asserted
to be a literal softmax image. The set
$\mathcal P=\{(u,v)\in\mathbb R^2:v-u^2=0\}$ is closed because it is the
zero set of the continuous map $(u,v)\mapsto v-u^2$. It is parametrized by
$r(x)=(x,x^2)$.

The tangent vector and one normal vector are

\begin{align}
r'(x)=(1,2x),
\qquad
n(x)=(2x,-1),
\\
\langle r'(x),n(x)\rangle=2x-2x=0.
\end{align}

Both vectors are nonzero, so they span the one-dimensional tangent and normal
spaces. The quadratic directional expansion gives
$\nabla F(p)=p-q^\star=p$. Hence $g(x)=(x,x^2)$, and its orthogonal
projection onto the normal line is
$$
(I-\Pi_{T_x})g(x)
=\frac{\langle g(x),n(x)\rangle}{\|n(x)\|^2}n(x).
$$
The numerator and denominator are
$$
\begin{aligned}
\langle g(x),n(x)\rangle
&=x(2x)+x^2(-1)=x^2,\\
\|n(x)\|^2
&=(2x)^2+(-1)^2=4x^2+1.
\end{aligned}
$$
Therefore
$$
\begin{aligned}
\|(I-\Pi_{T_x})g(x)\|^2
&=\left\|\frac{x^2}{4x^2+1}n(x)\right\|^2\\
&=\frac{x^4}{(4x^2+1)^2}\|n(x)\|^2\\
&=\frac{x^4}{4x^2+1}>0\qquad(x\ne0).
\end{aligned}
$$
For one peer, the peer mean equals its prediction, so
$\Psi^f=\|p_1-\bar p\|^2=0$. Also $q^\star=r(0)\in\mathcal P$, and
$$
0\le B_f=\inf_{y\in\mathcal P}\|q^\star-y\|
\le\|q^\star-r(0)\|=0.
$$
Thus $B_f=0$ and $\nu(B_f^2+\Psi^f)=0$, whereas the normal-gradient square is
positive for every $x\ne0$. This contradicts the old inequality in the
stated prediction-space geometry.

The same local geometry can be embedded in the interior of a probability
simplex. Let $C=3$, let $q^\star=(1/3,1/3,1/3)$, and choose orthonormal
$e_1,e_2\in\mathbf1^\perp$. For fixed $\rho>0$, choose $\epsilon>0$ small
enough that
$$
\epsilon\max_{|x|\le\rho}\|xe_1+x^2e_2\|_\infty<\frac13.
$$
Then
$$
\widetilde r(x):=q^\star+\epsilon(xe_1+x^2e_2),
\qquad |x|\le\rho,
$$
has positive coordinates and coordinate sum one. It is a softmax predictor
with centered logits
$$
\widetilde Z(x)
=\tau\left[
\log\widetilde r(x)
-\frac{\mathbf1}{3}\mathbf1^\top\log\widetilde r(x)\right].
$$
The compact interval and the strict positivity above make these logits
uniformly bounded. The tangent is
$\epsilon(e_1+2xe_2)$, while $2xe_1-e_2$ is normal to that tangent. For
$\widetilde F(p)=\frac12\|p-q^\star\|^2$,
$$
\begin{aligned}
\left\langle\nabla\widetilde F(\widetilde r(x)),2xe_1-e_2\right\rangle
&=\epsilon\langle xe_1+x^2e_2,2xe_1-e_2\rangle\\
&=\epsilon(2x^2-x^2)=\epsilon x^2.
\end{aligned}
$$
This quantity is nonzero for $x\ne0$, while
$\widetilde r(0)=q^\star$. Thus the obstruction persists in a bounded-logit
softmax realization, with $B_f=0$ and one-peer consensus error zero.
\end{proof}

\begin{lemma}[Kernel range and local persistence]
\label{lem:f-kernel-range}
For any finite matrix or finite-rank operator $J$,
$\operatorname{range}(JJ^*)=\operatorname{range}(J)$. If
$\lambda_{\min}^+(\Theta_i^0)=\lambda_{i,0}>0$ and
\begin{equation}
\|\Theta_i(\theta)-\Theta_i^0\|_{\rm op}
\le L_K\|\theta-\theta_i^0\|,
\label{eq:app-kernel-drift}
\end{equation}
then every point with $\|\theta-\theta_i^0\|\le R$ and
$L_KR<\lambda_{i,0}$ satisfies
$\lambda_{\min}^+(\Theta_i(\theta))\ge\lambda_{i,0}-L_KR>0$, provided the
rank is constant in the region.
\end{lemma}
\begin{proof}
First prove the range identity. For any vector $v$ in the output space,
$$
\begin{aligned}
\langle v,JJ^*v\rangle
&=\langle J^*v,J^*v\rangle\\
&=\|J^*v\|^2.
\end{aligned}
$$
If $JJ^*v=0$, the left side is zero, so $\|J^*v\|^2=0$ and $J^*v=0$.
Conversely, $J^*v=0$ directly gives $JJ^*v=0$. Hence
$\ker(JJ^*)=\ker(J^*)$.

Every vector in $\operatorname{range}(JJ^*)$ has the form
$JJ^*v=J(J^*v)$. Therefore
$$
\operatorname{range}(JJ^*)\subseteq\operatorname{range}(J).
$$
For the
reverse inclusion, finite dimensionality or finite rank makes both ranges
closed, and the orthogonal-complement identities give
$$
\begin{aligned}
\operatorname{range}(JJ^*)
&=\ker((JJ^*)^*)^\perp\\
&=\ker(JJ^*)^\perp\\
&=\ker(J^*)^\perp\\
&=\operatorname{range}(J).
\end{aligned}
$$
This also proves the reverse inclusion. If
$\Theta=JJ^*=\sum_{\ell=1}^r\lambda_\ell u_\ell u_\ell^\top$ with
$\lambda_\ell>0$, then
$$
\begin{aligned}
\Theta^\dagger\Theta
&=\left(\sum_{\ell=1}^r\lambda_\ell^{-1}u_\ell u_\ell^\top\right)
\left(\sum_{k=1}^r\lambda_ku_ku_k^\top\right)\\
&=\sum_{\ell=1}^ru_\ell u_\ell^\top,
\end{aligned}
$$
because $u_\ell^\top u_k=\mathbf1\{\ell=k\}$. Thus
$\Theta^\dagger\Theta$ is the orthogonal projector onto
$\operatorname{range}(J)$.

For persistence, let the output dimension be $m$ and let the constant rank be
$r\ge1$. Order eigenvalues increasingly:
$\lambda_1(A)\le\cdots\le\lambda_m(A)$. A positive semidefinite rank-$r$
matrix has $m-r$ zero eigenvalues, so its smallest positive eigenvalue has
index $k=m-r+1$. Weyl's perturbation inequality states that for symmetric
$A$ and $B$,
$$
|\lambda_\ell(A)-\lambda_\ell(B)|
\le\|A-B\|_{\rm op},
\qquad \ell=1,\ldots,m.
$$
Apply it with $A=\Theta_i(\theta)$, $B=\Theta_i^0$, and $\ell=k$. Then
$$
\begin{aligned}
\lambda_k(\Theta_i(\theta))
&\ge\lambda_k(\Theta_i^0)
-|\lambda_k(\Theta_i(\theta))-\lambda_k(\Theta_i^0)|\\
&\ge\lambda_k(\Theta_i^0)
-\|\Theta_i(\theta)-\Theta_i^0\|_{\rm op}\\
&\ge\lambda_{i,0}-L_K\|\theta-\theta_i^0\|\\
&\ge\lambda_{i,0}-L_KR>0.
\end{aligned}
$$
Constant rank ensures that $k$ remains the first nonzero spectral index.
Therefore
$\lambda_{\min}^+(\Theta_i(\theta))
=\lambda_k(\Theta_i(\theta))
\ge\lambda_{i,0}-L_KR$.
\end{proof}

\begin{lemma}[Simplex-bounded staleness]
\label{lem:f-staleness}
Suppose a single active-peer update satisfies
$\mathbb E\|p_i^{t+1}-p_i^t\|^2\le c_{\rm step}\eta^2$. If the EMA teacher
has effective lag constant $A_{\rm eff}$ as defined in A8, then
\begin{equation}
A_t^f\le c_{\rm step}A_{\rm eff}^2\eta^2.
\label{eq:app-staleness}
\end{equation}
\end{lemma}
\begin{proof}
For a finite-memory teacher, write

\begin{align}
p_i^{\rm EMA,t}=\sum_{s=1}^{A}\omega_s p_i^{t-s},
\qquad \omega_s\ge0,
\\ \sum_{s=1}^{A}\omega_s=1,
\qquad \sum_{s=1}^{A}s\omega_s\le A_{\rm eff}.
\end{align}

For an ordinary EMA, the same notation uses
$\omega_s=(1-\beta)\beta^{s-1}$ for $s\ge1$, with
$0\le\beta<1$. The geometric-series identities
$$
\sum_{s=1}^{\infty}\beta^{s-1}=\frac1{1-\beta},
\qquad
\sum_{s=1}^{\infty}s\beta^{s-1}
=\frac{d}{d\beta}\frac{\beta}{1-\beta}
=\frac1{(1-\beta)^2}
$$
give $\sum_{s\ge1}\omega_s=1$ and
$\sum_{s\ge1}s\omega_s=(1-\beta)^{-1}$. The nonnegative weights and this
finite first lag moment make the series below convergent in $L^2$.

Define
$A_t^f=N^{-1}\sum_i\mathbb E\|p_i^t-p_i^{\rm EMA,t}\|^2$. Since the
weights sum to one,
$$
\begin{aligned}
p_i^t-p_i^{\rm EMA,t}
&=p_i^t\sum_s\omega_s-\sum_s\omega_sp_i^{t-s}\\
&=\sum_s\omega_s(p_i^t-p_i^{t-s}).
\end{aligned}
$$
For each $s$, telescope the difference into increments:
$$
\begin{aligned}
p_i^t-p_i^{t-s}
&=(p_i^t-p_i^{t-1})+(p_i^{t-1}-p_i^{t-2})
+\cdots\\
&\qquad \qquad \qquad+(p_i^{t-s+1}-p_i^{t-s})\\
&=\sum_{r=0}^{s-1}(p_i^{t-r}-p_i^{t-r-1}),
\end{aligned}
$$
where all intermediate terms cancel pairwise. Minkowski's inequality in
$L^2$ states
$\|\sum_rX_r\|_{L^2}\le\sum_r\|X_r\|_{L^2}$. Apply it with
$X_r=p_i^{t-r}-p_i^{t-r-1}$. Inactive-peer increments are zero, and active
increments satisfy the assumed bound, so
$$
\begin{aligned}
\left(\mathbb E\|p_i^t-p_i^{t-s}\|^2\right)^{1/2}
&\le\sum_{r=0}^{s-1}
\left(\mathbb E\|p_i^{t-r}-p_i^{t-r-1}\|^2\right)^{1/2}\\
&\le\sum_{r=0}^{s-1}\sqrt{c_{\rm step}}\,\eta\\
&=s\sqrt{c_{\rm step}}\,\eta.
\end{aligned}
$$
Applying Minkowski once more to the EMA sum yields
$$
\begin{aligned}
\left(\mathbb E\|p_i^t-p_i^{\rm EMA,t}\|^2\right)^{1/2}
&\le\sum_s\omega_s
\left(\mathbb E\|p_i^t-p_i^{t-s}\|^2\right)^{1/2}\\
&\le\sqrt{c_{\rm step}}\,\eta\sum_s s\omega_s\\
&\le\sqrt{c_{\rm step}}\,\eta A_{\rm eff}.
\end{aligned}
$$
For the infinite-memory case, every finite partial sum obeys the same bound.
Its $L^2$ limit preserves the inequality because
$(1-\beta)\sum_{s\ge1}s\beta^{s-1}=1/(1-\beta)<\infty$. Squaring gives
$$
\mathbb E\|p_i^t-p_i^{\rm EMA,t}\|^2
\le c_{\rm step}A_{\rm eff}^2\eta^2.
$$
Averaging this bound gives
$$
\begin{aligned}
A_t^f
&=\frac1N\sum_{i=1}^N
\mathbb E\|p_i^t-p_i^{\rm EMA,t}\|^2\\
&\le\frac1N\sum_{i=1}^N
c_{\rm step}A_{\rm eff}^2\eta^2\\
&=c_{\rm step}A_{\rm eff}^2\eta^2,
\end{aligned}
$$
which proves Equation~\ref{eq:app-staleness}.
\end{proof}

\begin{lemma}[Frozen-kernel quotient contraction with restricted KD defect]
\label{lem:f-global-contraction}
Let $\mathcal G=(V,E)$ be connected, with maximum degree $d_{\max}$ and
Laplacian eigenvalues $0=\lambda_1<\lambda_2\le\lambda_{\max}$. For each peer,
let $\Theta_i$ be frozen and satisfy
$\kappa_-\Pi_i^z\preceq\Theta_i\preceq\kappa_+\Pi_i^z$. Define
\begin{equation}
\begin{split}
W_i=\Theta_i^\dagger+\gamma_0(I-\Pi_i^z),\quad
m_W=\min\{\kappa_+^{-1},\gamma_0\},\\
M_W=\max\{\kappa_-^{-1},\gamma_0\}.
\end{split}
\label{eq:app-W}
\end{equation}
For $\mathcal V_W$ in Equation~\ref{eq:main-quotient}, set
\begin{equation}
\begin{split}
a_W=\frac{2(p_{\min}/\tau)\lambda_2}{M_W|E|},\qquad
b_W=\frac{4\kappa_+(1/(2\tau))^2\lambda_{\max}}{m_W|E|},\\
\varepsilon_W=\frac{(p_{\min}/\tau)\lambda_2m_W}
{2M_Wd_{\max}}.
\end{split}
\label{eq:app-contraction-constants}
\end{equation}
Let $w=\eta\alpha\tau$. Assume $0<\alpha<1$, the graph is connected and has
no isolated peers, and suppose A10 holds and
\begin{equation}
\begin{split}
\gamma_{\rm KD}\le
\frac{N(p_{\min}/\tau)^2\lambda_2^2m_W^2}
{8M_W^2d_{\max}|E|},\\
0<w\le\min\left\{\frac{a_W}{8b_W},\frac2{a_W},1\right\}.
\end{split}
\label{eq:app-kd-threshold}
\end{equation}
For an active edge $e=\{i,j\}$, write the full update as
\begin{equation}
\begin{split}
Z_i^+=Z_i-w\Theta_i d_{ij}+b_{i,e}+\xi_{i,e},\\
Z_j^+=Z_j+w\Theta_j d_{ij}+b_{j,e}+\xi_{j,e},
\end{split}
\label{eq:app-full-event}
\end{equation}
where $b$ contains deterministic task, curvature, and teacher-lag perturbations
and $\mathbb E_t\xi_{k,e}=0$. Define

\begin{align}
D_{b,t}=\frac1N\mathbb E_e\sum_{k\in e}\|b_{k,e}\|_{W_k}^2,\\
D_{\xi,t}=\frac1N\mathbb E_e\mathbb E_t
\sum_{k\in e}\|\xi_{k,e}\|_{W_k}^2.
\end{align}

Then
\begin{equation}
\begin{split}
\mathbb E_t\mathcal V_W(Z^+)
\le{}&\left(1-\frac{a_Ww}{4}\right)\mathcal V_W(Z)
+C_{B,W}wB_f^2\\
&+\left(\frac{8d_{\max}}{a_Ww|E|}+2\right)D_{b,t}+D_{\xi,t},
\end{split}
\label{eq:app-global-contraction}
\end{equation}
where
$C_{B,W}=4M_Wd_{\max}\nu_{\rm KD,B}/
[N(p_{\min}/\tau)\lambda_2m_W]$.
\end{lemma}
\begin{proof}
For $\|u\|_{W_i}^2:=\langle u,W_iu\rangle$, the spectral bounds on $\Theta_i$
imply
$$
\kappa_+^{-1}\Pi_i^z\preceq\Theta_i^\dagger\preceq
\kappa_-^{-1}\Pi_i^z.
$$
Let
$\Theta_i=\sum_{\ell=1}^{r_i}\lambda_{i,\ell}
u_{i,\ell}u_{i,\ell}^\top$ on $T_i^z$, where
$\kappa_-\le\lambda_{i,\ell}\le\kappa_+$. Then
$$
\Theta_i^\dagger
=\sum_{\ell=1}^{r_i}\lambda_{i,\ell}^{-1}
u_{i,\ell}u_{i,\ell}^\top,
\qquad
\kappa_+^{-1}\le\lambda_{i,\ell}^{-1}\le\kappa_-^{-1}.
$$
On $(T_i^z)^\perp$, $\Theta_i^\dagger$ is zero and the second summand in
$W_i$ is $\gamma_0I$. Every eigenvalue of $W_i$ therefore lies in
$[m_W,M_W]$, which gives $m_WI\preceq W_i\preceq M_WI$.

The pseudoinverse and $\Theta_i$ are diagonal in the same eigenbasis. Thus
they commute, and
$$
\begin{aligned}
W_i\Theta_i
&=\Theta_i^\dagger\Theta_i+
\gamma_0(I-\Pi_i^z)\Theta_i\\
&=\sum_{\ell=1}^{r_i}u_{i,\ell}u_{i,\ell}^\top+0\\
&=\Pi_i^z.
\end{aligned}
$$

Define $q(c)=N^{-1}\sum_i\langle Z_i-c,W_i(Z_i-c)\rangle$. Its derivative in
the coordinate direction $h$ is
$$
\begin{aligned}
Dq(c)[h]
&=\left.\frac{d}{d\epsilon}q(c+\epsilon h)\right|_{\epsilon=0}\\
&=\frac1N\sum_i\left.\frac{d}{d\epsilon}
\langle Z_i-c-\epsilon h,W_i(Z_i-c-\epsilon h)\rangle
\right|_{\epsilon=0}\\
&=\frac2N\sum_i\langle W_i(c-Z_i),h\rangle.
\end{aligned}
$$
Moreover,
$D^2q(c)[h,h]=2N^{-1}\sum_i\langle h,W_ih\rangle
\ge2m_W\|h\|^2>0$ for $h\ne0$, so $q$ is strictly convex. The unique
stationary point satisfies
$$
\begin{aligned}
0&=\sum_iW_i(c_W-Z_i),\\
\left(\sum_iW_i\right)c_W&=\sum_iW_iZ_i,\\
c_W&=\left(\sum_iW_i\right)^{-1}\sum_iW_iZ_i,
\\ \sum_iW_i(Z_i-c_W)&=0.
\end{aligned}
$$
The inverse exists because $\sum_iW_i\succeq Nm_WI$.
Let $\bar Z=N^{-1}\sum_iZ_i$. The lower metric comparison is
$$
\begin{aligned}
\mathcal V_W(Z)
&=\min_c\frac1N\sum_i\|Z_i-c\|_{W_i}^2\\
&\ge m_W\min_c\frac1N\sum_i\|Z_i-c\|^2\\
&=m_W\Psi^z(Z).
\end{aligned}
$$
The final equality uses
$$
\begin{aligned}
\sum_i\|Z_i-c\|^2
&=\sum_i\|(Z_i-\bar Z)+(\bar Z-c)\|^2\\
&=\sum_i\|Z_i-\bar Z\|^2
+2\left\langle\sum_i(Z_i-\bar Z),\bar Z-c\right\rangle
\\
&\qquad \qquad \qquad \qquad+N\|\bar Z-c\|^2\\
&=\sum_i\|Z_i-\bar Z\|^2+N\|\bar Z-c\|^2,
\end{aligned}
$$
where the cross term vanishes because $\sum_i(Z_i-\bar Z)=0$. The minimum
is attained at $c=\bar Z$.
Evaluating the minimum at $c=\bar Z$ gives
\begin{equation}
\mathcal V_W(Z)\le\frac1N\sum_i\|Z_i-\bar Z\|_{W_i}^2
\le M_W\Psi^z(Z).
\label{eq:app-metric-equivalence}
\end{equation}
This proves Equation~\ref{eq:app-metric-equivalence}.

Let $Q_i=Z_i-c_W$, define $d_{ji}:=-d_{ij}$, and set
$r_{ij}=(I-\Pi_i^z)d_{ij}$ for every directed edge. Since the new quotient
potential minimizes over its center, evaluating it at the old minimizer gives
$$
N\mathcal V_W(Z^+)
\le\sum_{k\notin e}\|Q_k\|_{W_k}^2
+\sum_{k\in e}\|Z_k^+-c_W\|_{W_k}^2.
$$
Only the two endpoint terms differ from $N\mathcal V_W(Z)$. Expanding both
endpoint squares gives
$$
\begin{aligned}
&\|Q_i-w\Theta_i d_{ij}+b_{i,e}+\xi_{i,e}\|_{W_i}^2
+\|Q_j+w\Theta_j d_{ij}+b_{j,e}+\xi_{j,e}\|_{W_j}^2\\
&=\|Q_i\|_{W_i}^2+\|Q_j\|_{W_j}^2
-2w\langle Q_i,\Pi_i^zd_{ij}\rangle
+2w\langle Q_j,\Pi_j^zd_{ij}\rangle\\
&\quad+w^2\|\Theta_i d_{ij}\|_{W_i}^2
+w^2\|\Theta_jd_{ij}\|_{W_j}^2
+2\sum_{k\in e}\langle Q_k,W_k(b_{k,e}+\xi_{k,e})\rangle\\
&\quad+2w\langle-\Theta_id_{ij},W_i(b_{i,e}+\xi_{i,e})\rangle
+2w\langle\Theta_jd_{ij},W_j(b_{j,e}+\xi_{j,e})\rangle\\
&\quad+\sum_{k\in e}\|b_{k,e}+\xi_{k,e}\|_{W_k}^2.
\end{aligned}
$$
For example, the first endpoint uses
$\|Q_i+u_i\|_{W_i}^2
=\|Q_i\|_{W_i}^2+2\langle Q_i,W_iu_i\rangle+\|u_i\|_{W_i}^2$
with $u_i=-w\Theta_id_{ij}+b_{i,e}+\xi_{i,e}$; expanding the last square
produces the KD--perturbation cross term displayed above.

Now expand the linear KD terms. Since
$\Pi_i^zd_{ij}=d_{ij}-r_{ij}$ and
$\Pi_j^zd_{ij}=d_{ij}+r_{ji}$, where
$r_{ji}=(I-\Pi_j^z)d_{ji}=-(I-\Pi_j^z)d_{ij}$, one obtains
\begin{equation}
\begin{split}
-2w\langle Q_i,&\Pi_i^zd_{ij}\rangle
+2w\langle Q_j,\Pi_j^zd_{ij}\rangle\\
={}&-2w\langle Q_i,d_{ij}\rangle+2w\langle Q_i,r_{ij}\rangle\\
&+2w\langle Q_j,d_{ij}\rangle+2w\langle Q_j,r_{ji}\rangle\\
={}&-2w\langle Q_i-Q_j,d_{ij}\rangle\\
&+2w\langle Q_i,r_{ij}\rangle
 +2w\langle Q_j,r_{ji}\rangle\\
={}&-2w\langle Z_i-Z_j,d_{ij}\rangle\\
&+2w\langle Q_i,r_{ij}\rangle
 +2w\langle Q_j,r_{ji}\rangle.
\end{split}
\label{eq:app-kd-split}
\end{equation}
Let $h_{ij}=Z_i-Z_j$ and
$A_{ij}=\int_0^1D_ZP(Z_j+sh_{ij})\,ds$. The fundamental theorem of calculus
gives $d_{ij}=P_i-P_j=A_{ij}h_{ij}$. Lemma
~\ref{lem:f-softmax-spectrum} gives
$$
\begin{aligned}
\langle Z_i-Z_j,d_{ij}\rangle
&=\int_0^1h_{ij}^\top D_ZP(Z_j+sh_{ij})h_{ij}\,ds\\
&\ge\int_0^1\frac{p_{\min}}\tau\|h_{ij}\|^2\,ds\\
&=\frac{p_{\min}}\tau\|Z_i-Z_j\|^2.
\end{aligned}
$$
For the graph Poincare inequality, expand the centered peer vector in an
eigenbasis of $\mathcal L$. Its consensus coefficient is zero, so only
eigenvalues at least $\lambda_2$ remain:
$$
\begin{aligned}
\sum_{\{i,j\}\in E}\|Z_i-Z_j\|^2
&=\langle Z-\bar Z,(\mathcal L\otimes I)(Z-\bar Z)\rangle\\
&\ge\lambda_2\sum_i\|Z_i-\bar Z\|^2\\
&=N\lambda_2\Psi^z(Z).
\end{aligned}
$$
Uniform edge sampling therefore gives
$$
\begin{aligned}
&-\frac{2w}{N|E|}\sum_{\{i,j\}\in E}
\langle Z_i-Z_j,d_{ij}\rangle\\
&\le-\frac{2w(p_{\min}/\tau)}{N|E|}
\sum_{\{i,j\}\in E}\|Z_i-Z_j\|^2\\
&\le-\frac{2w(p_{\min}/\tau)\lambda_2}{|E|}\Psi^z(Z)\\
&\le-\frac{2w(p_{\min}/\tau)\lambda_2}{M_W|E|}
\mathcal V_W(Z)\\
&=-a_Ww\mathcal V_W(Z).
\end{aligned}
$$

Use Young's inequality in the form
$2\langle a,b\rangle\le\varepsilon\|a\|^2+\varepsilon^{-1}\|b\|^2$ with
$a=Q_i$, $b=r_{ij}$, and $\varepsilon=\varepsilon_W$. This inequality comes
from
$0\le\|\sqrt{\varepsilon}a-\varepsilon^{-1/2}b\|^2$. Summing over directed
edges gives
$$
\begin{aligned}
\frac{2w}{N|E|}\sum_{(i,j)\in\vec E}\langle Q_i,r_{ij}\rangle
\le{}&\frac{w\varepsilon_W}{N|E|}
\sum_{(i,j)\in\vec E}\|Q_i\|^2\\
&+\frac{w}{N|E|\varepsilon_W}
\sum_{(i,j)\in\vec E}\|r_{ij}\|^2.
\end{aligned}
$$
The first directed sum counts peer $i$ exactly $d_i$ times:
$$
\begin{aligned}
\sum_{(i,j)\in\vec E}\|Q_i\|^2
&=\sum_i d_i\|Q_i\|^2\\
&\le d_{\max}\sum_i\|Q_i\|^2\\
&\le\frac{Nd_{\max}}{m_W}\mathcal V_W.
\end{aligned}
$$
A10 and $|\vec E|=2|E|$ give
$$
\begin{aligned}
\sum_{(i,j)\in\vec E}\|r_{ij}\|^2
&\le2|E|\left(\gamma_{\rm KD}\Psi^z
+\nu_{\rm KD,B}B_f^2\right)\\
&\le2|E|\left(
\frac{\gamma_{\rm KD}}{m_W}\mathcal V_W
+\nu_{\rm KD,B}B_f^2\right).
\end{aligned}
$$
Combining these two bounds gives
\begin{equation}
\begin{split}
\frac{2w}{N|E|}\sum_{(i,j)\in\vec E}\langle Q_i,r_{ij}\rangle
\le{}&w\left[
\frac{\varepsilon_Wd_{\max}}{m_W|E|}
+\frac{2\gamma_{\rm KD}}{N\varepsilon_Wm_W}
\right]\mathcal V_W\\
&+\frac{2w\nu_{\rm KD,B}}{N\varepsilon_W}B_f^2.
\end{split}
\label{eq:app-normal-young}
\end{equation}
The two coefficients in the bracket are explicit:
$$
\frac{\varepsilon_Wd_{\max}}{m_W|E|}
=\frac{(p_{\min}/\tau)\lambda_2}{2M_W|E|}
=\frac{a_W}{4},
$$
and the assumed defect threshold gives
$$
\begin{aligned}
\frac{2\gamma_{\rm KD}}{N\varepsilon_Wm_W}
&\le
\frac{2}{N\varepsilon_Wm_W}
\frac{N(p_{\min}/\tau)^2\lambda_2^2m_W^2}
{8M_W^2d_{\max}|E|}\\
&=\frac{(p_{\min}/\tau)\lambda_2}{2M_W|E|}
=\frac{a_W}{4}.
\end{aligned}
$$
Thus the bracket is at most $a_W/2$. The capacity coefficient is
$$
\frac{2\nu_{\rm KD,B}}{N\varepsilon_W}
=\frac{4M_Wd_{\max}\nu_{\rm KD,B}}
{N(p_{\min}/\tau)\lambda_2m_W}
=C_{B,W}.
$$

For the quadratic KD contribution, commutation of $W_i$ and $\Theta_i$ gives
$$
\begin{aligned}
\|\Theta_i d_{ij}\|_{W_i}^2
&=d_{ij}^\top\Theta_iW_i\Theta_id_{ij}\\
&=d_{ij}^\top\Theta_i\Pi_i^zd_{ij}\\
&=d_{ij}^\top\Theta_id_{ij}\\
&\le\kappa_+\|d_{ij}\|^2\\
&\le\frac{\kappa_+}{4\tau^2}\|Z_i-Z_j\|^2.
\end{aligned}
$$
The final line uses
$\|d_{ij}\|\le(2\tau)^{-1}\|Z_i-Z_j\|$, obtained by integrating the
softmax Jacobian bound. Summing the two endpoint squares and averaging over
the edge gives
$$
\begin{aligned}
Q_{\rm KD}
&:=\frac{w^2}{N|E|}\sum_{\{i,j\}\in E}
\left(\|\Theta_id_{ij}\|_{W_i}^2
+\|\Theta_jd_{ij}\|_{W_j}^2\right)\\
&\le\frac{w^2\kappa_+}{2\tau^2N|E|}
\sum_{\{i,j\}\in E}\|Z_i-Z_j\|^2\\
&\le\frac{w^2\kappa_+\lambda_{\max}}{2\tau^2|E|}\Psi^z\\
&\le\frac{b_Ww^2}{2}\mathcal V_W.
\end{aligned}
$$
The factor $1/2$ in the last line is retained because $b_W$ is twice the
preceding coefficient after $\Psi^z\le\mathcal V_W/m_W$.

For the deterministic KD--perturbation cross term, choose
$u=w\Theta_kd_{ij}$ and $v=b_{k,e}$ in
$2|\langle u,v\rangle_{W_k}|\le\|u\|_{W_k}^2+\|v\|_{W_k}^2$.
After endpoint and edge averaging, the two squares contribute
$Q_{\rm KD}+D_{b,t}$. The original KD square and this extra copy satisfy
$$
2Q_{\rm KD}\le b_Ww^2\mathcal V_W
\le\frac{a_Ww}{8}\mathcal V_W,
$$
where the last inequality uses $w\le a_W/(8b_W)$. The KD vectors are
measurable conditional on $(\mathcal F_t,e)$, so their inner products with
the centered $\xi_{k,e}$ have zero conditional expectation.

For deterministic perturbations, apply
$$
2\langle Q_k,W_kb_{k,e}\rangle
\le\varepsilon_b\|Q_k\|_{W_k}^2
+\varepsilon_b^{-1}\|b_{k,e}\|_{W_k}^2
$$
with $\varepsilon_b=a_Ww|E|/(8d_{\max})$. The degree count in the edge
average gives
$$
\begin{aligned}
\frac2N\mathbb E_e\sum_{k\in e}\langle Q_k,W_kb_{k,e}\rangle
&\le\frac{\varepsilon_b}{N|E|}
\sum_i d_i\|Q_i\|_{W_i}^2
+\varepsilon_b^{-1}D_{b,t}\\
&\le\frac{\varepsilon_bd_{\max}}{|E|}\mathcal V_W
+\varepsilon_b^{-1}D_{b,t}\\
&=\frac{a_Ww}{8}\mathcal V_W
+\frac{8d_{\max}}{a_Ww|E|}D_{b,t}.
\end{aligned}
$$
Conditional on $(\mathcal F_t,e)$, $Q_k$, $b_{k,e}$, and the KD vectors are
measurable, while $\mathbb E_t\xi_{k,e}=0$. Hence
$$
\mathbb E_t\langle Q_k,W_k\xi_{k,e}\rangle=0,\qquad
\mathbb E_t\langle b_{k,e},W_k\xi_{k,e}\rangle=0.
$$
The perturbation square has the exact conditional expansion
$$
\begin{aligned}
\mathbb E_t\|b_{k,e}+\xi_{k,e}\|_{W_k}^2
&=\|b_{k,e}\|_{W_k}^2
+2\mathbb E_t\langle b_{k,e},W_k\xi_{k,e}\rangle\\
&\qquad \qquad \qquad +\mathbb E_t\|\xi_{k,e}\|_{W_k}^2\\
&=\|b_{k,e}\|_{W_k}^2
+\mathbb E_t\|\xi_{k,e}\|_{W_k}^2.
\end{aligned}
$$
The KD--$b$ cross bound contributes one $D_{b,t}$ and the perturbation square
contributes another, giving $2D_{b,t}+D_{\xi,t}$. The coefficient of
$\mathcal V_W$ is
$$
-a_Ww+\frac{a_Ww}{2}+\frac{a_Ww}{8}
+\frac{a_Ww}{8}
=-\frac{a_Ww}{4}.
$$
Together with the capacity term and the inverse-$w$ deterministic
perturbation term, this proves
Equation~\ref{eq:app-global-contraction}.
\end{proof}

\begin{lemma}[Full-event consensus recursion]
\label{lem:f-consensus-recursion}
Let $G_t=N^{-1}\sum_i\|g_i^t\|^2$. Under A2, A4, A6, and A8, the
perturbations in Lemma~\ref{lem:f-global-contraction} satisfy
\begin{equation}
\begin{split}
D_{b,t}\le R_G\eta^2G_t+R_B\eta^2B_f^2
+R_\zeta\eta^2\zeta_f^2
+R_0\eta^4+R_Aw^2A_t^f,\\
D_{\xi,t}\le\Sigma_0\eta^2.
\end{split}
\label{eq:app-perturbation-moments}
\end{equation}
Consequently,
\begin{equation}
\begin{split}
\mathbb E_t\mathcal V_W(Z^{t+1})
\le(1-c_0\eta)\mathcal V_W(Z^t)+C_G\eta G_t
+C_B\eta B_f^2\\
+C_\zeta\eta\zeta_f^2+C_\eta\eta^2+C_A\eta A_t^f,
\end{split}
\label{eq:app-consensus-recursion}
\end{equation}
where $c_0=a_W\alpha\tau/4
=\alpha p_{\min}\lambda_2/(2M_W|E|)$.
\end{lemma}
\begin{proof}
For an active directed endpoint $(i,j)$, define the three deterministic pieces
of the perturbation in Equation~\ref{eq:app-full-event} by
$$
\begin{aligned}
b_{i,e}^{\rm task}&=-\eta(1-\alpha)G_i^{\rm task},\\
b_{i,e}^{\rm curv}&=r_i^z,\\
b_{i,e}^{\rm lag}&=-w\Theta_i(P_j^{\rm EMA}-P_j),\\
b_{i,e}&=b_{i,e}^{\rm task}+b_{i,e}^{\rm curv}+b_{i,e}^{\rm lag}.
\end{aligned}
$$
Here $G_i^{\rm task}$ is the logit-space private-task pushforward defined in
A4. At first order it is $J_i^z(J_i^p)^*g_i^t$ plus the finite-support
fidelity residual. This cross-kernel map is distinct from the prediction
kernel $J_i^p(J_i^p)^*$ used in Lemma~\ref{lem:f-descent}. The exact A4
closure used here is that finite constants
$c_{z,G},c_{z,B},c_{z,\zeta}$ satisfy
$$
\frac1N\sum_i\mathbb E_t\|G_i^{\rm task}\|^2
\le c_{z,G}G_t+c_{z,B}B_f^2+c_{z,\zeta}\zeta_f^2.
$$
All three deterministic pieces are measurable conditional on
$(\mathcal F_t,e)$.

The stochastic component is the centered pushforward
$\xi_{i,e}=\xi_i^z-\mathbb E_t[\xi_i^z\mid\mathcal F_t,e]$.
For any vectors $a,b,c$ in a Hilbert space,
$$
\begin{aligned}
\|a+b+c\|^2
&=\|a\|^2+\|b\|^2+\|c\|^2
+2\langle a,b\rangle+2\langle a,c\rangle+2\langle b,c\rangle\\
&\le3\|a\|^2+3\|b\|^2+3\|c\|^2,
\end{aligned}
$$
where each cross term uses
$2\langle u,v\rangle\le\|u\|^2+\|v\|^2$. Apply this inequality in the
$W_i$ norm to the three perturbation pieces. The task piece satisfies
$$
\begin{aligned}
\|b_{i,e}^{\rm task}\|_{W_i}^2
&=\eta^2(1-\alpha)^2\|G_i^{\rm task}\|_{W_i}^2\\
&\le M_W\eta^2(1-\alpha)^2\|G_i^{\rm task}\|^2.
\end{aligned}
$$
For nonnegative endpoint quantities $u_i$, uniform edge sampling gives
$$
\frac1N\mathbb E_e\sum_{i\in e}u_i
=\frac1N\sum_i\frac{d_i}{|E|}u_i
\le\frac1N\sum_i u_i.
$$
Apply this activation identity and the A4 task-pushforward closure. Before
the outer factor $3$, the task contribution obeys

\begin{align}
\frac{1}{N}\sum_i\mathbb{E}_t\|b_{i,e}^{\rm task}\|_{W_i}^2
&\le M_W(1-\alpha)^2\eta^2
\left(c_{z,G}G_t+c_{z,B}B_f^2 \right. \\
&\quad \left. +c_{z,\zeta}\zeta_f^2\right).
\end{align}

The outer three-piece inequality multiplies these constants by $3$. One
valid set of constants for the final $D_{b,t}$ bound is

\begin{align}
R_G&=3M_W(1-\alpha)^2c_{z,G},\\
R_B&=3M_W(1-\alpha)^2c_{z,B},\\
R_\zeta&=3M_W(1-\alpha)^2c_{z,\zeta}.
\end{align}

The integral remainder bound from Lemma~\ref{lem:f-kd-gradient} and the
fourth-moment assumption give the explicit calculation

\begin{align}
\mathbb E_t\|b_{i,e}^{\rm curv}\|_{W_i}^2
&\le M_W\mathbb E_t\|r_i^z\|^2\\
&\le\frac{M_WL_\Theta^2}{4}
\mathbb E_t\|\Delta\theta_i\|^4\\
&\le\frac{M_WL_\Theta^2G_{\theta,4}^4}{4}\eta^4.
\end{align}

The second line squares
$\|r_i^z\|\le(L_\Theta/2)\|\Delta\theta_i\|^2$, and the third uses the A4
conditional fourth-moment bound. Thus one may take
$R_0=3M_WL_\Theta^2G_{\theta,4}^4/4$ after the outer three-piece inequality.

For the teacher lag,
$\Theta_iW_i\Theta_i=\Theta_i$ and
$\Theta_i\preceq\kappa_+\Pi_i^z$ imply

\begin{align}
\|b_{i,e}^{\rm lag}\|_{W_i}^2
&=w^2(P_j^{\rm EMA}-P_j)^\top
\Theta_iW_i\Theta_i(P_j^{\rm EMA}-P_j)\\
&=w^2(P_j^{\rm EMA}-P_j)^\top\Theta_i
(P_j^{\rm EMA}-P_j)\\
&\le w^2\kappa_+\|P_j^{\rm EMA}-P_j\|^2.
\end{align}

The directed-edge average obeys

\begin{align}
\frac1{N|E|}\sum_{(i,j)\in\vec E}
\|P_j^{\rm EMA}-P_j\|^2
&=\frac1{N|E|}\sum_jd_j\|P_j^{\rm EMA}-P_j\|^2\\
&\le\frac{d_{\max}}{N|E|}
\sum_j\|P_j^{\rm EMA}-P_j\|^2\\
&\le\frac{d_{\max}}{|E|}A_t^f.
\end{align}

With $R_A=3\kappa_+d_{\max}/|E|$, the lag contribution is bounded by
$R_Aw^2A_t^f$. Finally, the assumed pushforward
variance bound and conditional centering give

\begin{align}
\mathbb E_t\|\xi_{i,e}\|^2
&=\mathbb E_t\left\|
\xi_i^z-\mathbb E_t[\xi_i^z\mid\mathcal F_t,e]\right\|^2\\
&=\mathbb E_t\|\xi_i^z\|^2
-\left\|\mathbb E_t[\xi_i^z\mid\mathcal F_t,e]\right\|^2\\
&\le\sigma_z^2\eta^2.
\end{align}

Since $\|u\|_{W_i}^2\le M_W\|u\|^2$ and every activation probability is at
most one,
$D_{\xi,t}\le M_W\sigma_z^2\eta^2=:\Sigma_0\eta^2$.
Thus the three-piece inequality is
$$
D_{b,t}\le R_G\eta^2G_t+R_B\eta^2B_f^2
+R_\zeta\eta^2\zeta_f^2+R_0\eta^4+R_Aw^2A_t^f.
$$

Insert this inequality and $D_{\xi,t}\le\Sigma_0\eta^2$ into
Equation~\ref{eq:app-global-contraction}. Define
$$
H_1:=\frac{8d_{\max}}{a_W\alpha\tau|E|}.
$$
Since $w=\eta\alpha\tau$,
$$
\frac{8d_{\max}}{a_Ww|E|}+2
=\frac{H_1}{\eta}+2.
$$
For the gradient contribution,
$$
\begin{aligned}
\left(\frac{H_1}{\eta}+2\right)R_G\eta^2G_t
&=R_G(H_1+2\eta)\eta G_t\\
&\le R_G(H_1+2\eta_{\max})\eta G_t.
\end{aligned}
$$
The same calculation applies to the $B_f^2$ and $\zeta_f^2$ terms. The
capacity term already has the required scale because
$C_{B,W}wB_f^2=C_{B,W}\alpha\tau\,\eta B_f^2$.

For the curvature term, taking $\eta_{\max}\le1$ gives
$$
\begin{aligned}
\left(\frac{H_1}{\eta}+2\right)R_0\eta^4
&=R_0(H_1\eta^3+2\eta^4)\\
&\le R_0(H_1+2)\eta^2.
\end{aligned}
$$
For the lag term, $w\le1$ and $w=\eta\alpha\tau$ give
$$
\begin{aligned}
\left(\frac{H_1}{\eta}+2\right)R_Aw^2A_t^f
&=R_A\left(H_1\eta\alpha^2\tau^2
+2\eta^2\alpha^2\tau^2\right)A_t^f\\
&\le R_A\alpha^2\tau^2(H_1+2\eta_{\max})
\eta A_t^f.
\end{aligned}
$$
Finally, $D_{\xi,t}$ contributes $\Sigma_0\eta^2$. Absorb the displayed
finite coefficients into $C_G,C_B,C_\zeta,C_\eta,C_A$. The contraction term
is
$$
\frac{a_Ww}{4}
=\frac{a_W\alpha\tau}{4}\eta
=c_0\eta.
$$
These bounds give
Equation~\ref{eq:app-consensus-recursion}.
\end{proof}

\begin{lemma}[Activation-corrected task descent]
\label{lem:f-descent}
Let $q_i=\sum_{e\ni i}\Pr(e_t=e)=d_i/|E|$ and
\begin{equation}
\mathcal R_t=\frac1N\sum_iq_i^{-1}\big(F_i(p_i^t)-F_i^\star\big).
\label{eq:app-task-potential}
\end{equation}
Under A2, A4, A6, and the task clause of A10, there are constants
$a_F,b_F,D_B,D_\zeta,D_\eta,D_A>0$ such that
\begin{equation}
\begin{split}
\mathbb E_t\mathcal R_{t+1}
\le\mathcal R_t-a_F\eta G_t+b_F\eta\mathcal V_W(Z^t)
+D_B\eta B_f^2+D_\zeta\eta\zeta_f^2\\+D_\eta\eta^2+D_A\eta A_t^f.
\end{split}
\label{eq:app-task-descent}
\end{equation}
\begin{equation}
\|\nabla_fF(\bar p^t)\|^2\le2G_t+2L_f^2\psi_t^f.
\label{eq:app-gradient-transfer}
\end{equation}
\end{lemma}
\begin{proof}
Let $\langle\cdot,\cdot\rangle_\mu$ and $\|\cdot\|_\mu$ denote the prediction
Hilbert-space inner product and norm. Set
$g_i^t=\nabla_fF_i(p_i^t)$. For an active endpoint $i$, A4 decomposes the
prediction update as
$$
\begin{aligned}
\Delta p_i
&=J_i^p\Delta\theta_i+r_i^p+\xi_i^p\\
&=-\eta(1-\alpha)(K_ig_i^t+\delta_i^t)
-\eta\alpha H_i(P_i-P_j^T)+r_i^p+\xi_i^p.
\end{aligned}
$$
where $K_i=J_i^p(J_i^p)^*$,
$H_i=\tau J_i^p(J_i^z)^*$, and $\|H_i\|_{\rm op}\le H_{\max}$ by A4.
The remainder and noise satisfy
$\mathbb E_t\|r_i^p\|^2\le R_p\eta^4$ and
$\mathbb E_t\|\xi_i^p\|^2\le\Sigma_p\eta^2$.
Functional $L_f$-smoothness means, for every $u,v$,
$$
F_i(v)\le F_i(u)+\langle\nabla_fF_i(u),v-u\rangle
+\frac{L_f}{2}\|v-u\|_\mu^2.
$$
Taking $u=p_i^t$, $v=p_i^t+\Delta p_i$ gives the one-event inequality
$$
\begin{aligned}
F_i(p_i^{t+1})-F_i(p_i^t)
&\le\langle g_i^t,\Delta p_i\rangle_\mu
+\frac{L_f}{2}\|\Delta p_i\|_\mu^2.
\end{aligned}
$$

The task inner product is
$$
-\eta(1-\alpha)\langle g_i^t,K_ig_i^t\rangle_\mu
-\eta(1-\alpha)\langle g_i^t,\delta_i^t\rangle_\mu.
$$
Because $K_i=\Pi_i^pK_i\Pi_i^p$ and
$K_i\succeq\underline\kappa_p\Pi_i^p$,
$$
\begin{aligned}
\langle g_i^t,K_ig_i^t\rangle_\mu
&=\langle\Pi_i^pg_i^t,K_i\Pi_i^pg_i^t\rangle_\mu\\
&\ge\underline\kappa_p\|\Pi_i^pg_i^t\|_\mu^2\\
&=\underline\kappa_p\left(\|g_i^t\|_\mu^2
-\|(I-\Pi_i^p)g_i^t\|_\mu^2\right).
\end{aligned}
$$
The last equality is the orthogonal decomposition
$$
\begin{aligned}
\|g_i^t\|_\mu^2
&=\|\Pi_i^pg_i^t+(I-\Pi_i^p)g_i^t\|_\mu^2\\
&=\|\Pi_i^pg_i^t\|_\mu^2
+2\langle\Pi_i^pg_i^t,(I-\Pi_i^p)g_i^t\rangle_\mu\\
&\quad+\|(I-\Pi_i^p)g_i^t\|_\mu^2\\
&=\|\Pi_i^pg_i^t\|_\mu^2
+\|(I-\Pi_i^p)g_i^t\|_\mu^2,
\end{aligned}
$$
where the cross term is zero because the two projector ranges are
orthogonal.

For the residual, Young's inequality with $\varepsilon_\delta>0$ gives
$$
\begin{aligned}
2|\langle g_i^t,\delta_i^t\rangle_\mu|
&\le\varepsilon_\delta\|g_i^t\|_\mu^2
+\varepsilon_\delta^{-1}\|\delta_i^t\|_\mu^2.
\end{aligned}
$$
The inequality is obtained by expanding
$\|\sqrt{\varepsilon_\delta}g_i^t
\mp\varepsilon_\delta^{-1/2}\delta_i^t\|_\mu^2\ge0$ and choosing the sign
that matches the inner product.
Define the task decrease before multiplying by $\eta$ as
$$
\mathfrak D_{\rm task}
:=\frac{1-\alpha}{N}\sum_i
\left(\langle g_i^t,K_ig_i^t\rangle_\mu
+\langle g_i^t,\delta_i^t\rangle_\mu\right).
$$
The preceding projection and Young bounds give
\begin{align}
    \mathfrak D_{\rm task} \ge (1-\alpha) \left[ \underline\kappa_p G_t -\frac{\underline\kappa_p}{N} \sum_i \|(I-\Pi_i^p)g_i^t\|_\mu^2 -\frac{\varepsilon_\delta}{2}G_t -\frac{1}{2\varepsilon_\delta N} \sum_i \|\delta_i^t\|_\mu^2 \right].
\end{align}

Insert the A10 task-normal bound and the A4 fidelity bound
$N^{-1}\sum_i\|\delta_i^t\|_\mu^2
\le\chi_GG_t+\chi_BB_f^2+\chi_\zeta\zeta_f^2$. The result is
$$
\begin{aligned}
&\mathfrak D_{\rm task}\ge(1-\alpha)\left[\underline\kappa_p(1-\gamma_{\rm task})
-\frac{\varepsilon_\delta}{2}
-\frac{\chi_G}{2\varepsilon_\delta}\right]G_t\\
&\quad-(1-\alpha)\underline\kappa_p\nu_{\rm task,\Psi}\psi_t^f\\
&\quad-(1-\alpha)\left[\underline\kappa_p\nu_{\rm task,B}
+\frac{\chi_B}{2\varepsilon_\delta}\right]B_f^2
-(1-\alpha)\frac{\chi_\zeta}{2\varepsilon_\delta}\zeta_f^2.
\end{aligned}
$$
The task-margin requirement chooses $\varepsilon_\delta>0$ so the
coefficient of $G_t$ in this display is positive. Since
$\psi_t^f\le C_{fz}\psi_t^z\le(C_{fz}/m_W)\mathcal V_W(Z^t)$, its negative
term becomes a finite positive coefficient multiplying
$\mathcal V_W(Z^t)$ in the upper descent bound.

For the KD term, Young's inequality with $\varepsilon_{\rm KD}>0$ gives
$$
\begin{aligned}
2|\langle g_i^t,H_i(P_i-P_j^T)\rangle_\mu|
&\le\varepsilon_{\rm KD}\|g_i^t\|_\mu^2
+\varepsilon_{\rm KD}^{-1}H_{\max}^2
\|P_i-P_j^T\|_\mu^2.
\end{aligned}
$$
The teacher decomposition is
$P_i-P_j^T=(P_i-P_j)+(P_j-P_j^{\rm EMA})$. Hence
$$
\begin{aligned}
\|P_i-P_j^T\|_\mu^2
&=\|P_i-P_j\|_\mu^2
+2\langle P_i-P_j,P_j-P_j^{\rm EMA}\rangle_\mu\\
&\quad+\|P_j-P_j^{\rm EMA}\|_\mu^2\\
&\le2\|P_i-P_j\|_\mu^2
+2\|P_j-P_j^{\rm EMA}\|_\mu^2.
\end{aligned}
$$
The final line uses
$2\langle u,v\rangle_\mu\le\|u\|_\mu^2+\|v\|_\mu^2$.
Let $q_{\min}=\min_iq_i$. The graph upper inequality and
Lemma~\ref{lem:f-consensus-equivalence} give
$$
\begin{aligned}
\frac1{|E|}\sum_{\{i,j\}\in E}\|P_i-P_j\|_\mu^2
&\le\frac{N\lambda_{\max}}{|E|}\psi_t^f\\
&\le\frac{N\lambda_{\max}C_{fz}}{|E|}\psi_t^z\\
&\le\frac{N\lambda_{\max}C_{fz}}{|E|m_W}
\mathcal V_W(Z^t).
\end{aligned}
$$
The factor $N$ appears because
$\sum_i\|P_i-\bar P\|_\mu^2=N\psi_t^f$. In the activation-corrected
potential, the directed endpoint average is
$$
\begin{aligned}
&\frac1{N|E|}\sum_{(i,j)\in\vec E}
q_i^{-1}\|P_i-P_j\|_\mu^2\\
&\le\frac{2}{Nq_{\min}|E|}
\sum_{\{i,j\}\in E}\|P_i-P_j\|_\mu^2\\
&\le\frac{2\lambda_{\max}C_{fz}}
{q_{\min}|E|m_W}\mathcal V_W(Z^t).
\end{aligned}
$$
Similarly,
$$
\frac1{N|E|}\sum_{(i,j)\in\vec E}
q_i^{-1}\|P_j-P_j^{\rm EMA}\|_\mu^2
\le\frac{d_{\max}}{q_{\min}|E|}A_t^f.
$$
The constants may therefore be chosen as
$$
C_{\rm KD,V}:=\frac{2H_{\max}^2\lambda_{\max}C_{fz}}
{\varepsilon_{\rm KD}q_{\min}|E|m_W},
\qquad
C_{\rm KD,A}:=\frac{2H_{\max}^2d_{\max}}
{\varepsilon_{\rm KD}q_{\min}|E|}.
$$
The first KD square contributes $C_{\rm KD,V}\mathcal V_W(Z^t)$ and the
teacher-lag square contributes $C_{\rm KD,A}A_t^f$.
The factor $2$ in $C_{\rm KD,A}$ is a harmless enlargement. Define
$$
\begin{aligned}
d_0
&:=(1-\alpha)\left[
\underline\kappa_p(1-\gamma_{\rm task})
-\frac{\varepsilon_\delta}{2}
-\frac{\chi_G}{2\varepsilon_\delta}\right],\\
d_{\rm task}
&:=d_0-\frac{\alpha\varepsilon_{\rm KD}}{2}.
\end{aligned}
$$
The first line is the task coefficient obtained above. The factor
$\alpha\varepsilon_{\rm KD}/2$ is the gradient square in the KD Young
inequality after multiplying by the update weight $\eta\alpha$. Choose
$\varepsilon_{\rm KD}>0$ so that $d_{\rm task}>0$. The remaining KD squares
then contribute
$$
\eta\alpha C_{\rm KD,V}\mathcal V_W(Z^t)
+\eta\alpha C_{\rm KD,A}A_t^f,
$$
which are included in $b_F\eta\mathcal V_W(Z^t)$ and
$D_A\eta A_t^f$, respectively.

The activation correction is exact. If edge $e$ is sampled uniformly, then
$\Pr(i\in e)=q_i=d_i/|E|$. Consequently
$$
\begin{aligned}
\mathbb E_e\left[\frac1N\sum_{i\in e}q_i^{-1}h_i\right]
&=\frac1N\sum_e\Pr(e)
\sum_{i=1}^N\mathbf1\{i\in e\}q_i^{-1}h_i\\
&=\frac1N\sum_iq_i^{-1}h_i
\sum_e\Pr(e)\mathbf1\{i\in e\}\\
&=\frac1N\sum_iq_i^{-1}\Pr(i\in e)h_i\\
&=\frac1N\sum_ih_i.
\end{aligned}
$$
For the smoothness square, denote the task and KD components by
$$
u_i^{\rm task}:=-\eta(1-\alpha)(K_ig_i^t+\delta_i^t),
\qquad
u_i^{\rm KD}:=-\eta\alpha H_i(P_i-P_j^T).
$$
Cauchy--Schwarz in $\mathbb R^4$ gives
$$
\begin{aligned}
\|\Delta p_i\|_\mu^2
&=\|u_i^{\rm task}+u_i^{\rm KD}+r_i^p+\xi_i^p\|_\mu^2\\
&\le\left(
\|u_i^{\rm task}\|_\mu+\|u_i^{\rm KD}\|_\mu
+\|r_i^p\|_\mu+\|\xi_i^p\|_\mu\right)^2\\
&\le4\left(
\|u_i^{\rm task}\|_\mu^2+\|u_i^{\rm KD}\|_\mu^2
+\|r_i^p\|_\mu^2+\|\xi_i^p\|_\mu^2\right).
\end{aligned}
$$
The activation identity gives an exact uniform average for the task square:
$$
\begin{aligned}
&\frac1N\mathbb E_e\sum_{i\in e}q_i^{-1}
\mathbb E_t\|u_i^{\rm task}\|_\mu^2\\
&=\frac1N\sum_i\mathbb E_t\|u_i^{\rm task}\|_\mu^2\\
&\le2\eta^2(1-\alpha)^2
\left[\overline\kappa_p^2G_t
+\frac1N\sum_i\mathbb E_t\|\delta_i^t\|_\mu^2\right]\\
&\le2\eta^2(1-\alpha)^2
\left[(\overline\kappa_p^2+\chi_G)G_t
+\chi_BB_f^2+\chi_\zeta\zeta_f^2\right].
\end{aligned}
$$
For the KD square, the teacher decomposition and the two activation-corrected
edge bounds above give finite constants
$C_{\rm sq,V},C_{\rm sq,A}>0$ such that
$$
\begin{aligned}
&\frac1N\mathbb E_e\sum_{i\in e}q_i^{-1}
\mathbb E_t\|u_i^{\rm KD}\|_\mu^2\\
&\le\eta^2\alpha^2H_{\max}^2
\frac1{N|E|}\sum_{(i,j)\in\vec E}q_i^{-1}
\mathbb E_t\|P_i-P_j^T\|_\mu^2\\
&\le\eta^2C_{\rm sq,V}\mathcal V_W(Z^t)
+\eta^2C_{\rm sq,A}A_t^f.
\end{aligned}
$$
The remainder and noise bounds, together with the same activation identity,
give

\begin{align}
\frac1N\mathbb E_e\sum_{i\in e}q_i^{-1}
\mathbb E_t\|r_i^p\|_\mu^2\le R_p\eta^4,
\\
\frac1N\mathbb E_e\sum_{i\in e}q_i^{-1}
\mathbb E_t\|\xi_i^p\|_\mu^2\le\Sigma_p\eta^2.
\end{align}

Collecting the four squares produces constants
$C_{\Delta,G},C_{\Delta,V},C_{\Delta,B},C_{\Delta,\zeta},C_{\Delta,A}$
such that

\begin{align}
&\frac1N\mathbb E_e\sum_{i\in e}q_i^{-1}
\mathbb E_t\|\Delta p_i\|_\mu^2\\
&\le\eta^2\big[
C_{\Delta,G}G_t+C_{\Delta,V}\mathcal V_W(Z^t)\\
&\qquad +C_{\Delta,B}B_f^2+C_{\Delta,\zeta}\zeta_f^2
+C_{\Delta,A}A_t^f\big]\\
&\qquad+4R_p\eta^4+4\Sigma_p\eta^2.
\end{align}

Let $d_{\rm task}>0$ be the remaining coefficient of $G_t$ after the
first-order task and KD inner-product bounds. Choose the admissible step-size
ceiling so that
$\eta_{\max}\le1$ and
$L_fC_{\Delta,G}\eta_{\max}/2\le d_{\rm task}/2$. Then
$$
-d_{\rm task}\eta G_t
+\frac{L_f}{2}C_{\Delta,G}\eta^2G_t
\le-\frac{d_{\rm task}}2\eta G_t.
$$
Set $a_F=d_{\rm task}/2$. Each remaining smoothness term satisfies
$\eta^2X\le\eta_{\max}\eta X$, while
$4R_p\eta^4+4\Sigma_p\eta^2
\le4(R_p+\Sigma_p)\eta^2$. Absorbing these finite coefficients into
$b_F,D_B,D_\zeta,D_\eta,D_A$ proves
Equation~\ref{eq:app-task-descent}.

For the gradient transfer, define $F=N^{-1}\sum_iF_i$ and
$\bar p^t=N^{-1}\sum_ip_i^t$. Add and subtract $N^{-1}\sum_i
\nabla_fF_i(p_i^t)$:

\begin{align}
\nabla_fF(\bar p^t)
&=\frac1N\sum_i\nabla_fF_i(\bar p^t)\\
&=\frac1N\sum_i g_i^t
+\frac1N\sum_i\left[
\nabla_fF_i(\bar p^t)-\nabla_fF_i(p_i^t)\right].
\end{align}

For any vectors $v_i$, the Jensen bound follows from

\begin{align}
\left\|\frac1N\sum_iv_i\right\|_\mu^2
=\frac1{N^2}\sum_{i,j}\langle v_i,v_j\rangle_\mu
\le\frac1{N^2}\sum_{i,j}
\frac{\|v_i\|_\mu^2+\|v_j\|_\mu^2}{2}\\
=\frac1N\sum_i\|v_i\|_\mu^2.
\end{align}

Apply this bound to both sums and use
$\|a+b\|^2\le2\|a\|^2+2\|b\|^2$:

\begin{align}
\|\nabla_fF(\bar p^t)\|_\mu^2
&\le\frac2N\sum_i\|g_i^t\|_\mu^2\\
&\quad+\frac2N\sum_i\|\nabla_fF_i(\bar p^t)-\nabla_fF_i(p_i^t)\|_\mu^2\\
&\le2G_t+\frac{2L_f^2}{N}\sum_i\|\bar p^t-p_i^t\|_\mu^2\\
&=2G_t+2L_f^2\psi_t^f.
\end{align}

\end{proof}

\begin{lemma}[Strict Lyapunov recursion]
\label{lem:f-lyapunov}
Assume there exists $\lambda>0$ such that
\begin{equation}
a_F-\lambda C_G>0,\qquad \lambda c_0-b_F>0.
\label{eq:app-small-gain}
\end{equation}
Such a $\lambda$ exists exactly when $b_FC_G<a_Fc_0$; this is the
explicit task--gossip small-gain requirement.
For $\mathcal L_t=\mathcal R_t+\lambda\mathcal V_W(Z^t)$, there are constants
$c_G,c_V,K_B,K_\zeta,K_\eta>0$ such that
\begin{equation}
\begin{split}
\mathbb E_t\mathcal L_{t+1}
\le\mathcal L_t-c_G\eta G_t-c_V\eta\mathcal V_W(Z^t)
+K_B\eta B_f^2\\+K_\zeta\eta\zeta_f^2+K_\eta\eta^2.
\end{split}
\label{eq:app-lyapunov}
\end{equation}
\end{lemma}
\begin{proof}
The first inequality in Equation~\ref{eq:app-small-gain} is equivalent to
$\lambda<a_F/C_G$, while the second is equivalent to
$\lambda>b_F/c_0$. Thus a valid $\lambda$ exists precisely when
$$
\begin{aligned}
\frac{b_F}{c_0}&<\frac{a_F}{C_G}
\quad\Longleftrightarrow\quad
b_FC_G<a_Fc_0.
\end{aligned}
$$
When this strict inequality holds, every
$\lambda\in(b_F/c_0,a_F/C_G)$ satisfies both margins. Fix one such
$\lambda$.
Multiply Equation~\ref{eq:app-consensus-recursion} by $\lambda$ and display
each term:
$$
\begin{aligned}
\lambda\mathbb E_t\mathcal V_W&(Z^{t+1})
\le\lambda\mathcal V_W(Z^t)-\lambda c_0\eta\mathcal V_W(Z^t)
+\lambda C_G\eta G_t\\
&\quad+\lambda C_B\eta B_f^2+\lambda C_\zeta\eta\zeta_f^2
+\lambda C_\eta\eta^2+\lambda C_A\eta A_t^f.
\end{aligned}
$$
Equation~\ref{eq:app-task-descent} is
$$
\begin{aligned}
\mathbb E_t\mathcal R_{t+1}
&\le\mathcal R_t-a_F\eta G_t+b_F\eta\mathcal V_W(Z^t)
+D_B\eta B_f^2+D_\zeta\eta\zeta_f^2\\
&\quad+D_\eta\eta^2+D_A\eta A_t^f.
\end{aligned}
$$
By definition,
$\mathcal L_{t+1}=\mathcal R_{t+1}
+\lambda\mathcal V_W(Z^{t+1})$. Adding the two displays term by term gives
$$
\begin{aligned}
\mathbb E_t\mathcal L_{t+1}
&\le\mathcal L_t-(a_F-\lambda C_G)\eta G_t
-(\lambda c_0-b_F)\eta\mathcal V_W(Z^t)\\
&\quad+(D_B+\lambda C_B)\eta B_f^2
+(D_\zeta+\lambda C_\zeta)\eta\zeta_f^2\\
&\quad+(D_\eta+\lambda C_\eta)\eta^2
+(D_A+\lambda C_A)\eta A_t^f.
\end{aligned}
$$
Set
$$
\begin{aligned}
c_G&:=a_F-\lambda C_G>0,\\
c_V&:=\lambda c_0-b_F>0,\\
K_B&:=D_B+\lambda C_B,\\
K_\zeta&:=D_\zeta+\lambda C_\zeta,\\
K_A&:=D_A+\lambda C_A,\\
K_\eta^{(0)}&:=D_\eta+\lambda C_\eta.
\end{aligned}
$$
Lemma~\ref{lem:f-staleness} gives
$A_t^f\le\bar A^f\eta^2$ with
$\bar A^f=c_{\rm step}A_{\rm eff}^2$. For $\eta\le1$,
$$
K_A\eta A_t^f
\le K_A\bar A^f\eta^3
\le K_A\bar A^f\eta^2.
$$
Defining $K_\eta:=K_\eta^{(0)}+K_A\bar A^f$ proves
Equation~\ref{eq:app-lyapunov}.
\end{proof}

\begin{theorem}[Constant-step convergence]
\label{thm:fa}
Under A1--A10, $0<\alpha<1$, a connected graph with no isolated peers, the
frozen-kernel condition, the thresholds in
Equations~\ref{eq:app-kd-threshold} and \ref{eq:app-small-gain}, and
$0<\eta\le\eta_{\max}$, there are constants $c,C_0,C_1,C_2,C_3>0$ such that
\begin{equation}
\frac1T\sum_{t=0}^{T-1}
\left[\mathbb E\|\nabla_fF(\bar p^t)\|_\mu^2+c\Psi_t^f\right]
\le\frac{C_0}{\eta T}+C_1\eta+C_2B_f^2+C_3\zeta_f^2.
\label{eq:app-main-theorem}
\end{equation}
\end{theorem}
\begin{proof}
Take total expectation in Equation~\ref{eq:app-lyapunov}. The tower property
$\mathbb E[\mathbb E_tX]=\mathbb EX$ gives, for each $t$,
$$
\begin{aligned}
c_G\eta\mathbb EG_t+c_V\eta\mathbb E\mathcal V_W(Z^t)
&\le\mathbb E\mathcal L_t-\mathbb E\mathcal L_{t+1}\\
&\quad+K_B\eta B_f^2+K_\zeta\eta\zeta_f^2+K_\eta\eta^2.
\end{aligned}
$$
Summing from $t=0$ to $T-1$ telescopes the Lyapunov differences:
$$
\sum_{t=0}^{T-1}
\left(\mathbb E\mathcal L_t-\mathbb E\mathcal L_{t+1}\right)
=\mathbb E\mathcal L_0-\mathbb E\mathcal L_T.
$$
Both terms in
$\mathcal L_T=\mathcal R_T+\lambda\mathcal V_W(Z^T)$ are nonnegative,
because $F_i(p_i^T)-F_i^\star\ge0$ and $\mathcal V_W\ge0$. Thus
$$
\begin{aligned}
&c_G\eta\sum_{t=0}^{T-1}\mathbb EG_t
+c_V\eta\sum_{t=0}^{T-1}\mathbb E\mathcal V_W(Z^t)\\
&\le\mathbb E\mathcal L_0+T(K_B\eta B_f^2+K_\zeta\eta\zeta_f^2+K_\eta\eta^2).
\end{aligned}
$$
For deterministic initialization, $\mathbb E\mathcal L_0=\mathcal L_0$.

Lemma~\ref{lem:f-consensus-equivalence} and the metric comparison provide a
finite constant $C_{fz}$ such that
$\psi_t^f\le C_{fz}\mathcal V_W(Z^t)/m_W$ pathwise. Taking expectation gives
$$
\Psi_t^f
\le\frac{C_{fz}}{m_W}\mathbb E\mathcal V_W(Z^t).
$$
Taking expectation in Equation~\ref{eq:app-gradient-transfer} gives
$$
\begin{aligned}
\mathbb E\|\nabla_fF(\bar p^t)\|_\mu^2
&\le2\mathbb EG_t+2L_f^2\Psi_t^f\\
&\le2\mathbb EG_t+\frac{2L_f^2C_{fz}}{m_W}
\mathbb E\mathcal V_W(Z^t).
\end{aligned}
$$
For any fixed $c>0$, define
$$
K_*:=\max\left\{\frac2{c_G},
\frac{2L_f^2C_{fz}/m_W+cC_{fz}/m_W}{c_V}\right\}.
$$
The definition of $K_*$ ensures, for each $t$,
$$
\begin{aligned}
&\mathbb E\|\nabla_fF(\bar p^t)\|_\mu^2+c\Psi_t^f\\
&\le2\mathbb EG_t
+\frac{(2L_f^2+c)C_{fz}}{m_W}
\mathbb E\mathcal V_W(Z^t)\\
&\le K_*\left(
c_G\mathbb EG_t+c_V\mathbb E\mathcal V_W(Z^t)\right).
\end{aligned}
$$
Summing this inequality, applying the telescoped Lyapunov bound, and dividing
by $\eta T$ gives
$$
\begin{aligned}
\frac1T\sum_{t=0}^{T-1}
\left[\mathbb E\|\nabla_fF(\bar p^t)\|_\mu^2+c\Psi_t^f\right]
&\le\frac{K_*\mathcal L_0}{\eta T}+K_*K_\eta\eta\\
&\quad+K_*K_BB_f^2+K_*K_\zeta\zeta_f^2.
\end{aligned}
$$
Set $C_0=K_*\mathcal L_0$, $C_1=K_*K_\eta$,
$C_2=K_*K_B$, and $C_3=K_*K_\zeta$. This proves
Equation~\ref{eq:app-main-theorem}.
\end{proof}

The $\zeta_f^2$ term cannot generally be removed. With two identical
full-rank peers, $B_f=0$, no noise, and local quadratic risks centered at $b$
and $-b$, the symmetric CE+KD recursion has a nonzero disagreement fixed point
independent of $\eta$. This example satisfies the kernel assumptions and has
zero KD projection defect. Deterministic non-IID task drift is therefore a
separate neighbourhood contribution rather than a discretization error.

\section{Experimental Protocol and Complete Results}
\label{app:experiments}

This appendix gives the complete empirical validation of the convergence
claims. The campaign contains 25 stored trajectories: 20 random-edge KD runs,
four isolated local-training controls, and one homogeneous D-SGD control.
Twenty-four histories end at 70,000 edge events. The Setting-C, $\eta=0.05$,
seed-1 history contains 101 snapshots through 100,000 events, although its
copied configuration retains a 70,000-event target. All reported values use
the actual event indices stored in the JSON history. The figures and summary
statistics in this appendix are regenerated by
\texttt{appendix\_analysis.py} from the files in
\texttt{results/full}, but due to the limit of 50MB of data supplement, we only upload the code with the detailed execution procedures in \texttt{README.md}.
For running the whole experiment, we used MacStudio M4 Max 128 GB and we ran the whole experiments set for $\approx100$ hours. 

\subsection{Experimental Setup and Theorem-to-Protocol Mapping}
\label{app:exp-setup}

We simulate $N=20$ peers on CIFAR-10. Each peer receives a private shard from a
Dirichlet label partition with concentration $\gamma_{\rm part}=0.3$. For each
seed, the same connected $\operatorname{ER}(0.3)$ graph is used throughout the
run. Seed 0 has 51 edges and seed 1 has 58 edges, with graph spectral gaps
$1.232$ and $1.822$, respectively. The reference measure $\mu$ is the
empirical measure on a fixed 2,000-image CIFAR-10 test subset. The same subset
is used for KD and evaluation, so all function-space measurements are
transductive on $\mu$.

\begin{table}[H]
\centering
\caption{Model rosters used in the four settings. The parameter counts are
approximate and are included to describe capacity scale, not to estimate
$B_f$.}
\label{tab:app-architectures}
\begin{tabular}{@{}p{.09\linewidth}p{.52\linewidth}p{.22\linewidth}@{}}
\toprule
Set. & Roster / size & Purpose \\
\midrule
A & $20\times$ ResNet-m ($2.8$M each) & homogeneous reference \\
B & $7$s/$7$m/$6$l ($.7$--$6.3$M) & width heterogeneity \\
C & $7$xs/$7$m/$6$xl ($.7$--$11.2$M) & wider capacity range \\
D & ResNet-18, MobileNetV2, ShuffleNetV2 ($1.3$--$11.2$M) &
mixed families \\
\bottomrule
\end{tabular}
\end{table}

Every peer first receives 1,500 local cross-entropy pretraining steps. The
gossip phase samples one undirected edge per event and updates both endpoints.
The teacher is live, so there is no stored-version delay in these runs. Each
KD update uses the shared reference batch, $\alpha=0.5$, $\tau=4$, constant
step size $\eta$, batch size 64, weight decay $5\times10^{-4}$, label
smoothing $0.1$, logit clipping at 10, and parameter-gradient clipping at 5.
The main runs use $\eta=0.05$ and seeds 0 and 1. The seed-0 sensitivity sweep
uses $\eta\in\{0.0125,0.025,0.05,0.1\}$. Evaluation is recorded every 1,000
events. Table~\ref{tab:app-protocol} records the remaining settings.

\begin{table}[H]
\centering
\caption{Fixed protocol for the completed campaign.}
\label{tab:app-protocol}
\begin{tabular}{@{}p{.42\linewidth}p{.45\linewidth}@{}}
\toprule
Item & Value \\
\midrule
Dataset / peers & CIFAR-10 / $N=20$ \\
Private partition & Dirichlet, $\gamma_{\rm part}=0.3$ \\
Graph & connected $\operatorname{ER}(0.3)$ \\
Reference / probe support & 2,000 images \\
Pretraining & 1,500 CE steps, lr $0.05$, batch 128 \\
Gossip horizon / logging & 70,000 events / every 1,000 \\
KD weight / temperature & $\alpha=0.5$ / $\tau=4$ \\
Teacher / gossip momentum & live / 0 \\
Batch / weight decay & 64 / $5\times10^{-4}$ \\
Label smoothing & $0.1$ \\
Logit / gradient clip & $10$ / $5$ \\
Device & PyTorch MPS (MacStudio M4 Max 128GB)\\
\bottomrule
\end{tabular}
\end{table}

The measured consensus statistic is
\begin{equation}
\widehat\Psi_t^f
=\frac{1}{N}\sum_{i=1}^N
\left\|p_i^t-\overline p^t\right\|_\mu^2,
\qquad
\overline p^t=\frac{1}{N}\sum_{i=1}^N p_i^t,
\label{eq:app-consensus}
\end{equation}
where $p_i^t$ is the temperature-scaled prediction of peer $i$ on the
reference support. For the stationarity channel, let
$\widehat S_t$ denote the stored squared residual between the mean
temperature-one prediction and the label-smoothed target on $\mu$. If
$T_k=1000k$ denotes a logged horizon, the theorem-facing prefix statistic is
\begin{equation}
\begin{aligned}
\widehat G_T
&=\frac{1}{K(T)+1}\sum_{k=0}^{K(T)}\widehat S_{T_k},\\
K(T)&=\max\{k:T_k\leq T\}.
\end{aligned}
\label{eq:app-stationarity}
\end{equation}
Thus, $\widehat G_T$ is a sampled time average, whereas the final raw value
$\widehat S_T$ is an instantaneous terminal residual.

The protocol provides operational controls corresponding to several analytical
objects. The finite reference support, logit clipping, and gradient clipping
make the numerical observables bounded, which is the implementation analogue
of the bounded-support and moment controls used in the analysis. The common
prediction space makes $\widehat\Psi_t^f$ well-defined for all four rosters. The restricted task/KD alignment condition associated with
A10 is probed by Settings B--D, especially the mixed-family Setting D. This appendix uses the experiments to test the theorem's observable consequences.

\subsection{Synthetic Verification of the Convergence-Rate Diagnostic}
\label{app:synthetic-rate}

Before interpreting neural-network traces, we calibrated the rate estimator on
known signals. For a prescribed floor $b$ and exponent $p$, the calibration
trace has the form
\begin{equation}
G_T=b+aT^{-p}+\varepsilon_T,
\label{eq:app-synthetic-rate}
\end{equation}
with deterministic $p\in\{0.5,1.0,1.5\}$ and small zero-mean perturbations.
The diagnostic fits a line to
$\log(G_T-b)$ against $\log T$ over the positive-excess region. The recovered
exponents are $0.500$, $0.988$, and $1.428$, respectively. Figure
\ref{fig:app-synthetic-rate} shows the traces, the fitted lines, and the
prescribed reference slopes.

\begin{figure*}[t]
\centering
\includegraphics[width=0.98\textwidth]{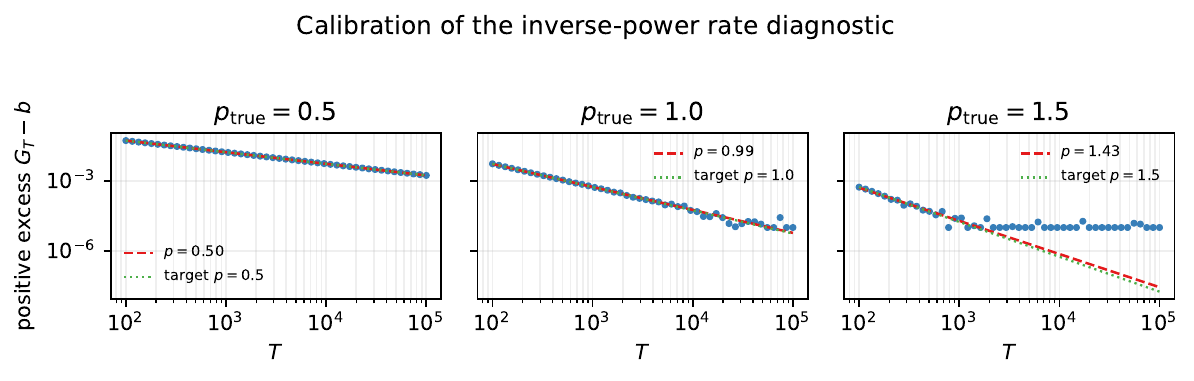}
\caption{Synthetic calibration of the inverse-power rate diagnostic. The
estimator recovers the known exponents $p=0.5$, $1.0$, and $1.5$ from
finite noisy traces. This validates the numerical diagnostic. This is regression, so it does not serve as any approximation for the exponent used in the theory, it only shows the variation of the data towards different exponents.}
\label{fig:app-synthetic-rate}
\end{figure*}

First, the implementation can recover an inverse-power exponent when the floor
and transient are known. Second, a neural trajectory must estimate its
neighbourhood from the same finite data, which introduces floor-estimation
error and makes the early transient and terminal noise relevant. The neural
results below therefore report the global fit and the late dynamic exponent
separately.

\subsection{Stored Convergence Traces}
\label{app:traces}

Figure~\ref{fig:app-convergence-traces} displays the seed-0 main trajectories.
The large panels show the instantaneous function-space disagreement
$\widehat\Psi_t^f$. The inset in each panel shows the prefix stationarity
statistic $\widehat G_T$. KD produces a rapid reduction in disagreement in all
four rosters, including the mixed-family Setting D where in the system, both architecture and the parameter sizes are different across models. The stationarity trace is less uniform across
rosters, which is expected because it also contains local-task and
representation effects.
\begin{figure}[H]
\centering
\includegraphics[width=1\linewidth]{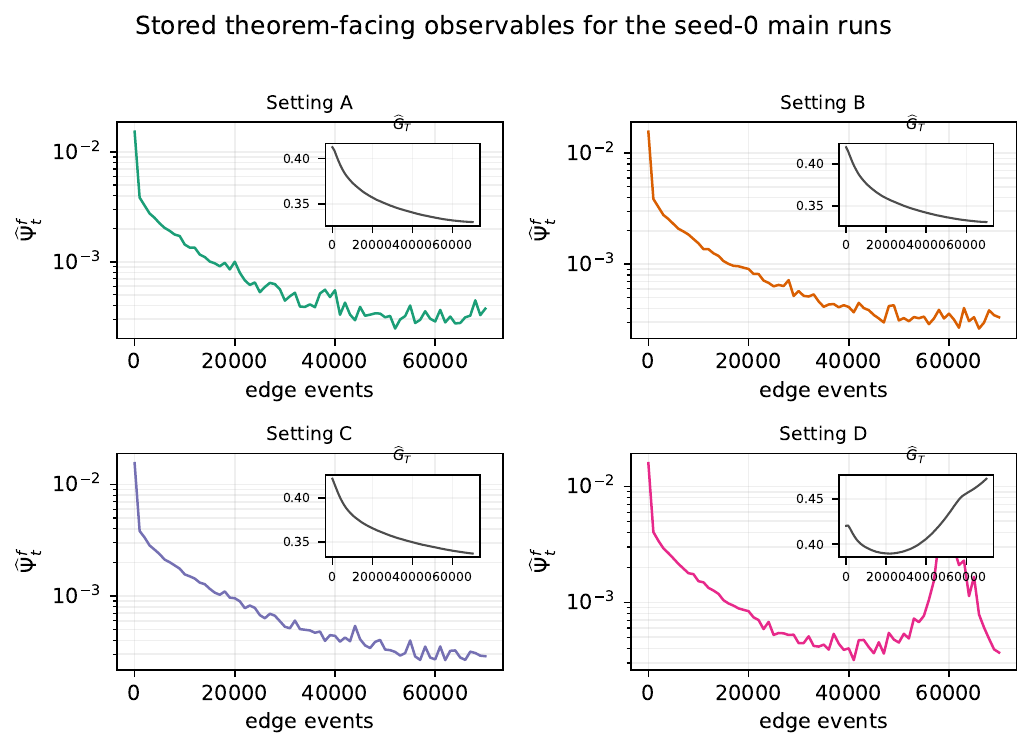}
\caption{Seed-0 main trajectories. The large panels show the instantaneous
function-space disagreement in Equation~\ref{eq:app-consensus}; each inset
shows the prefix stationarity statistic in Equation~\ref{eq:app-stationarity}.
The same output-space observable is available for homogeneous, width-heterogeneous,
and mixed-family peers.}
\label{fig:app-convergence-traces}
\end{figure}

\subsection{Empirical Convergence-Rate Diagnostics}
\label{app:rate-diagnostics}

For each stored history, the rate analysis estimates a terminal neighbourhood
$\widehat b_G$ by regressing the prefix statistic on $1/T$ over the final 80
percent of positive horizons (after pre-training, because if the models start from nothing but noise parameters, it does not make sense). It then defines the regularized excess
\begin{equation}
E_T^+
=\max\{\widehat G_T-\widehat b_G,\varepsilon_b\},
\label{eq:app-rate-excess}
\end{equation}
where $\varepsilon_b$ is a robust terminal-noise estimate. The dynamic
exponent is the negative local slope of $\log E_T^+$ against $\log T$ in a
seven-snapshot moving window. The late value
$\widetilde p_{\rm dyn}$ is the median over the latter half of the
resolvable positive-excess region.

Table~\ref{tab:app-main-results} gives the main-run diagnostics. Settings A and
B have late $85\% $ dynamic exponents close to $1.00$ for both seeds. Setting C reaches
$1.21$ for seed 0 and $1.90$ for the extended seed-1 history. The mixed-family
Setting D also reaches the consensus neighbourhood. 
\begin{table*}[t]
\centering
\caption{Main KD trajectories at $\eta=0.05$. $\widehat b_\Psi$ and
$\widehat b_G$ are finite-horizon final-85\% $1/T$-regression intercepts.
$\widehat G_T$ is the prefix statistic in Equation~\ref{eq:app-stationarity}.
A dash indicates that no positive stationarity excess remained above three
terminal-noise standard errors. $\widetilde p_{\rm dyn}$ is the best fit exponent for that specific experiment.}
\label{tab:app-main-results}
\begin{tabular}{@{}llrrrrrrr@{}}
\toprule
Set & Seed & $T$ & $\widehat\Psi_T^f$ & $\widehat b_\Psi$ &
$\widehat G_T$ & $\widehat b_G$ & $\widetilde p_{\rm dyn}$ & Ens. \\
\midrule
A & 0 & 70k & $3.739\mathrm e{-4}$ & $3.868\mathrm e{-4}$ &
.330 & .316 & 1.03 & .729 \\
A & 1 & 70k & $3.522\mathrm e{-4}$ & $2.793\mathrm e{-4}$ &
.397 & .389 & 1.02 & .643 \\
B & 0 & 70k & $3.326\mathrm e{-4}$ & $3.983\mathrm e{-4}$ &
.331 & .316 & 1.03 & .721 \\
B & 1 & 70k & $2.809\mathrm e{-4}$ & $2.714\mathrm e{-4}$ &
.399 & .391 & .99 & .622 \\
C & 0 & 70k & $2.882\mathrm e{-4}$ & $4.081\mathrm e{-4}$ &
.337 & .323 & 1.21 & .726 \\
C & 1 & 100k & $4.101\mathrm e{-4}$ & $2.782\mathrm e{-4}$ &
.399 & .398 & 1.90 & .603 \\
D & 0 & 70k & $3.688\mathrm e{-4}$ & $8.264\mathrm e{-4}$ &
.472 & .526 & -- & .535 \\
D & 1 & 70k & $3.100\mathrm e{-4}$ & $3.155\mathrm e{-4}$ &
.642 & .729 & -- & .329 \\
\bottomrule
\end{tabular}
\end{table*}

The global log-log fit is intentionally not used as the primary rate claim.
It includes early transients and is sensitive to the fitted floor. In the
present data its exponent is below one for many runs, while the late dynamic
exponent for A--D is near or above $1.00$. This pattern is consistent with a
finite trajectory entering the inverse-time regime after an initial
communication and optimization transient. The synthetic calibration shows why
the distinction is necessary: a correct estimator still returns a biased
global slope if the fit includes a non-asymptotic region or a poorly resolved
floor.

\subsection{Analysis and Verification of the Error Neighbourhood}
\label{app:error-neighbourhood}

The convergence bound has the schematic form
\begin{equation}
\frac{1}{T}\sum_{t<T}
\left(
\|\nabla_f F(\overline p^t)\|_\mu^2
c_\Psi\Psi_t^f
\right)
\leq
\frac{C_0}{\eta T}
C_1\eta
C_2B_f^2
C_3\zeta_f^2.
\label{eq:app-bound}
\end{equation}
The step-size sweep tests the two finite-horizon consequences that are
identifiable in the present campaign. At fixed $T$, a smaller $\eta$ reduces
the explicit discretization contribution but contracts more slowly because the
communication rate is proportional to $\eta$. A larger $\eta$ reaches the
neighbourhood sooner but can leave a larger stochastic neighbourhood. The
measured intercepts are not used to identify $B_f^2$ and $\zeta_f^2$
separately.

Figure~\ref{fig:app-convergence-traces} reports the fitted consensus and
stationarity intercepts for the complete seed-0 sweep. For A--C, the consensus
intercept decreases substantially from $\eta=0.0125$ to $\eta=0.05$ and then
changes little at $\eta=0.1$. Setting D has a higher and less regular stationarity
intercept, which is consistent with the additional representation and local
task mismatch in the mixed-model architecture roster.
\begin{figure}[H]
\centering
\includegraphics[width=1\linewidth]{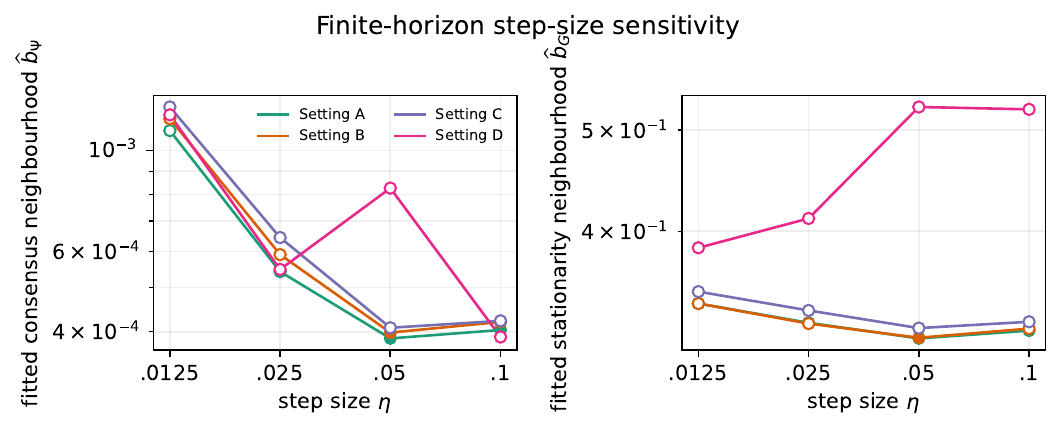}
\caption{Finite-horizon step-size sensitivity. Filled markers satisfy the
stored convergence gate and open markers do not. The fitted intercepts
summarize the observed neighbourhood of the prefix statistic and are not
identifications of individual terms in Equation~\ref{eq:app-bound}.}
\label{fig:app-eta-sensitivity}
\end{figure}

\begin{table}
\centering
\caption{Seed-0 step-size sweep. The intercepts are descriptive fits to the
stored histories. The convergence flag is the simulator's fixed tail-stability
criterion, not a theorem test.}
\label{tab:app-eta-sweep}
\begin{tabular}{@{}lrrrr}
\toprule
Set & $\eta$ & $\widehat b_\Psi$ & $\widehat b_G$ &
$\widehat\Psi_T^f$ \\
\midrule
A & .0125 & $1.105\mathrm e{-3}$ & .341 & $6.870\mathrm e{-4}$ \\
A & .0250 & $5.420\mathrm e{-4}$ & .327 & $3.046\mathrm e{-4}$ \\
A & .0500 & $3.868\mathrm e{-4}$ & .316 & $3.739\mathrm e{-4}$ \\
A & .1000 & $4.036\mathrm e{-4}$ & .321 & $3.389\mathrm e{-4}$ \\
\midrule
B & .0125 & $1.174\mathrm e{-3}$ & .341 & $7.553\mathrm e{-4}$ \\
B & .0250 & $5.907\mathrm e{-4}$ & .326 & $3.260\mathrm e{-4}$ \\
B & .0500 & $3.983\mathrm e{-4}$ & .316 & $3.326\mathrm e{-4}$ \\
B & .1000 & $4.204\mathrm e{-4}$ & .323 & $3.284\mathrm e{-4}$ \\
\midrule
C & .0125 & $1.245\mathrm e{-3}$ & .350 & $8.229\mathrm e{-4}$ \\
C & .0250 & $6.437\mathrm e{-4}$ & .336 & $3.615\mathrm e{-4}$ \\
C & .0500 & $4.081\mathrm e{-4}$ & .323 & $2.882\mathrm e{-4}$ \\
C & .1000 & $4.226\mathrm e{-4}$ & .328 & $2.483\mathrm e{-4}$ \\
\midrule
D & .0125 & $1.197\mathrm e{-3}$ & .386 & $7.098\mathrm e{-4}$ \\
D & .0250 & $5.479\mathrm e{-4}$ & .411 & $3.153\mathrm e{-4}$ \\
D & .0500 & $8.264\mathrm e{-4}$ & .526 & $3.688\mathrm e{-4}$ \\
D & .1000 & $3.895\mathrm e{-4}$ & .523 & $2.729\mathrm e{-4}$ \\
\bottomrule
\end{tabular}
\end{table}

The low-step-size rows should not be read as contradictions of the
$\mathcal O(1/(\eta T))$ transient. At the common event horizon, decreasing
$\eta$ also decreases the amount of function-space mixing completed by the
experiment. A longer horizon proportional to $1/\eta$ would be required to
compare the asymptotic neighbourhoods at equal effective communication time.
Likewise, the current data do not include an independent capacity probe that
would estimate $B_f$ directly. The results support the aggregate
transient--neighbourhood decomposition in Equation~\ref{eq:app-bound}, while
leaving the separate constants $B_f$ and $\zeta_f$ unidentifiable.

\subsection{Statistical Analysis of the Data Results}
\label{app:statistics}

The main KD campaign has two graph and partition seeds per setting. We report
means and sample standard deviations over those two seeds, rather than treating
the 20 peers within one run as independent repetitions. Figure
\ref{fig:app-seed-statistics} shows the resulting dispersion. The consensus
statistics are relatively stable: the seed standard deviations of terminal
$\widehat\Psi_T^f$ are $1.5\times10^{-5}$, $3.7\times10^{-5}$,
$8.6\times10^{-5}$, and $4.2\times10^{-5}$ for A--D, respectively. The
variation in the stationarity prefix is also small for A--C, with means
$0.363$, $0.365$, and $0.368$, while D has mean $0.557$ and a larger seed
standard deviation of $0.120$. This separates the robust consensus effect from
the more variable task-stationarity channel.

\begin{figure}[H]
\centering
\includegraphics[width=1\linewidth]{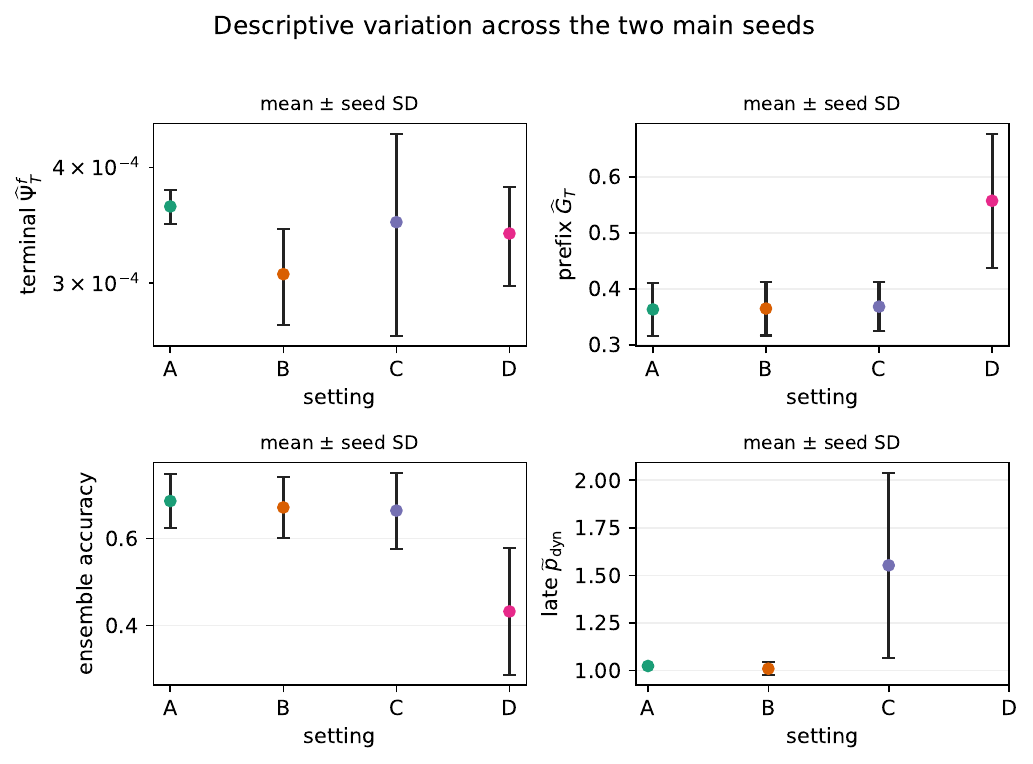}
\caption{Descriptive mean and seed standard deviation for the two main runs in
each setting. The lower-right panel omits Setting D because its stationarity
excess is not resolvable for either seed. }
\label{fig:app-seed-statistics}
\end{figure}

The mechanism control compares the seed-0 KD run at $\eta=0.05$ with isolated
local training using the same roster, data partition, and pretraining
protocol. The isolated runs finish with disagreement between
$1.36\times10^{-2}$ and $1.51\times10^{-2}$, whereas the KD runs finish
between $2.88\times10^{-4}$ and $3.74\times10^{-4}$. The paired reductions are
$36.3\times$, $40.9\times$, $47.5\times$, and $41.0\times$ for A--D, with a
mean reduction of $41.4\times$. An exploratory bootstrap interval over the four
setting-level ratios is $[37.4,45.8]\times$. Because the bootstrap has only four
paired setting units and the runs are not independent samples from a target
population, this interval is descriptive and is not used as a significance
claim.

\begin{figure}[H]
\centering
\includegraphics[width=1\linewidth]{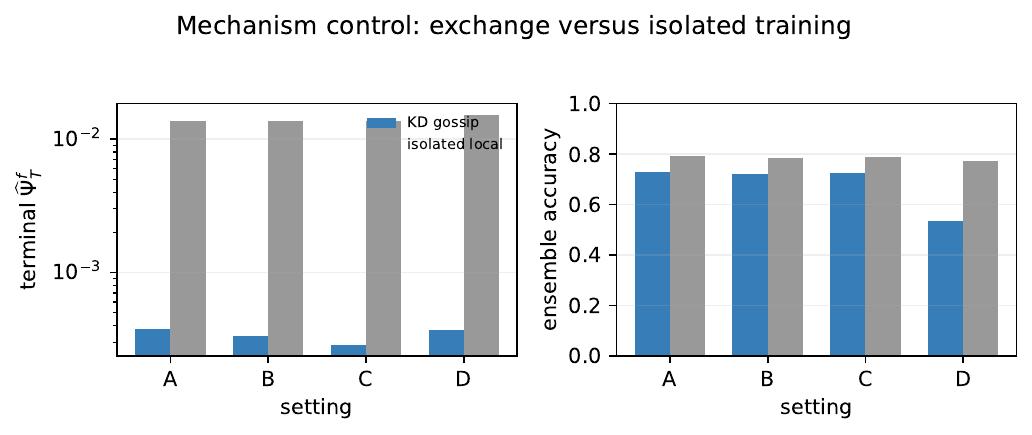}
\caption{Mechanism control at $\eta=0.05$, seed 0. KD gossip reduces
function-space disagreement by 36--47 times relative to isolated local
training across all rosters. The accuracy panel is included to show that the
consensus result is not being presented as a task-accuracy ranking.}
\label{fig:app-kd-isolated}
\end{figure}

The accuracy values are not interpreted as evidence that KD dominates local
training on task performance. In fact, the isolated control has higher
ensemble accuracy in these four runs. This is compatible with our
theoretical objective: the theorem concerns functional consensus and a
stationarity proxy, while the task channel contains local-data heterogeneity
and representation effects. 

\subsection{Sensitivity to Learning Rate and Model Heterogeneity}
\label{app:sensitivity}

The learning-rate sweep and the four rosters jointly test whether the
function-space mechanism is restricted to the homogeneous reference case.
Settings A--C have shared ResNet skeletons with different width schedules.
Their late rate diagnostics remain close to one at $\eta=0.05$, and their
terminal consensus values remain on the order of $10^{-4}$ across both main
seeds. Setting D replaces width changes with architecture-family changes. Its
terminal consensus values are of the same order as A--C, but its stationarity
residual is higher. This is the expected distinction between a communication observable,
which remains defined in prediction space, and a task-stationarity observable,
which depends on the reachable function classes and local objectives.

The isolated and D-SGD controls are interpreted as mechanism references. D-SGD
is reported only for homogeneous Setting A because parameter averaging is not
defined for B--D. Its final disagreement is $1.34\times10^{-3}$ and its
ensemble accuracy is $0.905$, but its optimizer and momentum configuration
differ from the KD phase. It is therefore not an optimizer ranking baseline.
The isolated control is more directly aligned with the consensus question of
local CE updates alone do not create the two-order reduction in
$\widehat\Psi_t^f$ observed after prediction exchange.

\subsection{Scope of the Empirical Evidence}
\label{app:empirical-scope}

The experimental evidence supports our theoretical statements. First, a
random-edge KD event contracts a common prediction-space disagreement measure
for homogeneous, width-heterogeneous, and mixed-family rosters. Second, the
synthetic calibration verifies the rate estimator, and the late dynamic
exponents for the shared-skeleton main runs are compatible with the inverse-time
transient in the theorem. Third, the learning-rate sweep exhibits the
predicted finite-horizon tradeoff between slow contraction at small $\eta$ and
the observed terminal neighbourhood at larger $\eta$. Fourth, the mechanism
control shows that the consensus reduction requires peer exchange rather than
local task training alone.

The evidence does not identify the numerical values of $B_f$, $\zeta_f$, or
the constants in Equation~\ref{eq:app-bound}. It also does not prove the
assumptions A4 and A10 as grounded truth. Their role is to provide
mathematical bridges in the proof. We evaluate the observables produced by those bridges, check the
rate-diagnostic implementation on known signals, vary capacity and
architecture family, and reports the finite-horizon behaviour.

\end{document}